\documentclass{article}

\PassOptionsToPackage{numbers, sort&compress}{natbib}

\usepackage[main, final]{neurips_2026}

\usepackage[utf8]{inputenc} 
\usepackage[T1]{fontenc}    
\usepackage{hyperref}       
\usepackage{url}            
\usepackage{booktabs}       
\usepackage{amsfonts}       
\usepackage{nicefrac}       
\usepackage{microtype}      
\usepackage{xcolor}         

\usepackage{graphicx} 
\usepackage[english]{babel}

\usepackage{amsthm}
\usepackage{amssymb}
\usepackage{amsfonts}
\usepackage{float}
\usepackage{xcolor}
\usepackage{algorithm}
\usepackage{algpseudocode}

\usepackage{amsmath}
\usepackage{graphicx}
\usepackage{enumitem}

\theoremstyle{theorem}
\newtheorem{theorem}{Theorem}
\newtheorem{lemma}{Lemma}
\newtheorem{definition}{Definition}

\newtheorem{assumption}{Assumption}

\theoremstyle{definition}

\theoremstyle{remark}
\newtheorem{remark}{Remark}

\usepackage{booktabs}
\usepackage{multirow}
\usepackage{tabularx}
\usepackage[table]{xcolor}
\usepackage{enumitem}

\title{A Distribution Mapping Approach to Counterfactually Fair Reinforcement Learning}

\author{
  Jianhan Zhang\\
  University of Michigan\\
  \texttt{jianhanz@umich.edu} \\
  \And
  Jitao Wang \\
  University of Michigan \\
  \texttt{jitwang@umich.edu} \\
  \And
  John D. Piette \\
  University of Michigan \\
  \texttt{jpiette@umich.edu}
  \And
  Donglin Zeng \\
  University of Michigan \\
  \texttt{dzeng@umich.edu} \\
  \And
  Chengchun Shi \\
  LSE \\
  \texttt{c.shi7@lse.ac.uk} \\
  \And
  Zhenke Wu \\
  University of Michigan \\
  \texttt{zhenkewu@umich.edu}  
}

\begin{document}

\maketitle

\begin{abstract}
Reinforcement learning (RL) seeks to optimize sequential decisions to maximize population-level benefits over time. However, when deployed in high-stakes settings such as healthcare, RL decisions might systematically restrict some subpopulation's access to valuable resources or services in a manner contrary to the values and goals of stakeholders. Counterfactual fairness (CF) offers a promising framework to address this problem based on causal reasoning. This paper develops a data preprocessing algorithm that, when used in tandem with policy learning, enables CF in RL. Our algorithm relies on a novel quantile distribution mapping method for sequentially estimating the counterfactual states and rewards in the data preprocessing step, subsuming common additivity assumptions used for counterfactual prediction as special case. We theoretically prove that the per-step level of counterfactual unfairness and infinite-horizon suboptimality gap can be bounded under mild regularity conditions. We also empirically test our algorithm in numerical experiments as well as in application to a real-world interventional digital health dataset.
\end{abstract}

\section{Introduction}
\label{sec:intro}

Reinforcement learning (RL) has gained popularity in various fields, including healthcare, banking, autonomous driving, and large language model fine-tuning, as a tool to facilitate data-driven sequential decision-making. Using data on an agent's interactions with its environment, RL learns a sequential decision rule, often called a ``policy'', that maximizes population-level benefits in an environment of interest across a possibly infinite time horizon. However, concerns have been raised that the sequential decisions made by a fairness-unaware RL agent might inadvertently disadvantage individuals in certain subgroups that may be defined by sensitive attributes, e.g., race/ethnicity, gender, education level, etc. In particular, unfairness will arise if the sensitive attribute is associated with the state and reward variables in a way that leads the RL agent to systematically withhold valuable resources from an individual. Deploying such fairness-unaware RL agents in sensitive settings risks jeopardizing social justice and eroding public trust.

This work is motivated by the fairness implications of the PowerED study \citep{piette2023powerED}. The PowerED study aimed to evaluate the effectiveness of an RL-based intervention designed to reduce patients' self-reported opioid analgesic misuse behaviors while conserving counselors' time. Each week during patients' 12-week intervention period, state variables such as self-reported pain scores were used by an RL agent to select one of three treatment options: (1) a brief automated ``interactive voice response'' (IVR) call, (2) a longer IVR call, or (3) a live counselor call. In this setting, a fairness-unaware RL agent might systematically decline to assign live counselor calls to patients of certain races, essentially denying those patients' access to high-quality, human-based care. For example, if the sensitive attribute (e.g., gender) is directly used as one of the state variables, then the agent will use gender as a decision factor, possibly resulting in the system assigning live calls with a higher probability to patients of one gender than the others. Even if the sensitive attribute is excluded from the set of state variables, such unfairness can still arise because the state and reward variables often encode implicit information about the sensitive attribute. For example, a prior study indicates that Hispanic American patients often report fewer pain conditions, potentially due to cultural factors \citep{hollingshead2016pain}, which could mislead the agent to underestimate Hispanic patients' pain levels and thereby deny assigning more expensive live calls to them. 

To mitigate the potential unfairness in RL-based decision-making, this paper aims to develop methodologies that enforce counterfactual fairness (CF) in RL. CF requires that the distribution of the decisions made by the algorithm remain unchanged had an individual's sensitive attribute (e.g. race, gender, etc.) been switched to a different value while keeping unchanged the values of all historical exogenous variables in the model \citep{kusner2018counterfactualfairness,wang2025counterfactuallyfairreinforcementlearning}. CF enforces fairness at the \textit{individual level} based on causal reasoning. This differs from many other popular fairness criteria, such as demographic parity and equal opportunity, which are \textit{group-level} and based solely on \textit{statistical associations} rather than causation. A more detailed comparison of CF with other popular fairness criteria can be found in Section~\ref{sec:cf_comparison} of the Appendix. 

\paragraph{Related work.} There exists a growing volume of literature on achieving CF or related causal fairness notions in the \textit{single-stage} setting. \citet{kusner2018counterfactualfairness} and \citet{zuo2022counterfactual} proposed enforcing CF by building machine learning (ML) models that rely only on non-descendants of the sensitive attribute in causal graphs. \citet{wang2023adjusting} introduced a method that adjusts possibly unfair ML predictors via post-processing to enforce CF. \citet{chen2024flap} developed a data-preprocessing algorithm that enforces CF by removing information about the sensitive attribute from the training data. \citet{distefano2020counterfactual} considered enforcing a similar causally-inspired fairness definition through regularization during training. \citet{huang2022achieving} and \citet{chen2025cflb} proposed methods that enforce causal notions of fairness similar to CF in the bandit setting. \citet{bian2026double} discussed balancing counterfactual action fairness and counterfactual outcome fairness in one-step policy learning.

For sequential decision-making, focusing solely on static settings without taking into account the dynamic nature of the environment might lead to suboptimal outcomes \citep{damour2020fairness,liu2018delayed}. Causal directed acyclic graphs (DAGs) emerged as a framework for studying fairness in dynamical systems \citep{creager2020causal4fairness}. By modeling contextual Markov decision processes as causal DAGs, \citet{wang2025counterfactuallyfairreinforcementlearning} extended the definition of CF to the RL setting and proposed a data-preprocessing algorithm to enforce CF in offline RL by removing sensitive information from training trajectories. However, this algorithm requires the noise to be additive to the state and reward variables, which may be restrictive in practice.

\paragraph{Contributions.} Our work makes three major contributions. First, we propose a data-preprocessing algorithm that ensures CF in RL. In particular, this algorithm extends the algorithm introduced in \cite{wang2025counterfactuallyfairreinforcementlearning} to a strictly larger class of underlying contextual Markov decision processes where the additive noise assumption is allowed to be violated. Second, we provide theoretical guarantees for our proposed algorithm by deriving bounds for the suboptimality gap and counterfactual unfairness resulting from the policy learned using the proposed algorithm. Third, we empirically test our proposed algorithm using numerical experiments and real data analysis.

\section{Preliminaries}
\label{sec:prelim}

\subsection{Structural Causal Models and Counterfactuals}
\label{subsec:scm}
We define a structural causal model (SCM) as the triplet $\mathcal{M} = \{\mathcal{U}, \mathcal{V}, \mathcal{F}\}$ which models the causal relationships among the variables of interest. In the triplet, $\mathcal{U}$ is the set of unobserved exogenous variables. The causes of unobserved exogenous variables are not explained by the model. Subsequently, an unobserved exogenous variable cannot be a descendant of any other variable. $\mathcal{V}=\{V_1, \dots, V_n\}$ is the set of observed variables. Each observed variable must be the descendant of at least one unobserved exogenous variable. $\mathcal{F}=\{f_1, \dots, f_n\}$ is a set of functions where each $f_i$ explains how the value of $V_i$ is determined. In particular, $V_i = f_i(\mathcal{V}_i, \mathcal{U}_i)$ where $\mathcal{V}_i \subset \mathcal{V}$ is the set of observed variables that are parents of $V_i$ and $\mathcal{U}_i \subset \mathcal{U}$ is the set of unobserved exogenous variables that are parents of $V_i$. If the values of all unobserved exogenous variables are set, then the values of all observed variables are fully determined using functions in $\mathcal{F}$. Every SCM can be represented by a graphical causal model; we assume the graphical causal model for the SCM in this paper is a directed acyclic graph (DAG).

The notion of counterfactuals can be defined using SCMs. Let $Y \in \mathcal{V}$ satisfy $Y = f_Y(Z,U_Y)$ where $Z \subset \mathcal{V}$ and $U_Y \subset \mathcal{U}$. Suppose the observed value of $Z$ is $z$ and the unobserved realized value of $U_Y$ is $u_Y$, so that the observed $Y$ is $Y^{Z \leftarrow z} = f_Y(z,u_Y)$. The counterfactual of $Y$ under an alternative value $Z=z'$ with $z'\neq z$ is denoted and defined by $Y^{Z \leftarrow z'} = f_Y(z',u_Y)$. By computing the value of $f_Y$ under $z'$ while keeping the value of $U_Y$ unchanged, $Y^{Z \leftarrow z'}$ answers the question ``what would be the value of $Y$, if $Z$ had taken the value $z'$ rather than $z$, given the observed data under $Z=z$?"

\subsection{Data Generating Process and Notation}
\label{subsec:cmdp}

\begin{definition}[CMDP; \cite{hallak2015contextualmarkovdecisionprocesses}]
Contextual Markov Decision Process (CMDP) is a tuple $(\mathcal{C}, \mathcal{S}, \mathcal{A}, \mathcal{M}(c))$ where $\mathcal{C}$ is called the context space, $\mathcal{S}$ and $\mathcal{A}$ are the state and action spaces correspondingly, and $\mathcal{M}$ is function mapping any context $c \in \mathcal{C}$ to an MDP $\mathcal{M}(c) = (\mathcal{S}, \mathcal{A}, p^c(y|x, a), r^c(x), \eta_0^c)$ where $p^c$, $r^c$, and $\eta_0^c$ are context-specific transition kernel, reward function, and initial state distribution, respectively.
\end{definition}

Here we consider data generated from a stationary CMDP where the sensitive attribute serves as the context variable. The dataset contains independent and identically distributed trajectories of $N$ individuals generated using some behavioral policy $\pi^b_t$, with the $j$-th individual having a horizon of $T_j+1$ (with index starting from $0$ and ending with $T_j$) and a time-invariant sensitive attribute $Z_j$ from the space of sensitive attributes $\mathcal{Z} = \{z^{(1)}, \dots, z^{(K)}\}$ that can only be caused by exogenous variables.  In this paper, we assume without the loss of generality that $T_j=T$ for all $j = 1, \dots, N$ for some constant $T$. For each $j = 1, \dots, N$ and $t = 0, \dots T$, let $S_{j,t} = (S_{j,t}^{(1)}, \dots, S_{j,t}^{(M)}) \in \mathcal{S}$ be the $M$-dimensional state of the $j$-th individual at time $t$, $A_{j,t} \in \mathcal{A}$ be the action taken by the $j$-th individual at time $t$, and $R_{j,t} \in \mathcal{R}$ be the reward received by the $j$-th individual after taking action $A_{j,t}$. Let $U_{j,t}^A$, $U_{j,t}^R$, and $\{U_{j,t}^{S^{(i)}}\}_{i=1}^M$ be scalar exogenous variables that are parents of the action, reward, and each dimension of the state, respectively, at time $t$. In particular, the reward and each dimension of the state has its own single scalar-valued exogenous variable; see Section \ref{sec:additional_single_scalar_noise_discussion} for additional discussion of this single scalar noise formulation. We write $U_t^S = (U_t^{S^{(1)}}, \dots, U_t^{S^{(M)}})$ for convenience. Furthermore, let $\bar{W}_{j,t} = (W_{j,0}, W_{j,1}, \dots, W_{j,t})$ for any time-dependent variable $W_{j,k}$. Also, denote the full history up to time $t$ by $H_{j,t} = \{Z_j, \bar{S}_{j,t}, \bar{A}_{j,t-1}, \bar{R}_{j,t-1}\}$ for $t \geq 1$ and $H_{j,t} = \{Z_j, \bar{S}_{j,0}\}$ for $t=0$. Since the trajectories are i.i.d., unless it is necessary to identify the associated individual, we omit the index $j$ in the variables defined above for simplicity.

In this paper, we represent the data-generating stationary CMDP using an SCM. More precisely, for the initial time step, there exist functions $\{f_0^{S^{(i)}}\}_{i=1}^M$ and $f_0^R$ and decision rule $\pi_0^b$ such that 
\begin{equation*}
S_0^{(i)} = f_0^{S^{(i)}}(Z, U_0^{S^{(i)}}), \text{ } A_0 = \pi^b_0(Z, S_0, U_0^A), \text{ } R_0 = f_0^R(Z, S_0, A_0, U_0^R).
\end{equation*}
For $t \geq 1$, there exist functions $\{f^{S^{(i)}}\}_{i=1}^M$ and $f^R$ and decision rule $\pi_t^b$ such that 
\begin{equation*}
S_t^{(i)} = f^{S^{(i)}}(Z, S_{t-1}, A_{t-1}, U_t^{S^{(i)}}), \text{ } A_t = \pi^b_t(Z, S_t, U_t^A), \text{ } R_t = f^R(Z, S_t, A_t, U_t^R).
\end{equation*}
Figure \ref{fig:cmdp_dag} offers a DAG representation of the data-generating CMDP.

\begin{figure}
    \centering
    \includegraphics[width=0.75\linewidth]{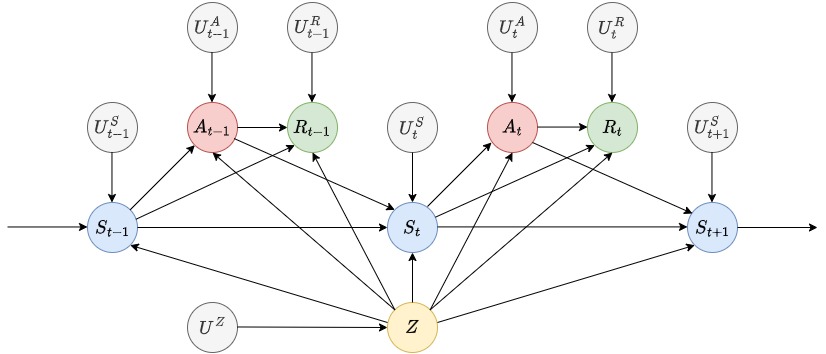}
    \caption{Causal DAG of CMDP \citep{wang2025counterfactuallyfairreinforcementlearning}}
    \label{fig:cmdp_dag}
\end{figure}

\subsection{Counterfactual Fairness in RL}
\label{subsec:cfrl}
CF was first proposed by \citet{kusner2018counterfactualfairness} for single-stage ML tasks, and \citet{wang2025counterfactuallyfairreinforcementlearning} extended the definition of CF to the RL setting. The current paper adopts the definition of CF introduced in \cite{wang2025counterfactuallyfairreinforcementlearning}, which we briefly review below.

We consider a general RL setting with data generated from a stationary CMDP and decisions made by a (possibly non-stationary) policy. Let $\pi_t$ denote the policy at time $t$. By the three-step procedure for counterfactual inference introduced in \cite{pearl2016causal}, the probability distributions of $U_t^S$ and $U_{t-1}^R$ can each be seen as a function of $H_t$. Following \cite{wang2025counterfactuallyfairreinforcementlearning}, for $t \geq 0$, we write $U_t^S = U_t^S(H_t)$ and $U_t^R = U_t^R(H_{t+1})$ by an abuse of notation and assume the mappings $U_t^{S}(\cdot)$ and $U_t^{R}(\cdot)$ are deterministic. Furthermore, let $\bar{U}_t(\cdot)$ be a vector-valued function such that $\bar{U}_t(\cdot) = (U_0^S(\cdot), U_0^R(\cdot), U_1^S(\cdot), U_1^R(\cdot), \dots, U_{t-1}^S(\cdot), U_{t-1}^R(\cdot), U_t^S(\cdot))$.

By the SCM representation of the data-generating CMDP, the counterfactual states, actions, and rewards can each be seen as a function of the historical exogenous variables $\bar{U}_t$. Moreover, $\bar{U}_t$ is assumed to be a deterministic function of $H_t$. Inspired by this observation, we let $A_t^{Z \leftarrow z'}(\bar{U}_t(h_t))$, $S_t^{Z \leftarrow z'}(\bar{U}_t(h_t))$, and $R_t^{Z \leftarrow z'}(\bar{U}_{t+1}(h_{t+1}))$ denote the counterfactual action, state, and reward, respectively, at time $t$ under $\pi_t$ for an individual with observed history $h_t$, had the individual's sensitive attribute been set to $z'$ and the individual followed the observed past action sequence $\bar{a}_{t-1}$ (or $\bar{a}_t$ in the case of the reward). A policy is counterfactually fair if, at each time step, the probability distribution of its action assignment to an individual would stay the same had the individual's sensitive attribute been switched to any other value while holding constant all historical actions and exogenous variables. This definition is formally stated below.

\begin{definition}[Counterfactual Fairness in CMDP; \cite{wang2025counterfactuallyfairreinforcementlearning}]
Given an observed trajectory $H_t = h_t = \{z, \bar{s}_t, \bar{a}_{t-1}, \bar{r}_{t-1}\}$, a decision rule $\pi_t$ is counterfactually fair at time $t$ if it satisfies the following condition:
\begin{equation*}
    \mathbb{P}^{\pi_t} \left( A_t^{Z \leftarrow z'} (\bar{U}_t(h_t)) = a \right) = \mathbb{P}^{\pi_t} \left( A_t^{Z \leftarrow z} (\bar{U}_t(h_t)) = a \right), \forall z' \in \mathcal{Z},  \forall a \in \mathcal{A}.
\end{equation*}
A policy $\{\pi_t\}_{t \geq 0}$ is said to be counterfactually fair if the above holds for any $t$.
\end{definition}

To characterize counterfactually fair policies, let $\tilde{S}_t = (S_t^{Z \leftarrow z^{(1)}}(\bar{U}_t(h_t)), \dots, S_t^{Z \leftarrow z^{(K)}}(\bar{U}_t(h_t)))$ and $\tilde{R}_t = \sum_{k=1}^K \mathbb{P}(Z=z^{(k)}) R_t^{Z \leftarrow z^{(k)}}(\bar{U}_{t+1}(h_{t+1}))$. Note that $\tilde{S}_t$ and $\tilde{R}_t$ do not depend on the value of the individual's observed sensitive attribute. \citet{wang2025counterfactuallyfairreinforcementlearning} showed that $\{\tilde{S}_t, \tilde{R}_{t}, A_{t}\}_{t \geq 0}$ follows a stationary Markov decision process (MDP) in which any policy is counterfactually fair. Moreover, the optimal policy in this MDP is stationary, so it can be conveniently learned by existing RL algorithms provided that $\{\tilde{S}_t, \tilde{R}_{t}, A_{t}\}_{t \geq 0}$ are known. We refer interested readers to Theorems 1 and 2 in \cite{wang2025counterfactuallyfairreinforcementlearning} for more details.

\section{Counterfactually Fair Sequential Marginal Distribution Mapping}
\label{sec:cfsmdm}

In this section, we present a novel data preprocessing method ``counterfactually fair sequential marginal distribution mapping'' (CFSMDM) which estimates the requisite counterfactuals to construct CF policies. CFSMDM achieves CF in policy learning via a two-stage procedure. First, we preprocess the training trajectories to obtain augmented training trajectories where the states are estimates $\hat{\tilde{S}}_t$ of $\tilde{S}_t$ and the rewards are estimates $\hat{\tilde{R}}_t$ of $\tilde{R}_t$. This step removes information of the sensitive attribute from the training trajectories. Then, we 
perform policy learning on the augmented training trajectories using existing offline RL algorithms, such as fitted Q iteration (FQI) \citep{riedmiller2005fqi}. The resulting policy should be close to being counterfactually fair because the preprocessed training trajectories contain no information about the sensitive attribute except possibly some residual information introduced by estimation errors. We follow the setup and notation introduced in Section~\ref{sec:prelim}. Assumptions~\ref{ass:state_monotonicity} and~\ref{ass:reward_monotonicity} ensure the identifiability of the counterfactual states and rewards via Theorems \ref{thm:state_matching} and \ref{thm:reward_matching}.

\begin{assumption}[Strict monotonicity of the state in the noise variable]
\label{ass:state_monotonicity}
For $t = 0$ and $i = 1, \dots, M$, $S_0^{(i)} = f_0^{S^{(i)}}(Z, U_0^{S^{(i)}})$ for some function $f_0^{S^{(i)}}$ that is deterministic and strictly increasing in $U_0^{S^{(i)}}$ on the support of $U_0^{S^{(i)}}$ for all levels of $Z$. For $t \geq 1$ and $i = 1, \dots, M$, $S_t^{(i)} = f^{S^{(i)}}(Z, A_{t-1}, S_{t-1}, U_t^{S^{(i)}})$ for some function $f^{S^{(i)}}$ that is deterministic and strictly increasing in $U_t^{S^{(i)}}$ on the support of $U_t^{S^{(i)}}$ for all levels of $(Z, A_{t-1}, S_{t-1})$.
\end{assumption}

\begin{assumption}[Strict monotonicity of the reward in the noise variable]
\label{ass:reward_monotonicity}
For $t \geq 0$, $R_t = f^R(Z, S_t, A_t, U_t^R)$ for some function $f^R$ that is deterministic and strictly increasing in $U_t^R$ on the support of $U_t^R$ for all levels of $(Z, A_t, S_t)$.
\end{assumption}

Section \ref{sec:cfsmdm_cfsdp} in the Appendix shows that Assumptions~\ref{ass:state_monotonicity} and~\ref{ass:reward_monotonicity} are strictly weaker than the additivity assumption made by the alternative method proposed by \cite{wang2025counterfactuallyfairreinforcementlearning}. In the counterfactual fairness literature, Assumptions \ref{ass:state_monotonicity} and \ref{ass:reward_monotonicity} are similar in spirit to the weak nonsensitive condition assumed by the marginal distribution mapping preprocessing procedure in \cite{chen2024flap} used to ensure counterfactual fairness in single-stage machine learning tasks. Outside of counterfactual fairness, a similar strict monotonicity assumption is imposed by \cite{athey2006nldid} to generalize the difference-in-differences method, and \cite{athey2006nldid} claims that its strict monotonicity assumption ``allows for a rich set of nonadditive structures that arise naturally in economic models''. Related monotonicity assumptions can also be found in quantile treatment effect literature \cite{chernozhukov2005ivqte, li2024dcqte}, though in these works the noise variable is often normalized to follow a $\text{Uniform}(0,1)$ distribution and directly interpreted as a latent rank. See Section \ref{sec:additional_monotonicity_discussion} for more discussion of Assumptions~\ref{ass:state_monotonicity} and~\ref{ass:reward_monotonicity}.

For $t \geq 0$, let $S_t^{Z \leftarrow z', (i)}(\bar{U}_t(h_t))$ denote the $i$-th component of $S_t^{Z \leftarrow z'}(\bar{U}_t(h_t))$. Theorems \ref{thm:state_matching} and \ref{thm:reward_matching} establish that conditional quantile matching can recover the counterfactual states and rewards.

\begin{theorem}[Quantile matching for counterfactual states]
\label{thm:state_matching}
Suppose Assumption~\ref{ass:state_monotonicity} hold, and consider an individual with observed history $h_t=(z, \bar{s}_{t}, \bar{a}_{t-1}, \bar{r}_{t-1})$. Then, for $t=0$ and each $i \in \{1, \dots, M\}$, the quantile level of $S_0^{Z \leftarrow z', (i)}(\bar{U}_0(h_0))$ with respect to the law of $f_0^{S^{(i)}}(z', U^{S^{(i)}}_0)$ does not vary by $z' \in \mathcal{Z}$. 

Similarly, for later time steps $t\geq 1$ and each $i \in \{1, \dots, M\}$, the quantile level of $S_t^{Z \leftarrow z', (i)}(\bar{U}_t(h_t))$ with respect to the law of $f^{S^{(i)}}(z', a_{t-1}, S_{t-1}^{Z \leftarrow z'}(\bar{U}_{t-1}(h_{t-1})), U_t^{S^{(i)}})$ does not vary by $z' \in \mathcal{Z}$.
\end{theorem}

\begin{theorem}[Quantile matching for counterfactual rewards]
\label{thm:reward_matching}
Suppose Assumption~\ref{ass:reward_monotonicity} hold, and consider an individual with observed history $h_t=(z, \bar{s}_{t}, \bar{a}_{t-1}, \bar{r}_{t-1})$. Then, for each time step $t \geq 0$, the quantile level of $R_t^{Z \leftarrow z'}(\bar{U}_{t+1}(h_{t+1}))$ with respect to the law of $f^{R}(z', a_{t}, S_{t}^{Z \leftarrow z'}(\bar{U}_{t}(h_{t})), U_t^{R})$ does not vary by $z' \in \mathcal{Z}$.
\end{theorem}

The implication is that the observed quantile level under the observed $z$ would remain identical had $Z$ taken a different value $z'$. This means we can use the observed data (comprised of individuals with $Z= z$ or $Z=z'$) to obtain estimated quantile levels and conditional quantile models for states and rewards. Then, for an individual with observed $Z=z$, one may shift to the estimated quantile model under a different $z'$, and use the estimated quantile level to retrieve the intended counterfactuals. See  Section~\ref{sec:smdm_intuition} of the Appendix for an example that illustrates the intuition. This observation inspires the sequential marginal distribution mapping (SMDM) method for estimating the counterfactual states and rewards, which are used to construct estimates $\hat{\tilde{S}}_t$ and $\hat{\tilde{R}}_t$ of the augmented states $\tilde{S}_t$ and rewards $\tilde{R}_t$, respectively. The full CFSMDM method combines data preprocessing (using SMDM) with policy learning. Due to limited space, pseudocodes of SMDM and CFSMDM are described in Algorithms \ref{alg:smdm} and \ref{alg:cfsmdm}, respectively, in Section~\ref{sec:pseudocode} of the Appendix.

\paragraph{Estimation of conditional quantiles.} Several methods can be employed to construct the conditional quantile models $\hat{Q}_0$, $\hat{Q}$, and $\hat{W}$ in Algorithm \ref{alg:smdm}. One option is the standard quantile regression \citep{koenker1978regression, koenker2005qrtextbook}. $\hat{Q}_0$ can be fitted using trajectory data from the initial time step, and $\hat{Q}$ and $\hat{W}$ can be fit using trajectory data from all time steps with $t \geq 1$ because the transition kernel is assumed to be stationary. However, standard quantile regression can be prone to model misspecification and might produce out-of-range estimates. Also, when standard quantile regression is used, it is likely necessary to fit multiple separate quantile regression models corresponding to a grid of quantile levels, which might lead to high computational cost when the state is high-dimensional. Related machine learning-based methods, including quantile regression forests \citep{nicolai2006qrforests} and quantile regression using neural networks \citep{taylor2000qrnn, oscar2022qrrelu}, may also be used to learn the conditional quantiles. In theory, these methods are often more flexible than standard quantile regression, and some of them, such as quantile regression using neural networks, allow the grid of quantile levels to be fit jointly in one single model. Other distributional regression techniques, such as distribution regression \citep{chernozhukov2013inference}, may also be employed.

\section{Theoretical Analysis}
\label{sec:theory}
In this section, we provide theoretical guarantees for the suboptimality gap and counterfactual unfairness control of policies learned using the proposed CFSMDM. In particular, the suboptimality gap for a history-dependent policy $\pi$ is defined as
\begin{equation*}
SG(\pi) = \mathbb{E}_{(s_t, a_t) \sim \pi^* P_t}\left[\sum_{t=0}^{\infty} \gamma^t r(s_t, a_t)\right] - \mathbb{E}_{(s_t, a_t) \sim \pi P_t}\left[\sum_{t=0}^{\infty} \gamma^t r(s_t, a_t)\right],
\end{equation*}
where $\pi^*$ is the optimal policy in the corresponding MDP, $\gamma\in (0,1)$ is the discount factor, and $r(\cdot, \cdot)$ is the reward function. For this theoretical analysis, we restrict to the case where the trajectories of all individuals in the training set have exactly $T$ time steps, and policy learning is performed using fitted Q iteration (FQI) with the Q function being estimated by ordinary least squares \citep{riedmiller2005fqi}. Also, we assume that the preprocessor training step and the policy learning step of CFSMDM are performed on two independent datasets, each containing $n$ independent and identically distributed oracle augmented transition tuples $\{\tilde{s}, a, \tilde{r}, \tilde{s}'\}$ sampled from the training trajectories. 

In the analysis that follows, let $\hat{\pi}$ be the stationary policy learned using CFSMDM. Also, define $\tilde{H}_t = \{\bar{\tilde{S}}_t, \bar{\tilde{R}}_{t-1}, \bar{A}_{t-1}\}$ for $t \geq 1$ and $\tilde{H}_t = \{\bar{\tilde{S}}_t\}$ for $t=0$, which denotes the history up to time $t$ in the MDP with the augmented state and reward. Note that, given the observed sensitive attribute $z^*$ and assuming the counterfactual state estimator and the quantile level estimator are both fixed, the estimated augmented state $\hat{\tilde{s}}_{t,n}$ at time $t$ can be written as a function of the augmented history up to time $t$ (i.e. $\hat{\tilde{s}}_{t,n} = g_{t,n}^{z^*}(\tilde{h}_t)$ for some function $g_{t,n}^{z^*}$ that depends on the observed sensitive attribute $z^*$, the counterfactual state estimator, the quantile level estimator, and the time step $t$). Therefore, we can write $\hat{\pi}(\hat{\tilde{s}}_{t,n}) = (\hat{\pi} \circ g_{t,n}^{z^*})(\tilde{h}_t)$. We define $\tilde{\pi}_{n}^{z^*} = \{\tilde{\pi}_{n,t}^{z^*}\}_{t \geq 0}$ where $\tilde{\pi}_{n,t}^{z^*} = (\hat{\pi} \circ g_{t,n}^{z^*})(\tilde{h}_t)$. Under this definition, $\tilde{\pi}_{n}^{z^*}$ is an alternative representation of the policy learned using CFSMDM, and it is a history-dependent non-stationary policy.  

As defined in Section \ref{subsec:proof_bounds} of the Appendix, $M_{t,n}$ is a random variable that denotes an upper bound of the augmented state estimation error at time $t$, and $\delta_{t,n}$ is a random variable that denotes a \textit{simultaneous} upper bound of both the augmented state and reward estimation errors \textit{up to} time $t$. Some additional assumptions needed for the bounds can be found in Section~\ref{subsubsec:assumptions} of the Appendix. Theorems \ref{thm:regret_bound} and \ref{thm:unfairness_bound} characterize the control of the suboptimality gap and the level of counterfactual unfairness, respectively.

\begin{theorem}[Suboptimality gap bound]
\label{thm:regret_bound}
Suppose Assumptions~\ref{ass:state_monotonicity}-\ref{ass:reward_monotonicity} and \ref{ass:space}-\ref{ass:lipschitz_conditional_quantile} hold. Also, suppose the trajectory for each individual in the training set has exactly $T$ time steps. Let $\kappa$ be any positive constant satisfying $\kappa \log{n} \geq 1$. For any $\omega \in \mathbb{N}$ and $z \in \mathcal{Z}$, with probability at least $\left(1 - n^{-\kappa} - d \exp(-n\lambda_0/8)\right) \mathbb{P}\left(\delta_{T,n} < \frac{\lambda_0}{4dL}\right)$, we have $\delta_{T,n} < \frac{\lambda_0}{4dL}$ and 
\begin{align*}
    SG(\tilde{\pi}_n^z) &\le 2\mathcal{B}L\frac{1-\gamma^{\omega}}{1-\gamma}\max_{t \leq \omega-1}M_{t,n} + 2\mathcal{B}L\frac{\gamma^{\omega}}{1-\gamma}M_{\tilde{\mathcal{S}}} \\
    &\quad + \frac{2C_1 d L R_{\max} \delta_{T,n}}{\lambda_0(\lambda_0 - 4dL\delta_{T,n})(1-\gamma)^3} + \frac{2C_2 d R_{\max} \kappa \log(n)}{(1-\gamma)^3 \lambda_0 \sqrt{n}} + \frac{2\gamma^B R_{\max}}{(1-\gamma)^2},
\end{align*}
where $C_1$ and $C_2$ are positive constants, $B$ is the number of FQI iterations, $\gamma \in (0, 1)$ is the discount factor, and $\mathcal{B}$, $L$, $d$, $\lambda_0$, $R_{\text{max}}$, and $M_{\tilde{S}}$ are positive constants defined in Assumptions \ref{ass:space} and \ref{ass:fqi}.
\end{theorem}

\begin{theorem}[Counterfactual unfairness bound]
\label{thm:unfairness_bound}
Suppose Assumptions~\ref{ass:state_monotonicity}-\ref{ass:reward_monotonicity} and \ref{ass:space}-\ref{ass:margin} hold. Also, suppose the trajectory for each individual in the training set has exactly $T$ time steps. Let $\xi_n = \left(\frac{C_1dLR_{max}\delta_{T,n}}{\lambda_0(\lambda_0 - 4dL\delta_{T,n})(1-\gamma)^2} + \frac{C_2dR_{max}\kappa \log(n)}{(1-\gamma)^2\lambda_0\sqrt{n}} + \frac{\gamma^BR_{max}}{1-\gamma}\right)I\left(\delta_{T,n} < \frac{\lambda_0}{4dL}\right)$ and $\eta_n = n^{-\kappa} + d\exp(-n\lambda_0/8)$, where $C_1, C_2$ are as defined in Theorem \ref{thm:regret_bound} and $\kappa$ is an arbitrary positive constant satisfying $\kappa \log{n} \geq 1$.. For any $z', z'' \in \mathcal{Z}$ and $t \geq 0$, 
\begin{align*}
    \mathbb{E}\left[\|\tilde{\pi}^{z'}_{t,n}(\tilde{h}_t) - \tilde{\pi}^{z''}_{t,n}(\tilde{h}_t)\|_2\right] \leq& \mathbb{E} \left[O(\xi_n^{\alpha}) + O((\mathcal{B}L\delta_{t,n})^{\alpha})\right] + \sqrt{2}\left[\mathbb{P}\left(\delta_{T,n} \geq \lambda_0/(4dL)\right) + \eta_n\right]
\end{align*}
where $O$ stands for the big-$O$ notation, $B$ is the number of FQI iterations, $\gamma \in (0, 1)$ is the discount factor, and $L$, $d$, $\lambda_0$, $R_{max}$, and $\alpha$ are positive constants defined in Assumptions \ref{ass:space}, \ref{ass:fqi}, and \ref{ass:margin}, respectively.
\end{theorem}

It is worth noting that increasing the training sample size $n$ generally tightens both the suboptimality gap bound and the counterfactual unfairness bound, provided that the counterfactual estimation errors decrease with $n$. See more discussions of the bounds in Section~\ref{sec:bound_discussion} of the Appendix.

\section{Numerical Experiments}
\label{sec:experiments}

In this section, we conduct simulation studies to compare the performance of CFSMDM relative to alternative methods using data generated from a known CMDP with nonadditive noise. We also present
additional numerical experiments in Section \ref{sec:additional_experiments} of the Appendix. In particular, in Section \ref{sec:additional_cf_value_experiment} of the Appendix, we conduct an additional numerical experiment using data generated from a CMDP with more restrictive additive noise scenarios where we show the proposed CFSMDM algorithm is comparable in optimal policy value and unfairness control relative to specially designed alternatives.

\paragraph{Methods to be compared.} The methods that are evaluated and compared in this section are:
\begin{itemize}[leftmargin=*]
    \item \textbf{Full:} The standard method that learns the optimal policy using all available covariates, including the sensitive attribute, as the state variable.
    \item \textbf{Unaware:} The method that learns the optimal policy using all available covariates except the sensitive attribute as the state variable. This method is an intuitive, yet often insufficient, way to achieve fairness.
    \item \textbf{Random:} A random policy that selects each available action with equal probability. This method should always be counterfactually fair because its decision-making is independent of both the training trajectory and deployment-time variables, but the quality of the decisions will suffer.
    \item \textbf{FLAP\_M:} The FLAP algorithm with (single-stage) marginal distribution mapping, as introduced in \cite{chen2024flap}, adapted to the RL setting. In our implementation of FLAP\_M, marginal distribution mapping is performed using empirical quantiles. 
    \item \textbf{ECOCF\_M:} The method introduced in \cite{wang2023adjusting}, adapted to the RL setting. In our implementation of ECOCF\_M, the counterfactual states are estimated using (single-stage) marginal distribution mapping with empirical quantiles. 
    \item \textbf{CFSDP:} The method introduced in \cite{wang2025counterfactuallyfairreinforcementlearning}. In our implementation of CFSDP, the transition kernel is learned using a neural network. 
    \item \textbf{CFSMDM:} The method proposed by this paper. In our implementation of CFSMDM, the quantiles of the conditional state and reward distributions are learned using standard quantile regression with a linear specification \citep{koenker1978regression, koenker2005qrtextbook}.
\end{itemize}
More details on the implementation of the baselines above can be found in Section~\ref{sec:details_experiments} of the Appendix. For all of the above methods except ``Random'', policy learning is performed using FQI \citep{riedmiller2005fqi}. The ``Random'' method is a policy on its own, so no policy learning is needed.

\paragraph{Evaluation metrics.} We are interested in the discounted cumulative reward (which we call ``policy value'' for brevity) and the level of counterfactual unfairness achieved by the policy learned using each of the above methods. In particular, the level of counterfactual fairness is assessed using the following CF metric proposed by \cite{wang2025counterfactuallyfairreinforcementlearning}:
\begin{equation*}
    \max_{z', z \in \mathcal{Z}} \frac{1}{NT} \sum_{i=1}^N \sum_{t=1}^T \mathbf{I} \left( A_t^{Z \leftarrow z'} (\bar{U}_t(h_{i,t})) \neq A_t^{Z \leftarrow z} (\bar{U}_t(h_{i,t})) \right).
\end{equation*}
Intuitively, this metric calculates the maximum empirical rate at which the policy of interest selects different actions for the same individual in the factual and counterfactual worlds across all individuals and time steps, with the maximum taken over all pairs of distinct sensitive attributes $z$ and $z'$ in $\mathcal{Z}$. This metric must lie in the interval $[0, 1]$, where $0$ and $1$ represent perfect and least empirical counterfactual fairness, respectively.

\paragraph{Experiment setup.} We generate data from a CMDP with nonadditive noise that satisfies Assumptions \ref{ass:state_monotonicity} and \ref{ass:reward_monotonicity}. The detailed CMDP setup can be found in Section~\ref{sec:details_experiments} in the Appendix. Let $\delta$ denote the strength of impact of the sensitive attribute on the state and reward variables. We run two experiments with different objectives:
\begin{itemize}[leftmargin=*]
    \item \textbf{Experiment 1:} This experiment evaluates how the policy value and the CF metric change with $N$, the number of individuals in the training dataset. In particular, for the training dataset, we fix $T=20, \delta=2.0$ and vary $N \in \{100, 200, 500, 1000, 2000\}$.
    \item \textbf{Experiment 2:} This experiment evaluates how the CF metric changes with $\delta$. In particular, for the training dataset, we fix $T=20, N=500$ and vary $\delta \in \{0.0, 1.0, 2.0, 3.0, 4.0\}$.
\end{itemize}
In both experiments, the training trajectories are generated using the ``Random'' policy as the behavioral policy. In particular, the sets of trajectories used for preprocessor training (if needed for the method of interest) and policy learning share the same sensitive attributes, while the remaining trajectory components are generated independently conditional on the shared sensitive attributes. The policy value resulting from the methods of interest is approximated by the discounted cumulative reward from a simulated trajectory with $10000$ individuals and $20$ horizons, generated from the known CMDP using the policy learned with the methods of interest. Similarly, the CF metric is calculated from a simulated dataset containing counterfactual trajectories of $10000$ individuals with $20$ horizons, generated from the known CMDP using the policy learned with the methods of interest. 

\paragraph{Results.} Figure~\ref{fig:results_simulation} summarizes the experiment results. Panels (a)-(c) summarize the results from Experiment 1, and Panel (d) summarizes the results from Experiment 2. In particular, panel (a) shows that CFSMDM outperforms all baselines other than ``Random'' in terms of fairness control. Moreover, although CFSMDM does not achieve perfect counterfactual fairness due to finite-sample counterfactual state and reward estimation errors, its CF metric decreases as the training sample size $N$ increases, which supports the consistency of CFSMDM. In panel (b), ``Full'' achieves the highest policy value because it uses the original states and rewards for reward maximization. CFSMDM achieves a lower policy value than ``Full'' and ``Unaware'', which demonstrates a tradeoff between fairness and policy value. Panel (c) further illustrates this tradeoff between fairness and policy value, as fairer methods generally result in lower policy values. Notably, FLAP\_M and ECOCF\_M are unfairer than CFSMDM due to their failure to account for the sequential nature of the underlying environment. Similarly, compared to CFSMDM, CFSDP is relatively unfair and achieves a lower policy value, possibly due to the violation of the additive noise assumption. CFSDP also appears unstable in both fairness control and policy value, as shown by the wide shaded bands in panels (a) and (b). In panel (d), CFSMDM remains relatively fair as we inject more unfairness into the underlying CMDP by increasing $\delta$, which shows that CFSMDM is robust against the level of unfairness present in the training trajectories.

\begin{figure}
    \centering
    \includegraphics[width=1\linewidth]{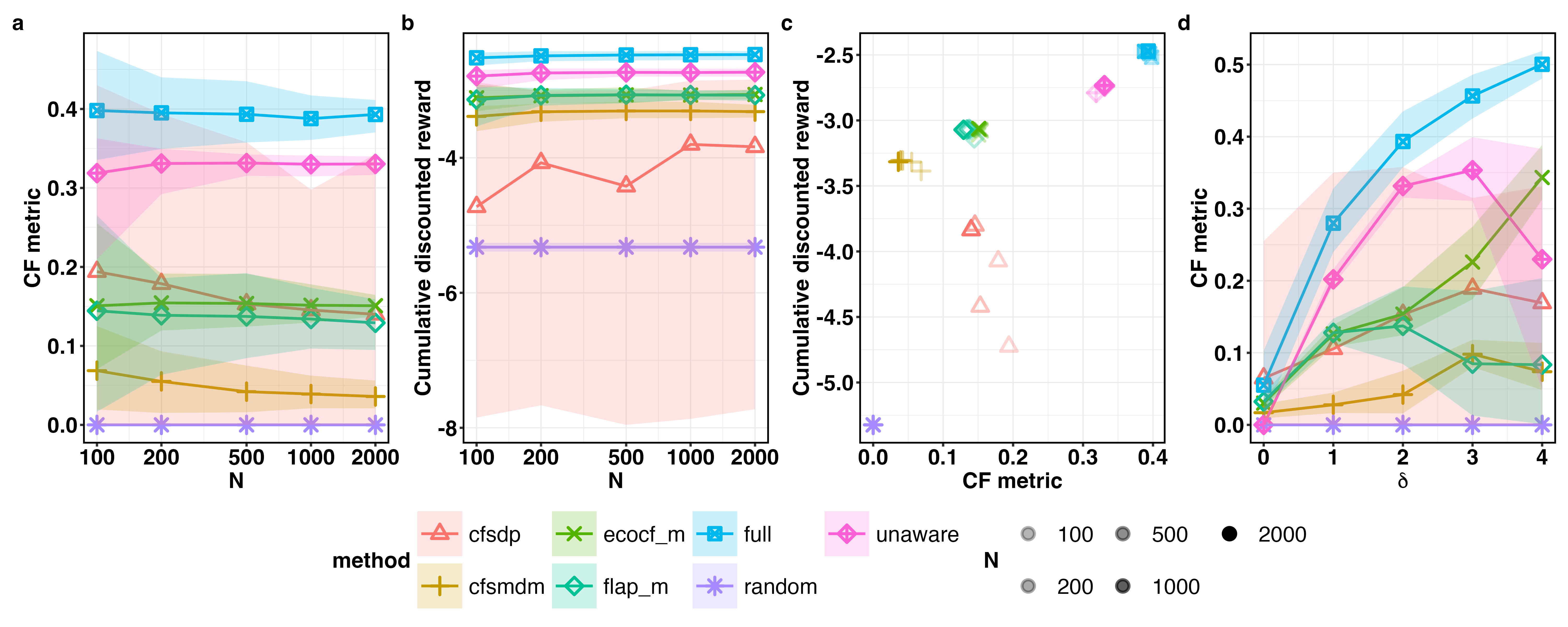}
    \caption{Experiment results under nonadditive noise (CMDP 1): (a) CF metric vs. Sample size ($N$); (b) Policy value vs. Sample size ($N$); (c) Policy value vs. CF metric; (d) CF metric vs. Size of effect of $Z$ on $S_t$ and $R_t$ ($\delta$). Results are aggregated over $50$ replications with different random seeds. Point represents the empirical mean. Shaded area represents the band between $0.025$ and $0.975$ empirical quantiles.}
    \label{fig:results_simulation}
\end{figure}

\section{Application to the PowerED Study on Reducing Opioid Misuse}
\label{sec:rda}

In this section, we use CFSMDM and other baseline methods to analyze a dataset from the PowerED study \citep{piette2023powerED}. The dataset consists of 207 patients, each participating in the intervention over a period of 12 weeks (one RL decision per week). We consider sex, education, age, and ethnicity, all binarized, as potential sensitive attributes, and we use weekly pain score and weekly pain interference score as the state variables. The reward is based on participants' weekly self-reports and calculated as \texttt{7 - weekly self-reported opioid medication risk score}.

The methods being run in this section are ``Full'', ``Unaware'', ``Random'', FLAP\_M, ECOCF\_M, CFSDP, and CFSMDM, all of which are as defined in Section~\ref{sec:experiments}. In particular, CFSDP estimates the transition kernel using a neural network, and CFSMDM estimates the conditional quantiles using quantile regression with a linear specification. We are interested in the counterfactual unfairness level and policy value achieved by CFSMDM and the baseline methods. Details on the implementation of the real data analysis can be found in Section~\ref{sec:details_rda_powered} of the Appendix.

Table \ref{tab:results_rda} summarizes the results of the analysis. For all candidate sensitive attributes, ``Full'' has the worst level of counterfactual fairness. Also, for all candidate sensitive attributes, ``Random'' is perfectly counterfactually fair but has made low-quality decisions resulting in low values. For most sensitive attributes, ``Unaware'', FLAP\_M, and ECOCF\_M have significantly better levels of counterfactual fairness than ``Full'' but worse than CFSDP and CFSMDM; the unfairness of ``Unaware'' reflects the fact that simply excluding the sensitive attribute from sequential decision-making (as was done in the clinical trial where this RL decision-process was used) is often insufficient for removing all the unfairness in the training trajectories. CFSMDM achieves the best counterfactual fairness for most sensitive attributes, and it noticeably outperforms CFSDP for ``education'' and ``ethnicity''. However, for ``sex'', CFSMDM is slightly more unfair than CFSDP and ``Unaware''. One possible explanation for this is that excluding the sensitive attribute from decision-making already removes most of the unfairness (e.g., because sex is uncorrelated with the included state variables), so the fairness improvement due to CFSMDM is outweighed by the unfairness introduced by finite-sample estimation errors in CFSMDM.

In Section \ref{sec:rda_ihs}, we present an additional real data analysis using the Intern Health Study dataset \cite{necamp2020ihs}. In Section \ref{sec:semisyn_experiments}, we conduct semi-synthetic experiments in generative model-based environments that mimic the actual environments underlying the PowerED and Intern Health Study datasets.

\begin{table}
    \centering
    \caption{CF metric and value of different methods when applied to the PowerED dataset. The analysis is repeated for $50$ times using different random seeds. The number without parentheses is the empirical mean, and the number inside the parentheses is the standard deviation. All values are rounded to two decimal points.}
    \label{tab:results_rda}
    \begin{tabular}{cccccc}
        \toprule
        Metric & Method & Education & Ethnicity & Sex & Age \\
        \midrule
        \multirow{7}{*}{Unfairness} & Full & 0.45 (0.14) & 0.46 (0.15) & 0.66 (0.15) & 0.35 (0.17) \\
                                    & Random & 0.00 (0.00) & 0.00 (0.00) & 0.00 (0.00) & 0.00 (0.00) \\
                                    & Unaware & 0.17 (0.04) & 0.17 (0.04) & 0.06 (0.02) & 0.13 (0.03) \\
                                    & CFSDP & 0.14 (0.04) & 0.18 (0.05) & 0.07 (0.02) & 0.10 (0.03) \\
                                    & ECOCF\_M & 0.16 (0.03) & 0.20 (0.04) & 0.14 (0.04) & 0.15 (0.03) \\
                                    & FLAP\_M & 0.17 (0.03) & 0.20 (0.04) & 0.14 (0.04) & 0.17 (0.04) \\
                                    & CFSMDM & 0.11 (0.03) & 0.10 (0.03) & 0.09 (0.05) & 0.10 (0.03) \\
        \midrule
        \multirow{7}{*}{Value}      & Full & 57.25 (0.56) & 57.26 (0.69) & 57.49 (0.52) & 57.35 (0.66) \\
                                    & Random & 57.04 (0.46) & 56.92 (0.66) & 57.02 (0.52) & 57.04 (0.55) \\
                                    & Unaware & 57.38 (0.54) & 57.36 (0.73) & 57.49 (0.54) & 57.47 (0.65) \\
                                    & CFSDP & 57.28 (0.56) & 57.19 (0.71) & 57.37 (0.57) & 57.36 (0.67) \\
                                    & ECOCF\_M & 57.30 (0.56) & 57.21 (0.73) & 57.44 (0.57) & 57.43 (0.63) \\
                                    & FLAP\_M & 57.33 (0.52) & 57.27 (0.72) & 57.48 (0.58) & 57.37 (0.65) \\
                                    & CFSMDM & 57.27 (0.51) & 57.26 (0.77) & 57.28 (0.55) & 57.41 (0.61) \\
        \bottomrule
    \end{tabular}
\end{table}

\section{Conclusion}
\label{sec:conclusion}

In this paper, we introduce a novel and general algorithm of sequential data preprocessing (CFSMDM) which, when used in tandem with off-the-shelf RL policy learning algorithms,  achieves CF. Specifically, we first estimate all the counterfactual states and rewards by matching the quantiles of conditional counterfactual distributions. We use the estimates to construct a preprocessed trajectory dataset, where the vector of all estimated counterfactual states serves as the augmented state variable and a weighted sum of all estimated counterfactual rewards serves as the augmented reward variable. We can then perform offline policy learning using the preprocessed trajectory dataset with the augmented states and rewards. Theoretically, we derive suboptimality gap and counterfactual unfairness bounds for the policy learned using data preprocessed by CFSMDM.

Our CFSMDM algorithm, together with the CFSDP algorithm by \cite{wang2025counterfactuallyfairreinforcementlearning}, provides fairness-aware researchers and professionals with tools for ensuring CF in RL with distinct advantages and drawbacks. Compared to CFSDP, CFSMDM requires strictly weaker assumptions and is therefore more general and robust. On the other hand, CFSDP tends to be easier to implement and less computationally heavy than CFSMDM: CFSDP only estimates the means of the conditional counterfactual state and reward distributions, and the mean of the whole state vector can usually be jointly estimated; in contrast, CFSMDM requires estimating a grid of quantiles, and the grid of quantiles for each dimension of the state vector usually needs to be estimated separately. 

\paragraph{Limitations.} This paper admits some limitations. First, the CFSMDM method requires (i) that Assumptions \ref{ass:state_monotonicity} and \ref{ass:reward_monotonicity} hold and (ii) that the components of the state vector and the reward are each affected by a scalar-valued noise variable, which might not be true in some settings. Second, CFSMDM usually requires fitting separate sets of quantile models for each dimension of the state vector; as a result, CFSMDM might be computationally heavier in general than some alternatives such as CFSDP when the state is high-dimensional. Third, since the counterfactual states and rewards are unknown in real data, our real data analyses compute the CF metric using the counterfactual states and rewards estimated by the SMDM procedure; consequently, the accuracy of the reported CF metric depends on the quality of the counterfactual estimation.

\paragraph{Directions for future research.} We identify several potentially meaningful directions for further extending the work presented in this paper. First, future work might further relax Assumptions \ref{ass:state_monotonicity} and \ref{ass:reward_monotonicity}. Second, the current CFSMDM algorithm only accommodates finite action spaces, and future research might focus on extending CFSMDM to settings with countably infinite or continuous action spaces. Third, extending CFSMDM to settings with resource constraints might be meaningful for many real-world application scenarios.

\section*{Acknowledgments and Disclosure of Funding}

We gratefully acknowledge the \href{https://cmcd.sph.umich.edu/research-program-areas/using-advances-in-mobile-health-to-make-opioid-use-safer/}{PowerED} and \href{https://www.internhealthstudy.org/}{Intern Health Study} participants, study investigators and research staff for their time and contributions, without whom this research would not have been possible. We also thank the study PIs (Amy Bohnert, Srijan Sen) and data lead (Sean Newman, Yu Fang) for making the data available. This work was partly supported through a Patient-Centered Outcomes Research Institute (PCORI) Project Program Award (ME-2025C1-44006). All statements in this report, including its findings and conclusions, are solely those of the authors and do not necessarily represent the views of the PCORI, its Board of Governors or Methodology Committee. The authors declare no competing interests.

\section*{Data and Code Availability Statement}

An anonymized and limited copy of the PowerED dataset that motivates and supports the findings in this paper is available upon reasonable request, subject to the limitations outlined in the study's Human Subjects Approval and written informed consent statement for study participants. The Intern Health Study dataset used in the additional real data analysis in Section \ref{sec:rda_ihs} is also available upon reasonable request. The Python code for the simulations and the data analysis is available as supplementary material to this paper. The same code can also be found on GitHub at \url{https://github.com/JianhanZhang/CFSMDM_RL}.







\bibliographystyle{apalike}
\bibliography{references}

@misc{hallak2015contextualmarkovdecisionprocesses,
      title={Contextual Markov Decision Processes}, 
      author={Assaf Hallak and Di Castro, Dotan and Shie Mannor},
      year={2015},
      eprint={1502.02259},
      archivePrefix={arXiv},
      primaryClass={stat.ML},
      url={https://arxiv.org/abs/1502.02259}, 
      howpublished={\textit{arXiv preprint arXiv:1502.02259}}
}

@misc{wang2025counterfactuallyfairreinforcementlearning,
      title={Counterfactually Fair Reinforcement Learning via Sequential Data Preprocessing}, 
      author={Jitao Wang and Chengchun Shi and John D. Piette and Joshua R. Loftus and Donglin Zeng and Zhenke Wu},
      year={2025},
      eprint={2501.06366},
      archivePrefix={arXiv},
      primaryClass={stat.ML},
      url={https://arxiv.org/abs/2501.06366v2}, 
      howpublished={\textit{arXiv preprint arXiv:2501.06366v2}}
}

@misc{kusner2018counterfactualfairness,
      title={Counterfactual Fairness}, 
      author={Matt J. Kusner and Joshua R. Loftus and Chris Russell and Ricardo Silva},
      year={2018},
      eprint={1703.06856},
      archivePrefix={arXiv},
      primaryClass={stat.ML},
      url={https://arxiv.org/abs/1703.06856}, 
      howpublished={\textit{arXiv preprint arXiv:1703.06856}}
}

@book{pearl2016causal,
  title={Causal inference in statistics: A primer},
  author={Pearl, Judea and Glymour, Madelyn and Jewell, Nicholas P},
  year={2016},
  publisher={John Wiley \& Sons}
}

@article{koenker1978regression,
 ISSN = {00129682, 14680262},
 URL = {http://www.jstor.org/stable/1913643},
 author = {Roger Koenker and Gilbert Bassett},
 journal = {Econometrica},
 number = {1},
 pages = {33--50},
 publisher = {[Wiley, Econometric Society]},
 title = {Regression Quantiles},
 urldate = {2026-07-18},
 volume = {46},
 year = {1978}
}

@article{taylor2000qrnn,
author = {Taylor, James W.},
title = {A quantile regression neural network approach to estimating the conditional density of multiperiod returns},
journal = {Journal of Forecasting},
volume = {19},
number = {4},
pages = {299-311},
doi = {https://doi.org/10.1002/1099-131X(200007)19:4<299::AID-FOR775>3.0.CO;2-V},
url = {https://onlinelibrary.wiley.com/doi/abs/10.1002/1099-131X%28200007%2919%3A4%3C299%3A%3AAID-FOR775%3E3.0.CO%3B2-V},
eprint = {https://onlinelibrary.wiley.com/doi/pdf/10.1002/1099-131X%28200007%2919%3A4%3C299%3A%3AAID-FOR775%3E3.0.CO%3B2-V},
year = {2000}
}

@article{oscar2022qrrelu,
  author  = {Oscar Hernan Madrid-Padilla and Wesley Tansey and Yanzhen Chen},
  title   = {Quantile regression with ReLU Networks: Estimators and minimax rates},
  journal = {Journal of Machine Learning Research},
  year    = {2022},
  volume  = {23},
  number  = {247},
  pages   = {1--42},
  url     = {http://jmlr.org/papers/v23/21-0309.html}
}

@article{nicolai2006qrforests,
  author  = {Nicolai Meinshausen},
  title   = {Quantile Regression Forests},
  journal = {Journal of Machine Learning Research},
  year    = {2006},
  volume  = {7},
  number  = {35},
  pages   = {983--999},
  url     = {http://jmlr.org/papers/v7/meinshausen06a.html}
}

@article{chernozhukov2013inference,
author = {Chernozhukov, Victor and Fernández-Val, Iván and Melly, Blaise},
title = {Inference on Counterfactual Distributions},
journal = {Econometrica},
volume = {81},
number = {6},
pages = {2205-2268},
doi = {https://doi.org/10.3982/ECTA10582},
url = {https://onlinelibrary.wiley.com/doi/abs/10.3982/ECTA10582},
eprint = {https://onlinelibrary.wiley.com/doi/pdf/10.3982/ECTA10582},
year = {2013}
}

@article{chen2024flap,
author = {Haoyu Chen and Wenbin Lu and Rui Song and Pulak Ghosh},
title = {On Learning and Testing of Counterfactual Fairness through Data Preprocessing},
journal = {Journal of the American Statistical Association},
volume = {119},
number = {546},
pages = {1286--1296},
year = {2024},
publisher = {Taylor \& Francis},
doi = {10.1080/01621459.2023.2186885},


URL = { 
    
        https://doi.org/10.1080/01621459.2023.2186885
    
    

},
eprint = { 
    
        https://doi.org/10.1080/01621459.2023.2186885
    
    

}

}

@book{koenker2005qrtextbook, place={Cambridge}, series={Econometric Society Monographs}, title={Quantile Regression}, publisher={Cambridge University Press}, author={Koenker, Roger}, year={2005}, collection={Econometric Society Monographs}}

@InProceedings{riedmiller2005fqi,
author="Riedmiller, Martin",
editor="Gama, Jo{\~a}o
and Camacho, Rui
and Brazdil, Pavel B.
and Jorge, Al{\'i}pio M{\'a}rio
and Torgo, Lu{\'i}s",
title="Neural Fitted Q Iteration -- First Experiences with a Data Efficient Neural Reinforcement Learning Method",
booktitle="Machine Learning: ECML 2005",
year="2005",
publisher="Springer Berlin Heidelberg",
address="Berlin, Heidelberg",
pages="317--328",
isbn="978-3-540-31692-3"
}

@Article{piette2023powerED,
author="Piette, John D
and Thomas, Laura
and Newman, Sean
and Marinec, Nicolle
and Krauss, Joel
and Chen, Jenny
and Wu, Zhenke
and Bohnert, Amy S B",
title="An Automatically Adaptive Digital Health Intervention to Decrease Opioid-Related Risk While Conserving Counselor Time: Quantitative Analysis of Treatment Decisions Based on Artificial Intelligence and Patient-Reported Risk Measures",
journal="J Med Internet Res",
year="2023",
month="Jul",
day="11",
volume="25",
pages="e44165",
issn="1438-8871",
doi="10.2196/44165",
url="https://www.jmir.org/2023/1/e44165",
url="https://doi.org/10.2196/44165",
url="http://www.ncbi.nlm.nih.gov/pubmed/37432726"
}

@InProceedings{le2019Batch,
  title = 	 {Batch Policy Learning under Constraints},
  author =       {Le, Hoang and Voloshin, Cameron and Yue, Yisong},
  booktitle = 	 {Proceedings of the 36th International Conference on Machine Learning},
  pages = 	 {3703--3712},
  year = 	 {2019},
  editor = 	 {Chaudhuri, Kamalika and Salakhutdinov, Ruslan},
  volume = 	 {97},
  series = 	 {Proceedings of Machine Learning Research},
  month = 	 {09--15 Jun},
  publisher =    {PMLR},
  url = 	 {https://proceedings.mlr.press/v97/le19a.html}
}

@article{hollingshead2016pain,
title = {The Pain Experience of Hispanic Americans: A Critical Literature Review and Conceptual Model},
journal = {The Journal of Pain},
volume = {17},
number = {5},
pages = {513-528},
year = {2016},
issn = {1526-5900},
doi = {https://doi.org/10.1016/j.jpain.2015.10.022},
url = {https://www.sciencedirect.com/science/article/pii/S1526590015009591},
author = {Nicole A. Hollingshead and Leslie Ashburn-Nardo and Jesse C. Stewart and Adam T. Hirsh}
}

@article{
wang2023adjusting,
title={Adjusting Machine Learning Decisions for Equal Opportunity and Counterfactual Fairness},
author={Yixin Wang and Dhanya Sridhar and David Blei},
journal={Transactions on Machine Learning Research},
issn={2835-8856},
year={2023},
url={https://openreview.net/forum?id=P6NcRPb13w},
note={}
}

@misc{bian2026double,
      title={Double Fairness Policy Learning: Integrating Action Fairness and Outcome Fairness in Decision-making}, 
      author={Zeyu Bian and Lan Wang and Chengchun Shi and Zhengling Qi},
      year={2026},
      eprint={2601.19186},
      archivePrefix={arXiv},
      primaryClass={stat.ML},
      url={https://arxiv.org/abs/2601.19186}, 
      howpublished={\textit{arXiv preprint arXiv:2601.19186}}
}

@misc{distefano2020counterfactual,
      title={Counterfactual fairness: removing direct effects through regularization}, 
      author={Di Stefano, Pietro G. and James M. Hickey and Vlasios Vasileiou},
      year={2020},
      eprint={2002.10774},
      archivePrefix={arXiv},
      primaryClass={cs.AI},
      url={https://arxiv.org/abs/2002.10774}, 
      howpublished={\textit{arXiv preprint arXiv:2002.10774}}
}

@inproceedings{zuo2022counterfactual,
author = {Zuo, Aoqi and Wei, Susan and Liu, Tongliang and Han, Bo and Zhang, Kun and Gong, Mingming},
title = {Counterfactual fairness with partially known causal graph},
year = {2022},
isbn = {9781713871088},
publisher = {Curran Associates Inc.},
address = {Red Hook, NY, USA},
booktitle = {Proceedings of the 36th International Conference on Neural Information Processing Systems},
articleno = {91},
numpages = {15},
location = {New Orleans, LA, USA},
series = {NIPS '22}
}

@inproceedings{
chen2025cflb,
title={Causal Logistic Bandits with Counterfactual Fairness Constraints},
author={Jiajun Chen and Jin Tian and Christopher John Quinn},
booktitle={Forty-second International Conference on Machine Learning},
year={2025},
url={https://openreview.net/forum?id=1N4Y0Yj8th}
}

@article{huang2022achieving, title={Achieving Counterfactual Fairness for Causal Bandit}, volume={36}, url={https://ojs.aaai.org/index.php/AAAI/article/view/20653}, DOI={10.1609/aaai.v36i6.20653}, abstractNote={In online recommendation, customers arrive in a sequential and stochastic manner from an underlying distribution and the online decision model recommends a chosen item for each arriving individual based on some strategy. We study how to recommend an item at each step to maximize the expected reward while achieving user-side fairness for customers, i.e., customers who share similar profiles will receive a similar reward regardless of their sensitive attributes and items being recommended. By incorporating causal inference into bandits and adopting soft intervention to model the arm selection strategy, we first propose the d-separation based UCB algorithm (D-UCB) to explore the utilization of the d-separation set in reducing the amount of exploration needed to achieve low cumulative regret. Based on that, we then propose the fair causal bandit (F-UCB) for achieving the counterfactual individual fairness. Both theoretical analysis and empirical evaluation demonstrate effectiveness of our algorithms.}, number={6}, journal={Proceedings of the AAAI Conference on Artificial Intelligence}, author={Huang, Wen and Zhang, Lu and Wu, Xintao}, year={2022}, month={Jun.}, pages={6952–6959} }

@inproceedings{damour2020fairness,
author = {D'Amour, Alexander and Srinivasan, Hansa and Atwood, James and Baljekar, Pallavi and Sculley, D. and Halpern, Yoni},
title = {Fairness is not static: deeper understanding of long term fairness via simulation studies},
year = {2020},
isbn = {9781450369367},
publisher = {Association for Computing Machinery},
address = {New York, NY, USA},
url = {https://doi.org/10.1145/3351095.3372878},
doi = {10.1145/3351095.3372878},
booktitle = {Proceedings of the 2020 Conference on Fairness, Accountability, and Transparency},
pages = {525–534},
numpages = {10},
location = {Barcelona, Spain},
series = {FAT* '20}
}

@InProceedings{creager2020causal4fairness,
  title = 	 {Causal Modeling for Fairness In Dynamical Systems},
  author =       {Creager, Elliot and Madras, David and Pitassi, Toniann and Zemel, Richard},
  booktitle = 	 {Proceedings of the 37th International Conference on Machine Learning},
  pages = 	 {2185--2195},
  year = 	 {2020},
  editor = 	 {Daumé III, Hal and Singh, Aarti},
  volume = 	 {119},
  series = 	 {Proceedings of Machine Learning Research},
  month = 	 {13--18 Jul},
  publisher =    {PMLR},
  url = 	 {https://proceedings.mlr.press/v119/creager20a.html}
}

@InProceedings{liu2018delayed,
  title = 	 {Delayed Impact of Fair Machine Learning},
  author =       {Liu, Lydia T. and Dean, Sarah and Rolf, Esther and Simchowitz, Max and Hardt, Moritz},
  booktitle = 	 {Proceedings of the 35th International Conference on Machine Learning},
  pages = 	 {3150--3158},
  year = 	 {2018},
  editor = 	 {Dy, Jennifer and Krause, Andreas},
  volume = 	 {80},
  series = 	 {Proceedings of Machine Learning Research},
  month = 	 {10--15 Jul},
  publisher =    {PMLR},
  url = 	 {https://proceedings.mlr.press/v80/liu18c.html}
}

@article{chernozhukov2010rearrangement,
author = {Chernozhukov, Victor and Fernández-Val, Iván and Galichon, Alfred},
title = {Quantile and Probability Curves Without Crossing},
journal = {Econometrica},
volume = {78},
number = {3},
pages = {1093-1125},
doi = {https://doi.org/10.3982/ECTA7880},
url = {https://onlinelibrary.wiley.com/doi/abs/10.3982/ECTA7880},
eprint = {https://onlinelibrary.wiley.com/doi/pdf/10.3982/ECTA7880},
year = {2010}
}

@misc{kingma2015adam,
      title={Adam: A Method for Stochastic Optimization}, 
      author={Diederik P. Kingma and Jimmy Ba},
      year={2017},
      eprint={1412.6980},
      archivePrefix={arXiv},
      primaryClass={cs.LG},
      url={https://arxiv.org/abs/1412.6980}, 
      howpublished={\textit{arXiv preprint arXiv:1412.6980}}
}

@inproceedings{seabold2010statsmodels,
  title={statsmodels: Econometric and statistical modeling with python},
  author={Seabold, Skipper and Perktold, Josef},
  booktitle={9th Python in Science Conference},
  year={2010},
}

@inproceedings{kakade2002approximate,
author = {Kakade, Sham and Langford, John},
title = {Approximately Optimal Approximate Reinforcement Learning},
year = {2002},
isbn = {1558608737},
publisher = {Morgan Kaufmann Publishers Inc.},
address = {San Francisco, CA, USA},
booktitle = {Proceedings of the Nineteenth International Conference on Machine Learning},
pages = {267–274},
numpages = {8},
series = {ICML '02}
}

@article{chernozhukov2005ivqte,
author = {Chernozhukov, Victor and Hansen, Christian},
title = {An IV Model of Quantile Treatment Effects},
journal = {Econometrica},
volume = {73},
number = {1},
pages = {245-261},
doi = {https://doi.org/10.1111/j.1468-0262.2005.00570.x},
url = {https://onlinelibrary.wiley.com/doi/abs/10.1111/j.1468-0262.2005.00570.x},
eprint = {https://onlinelibrary.wiley.com/doi/pdf/10.1111/j.1468-0262.2005.00570.x},
year = {2005}
}

@article{li2024dcqte,
author = {Ting Li and Chengchun Shi and Zhaohua Lu and Yi Li and Hongtu Zhu},
title = {Evaluating Dynamic Conditional Quantile Treatment Effects with Applications in Ridesharing},
journal = {Journal of the American Statistical Association},
volume = {119},
number = {547},
pages = {1736--1750},
year = {2024},
publisher = {Taylor \& Francis},
doi = {10.1080/01621459.2024.2314316},


URL = { 
    
        https://doi.org/10.1080/01621459.2024.2314316
    
    

},
eprint = { 
    
        https://doi.org/10.1080/01621459.2024.2314316
    
    

}

}

@article{perdomo2023linearOPE,
  author  = {Juan C. Perdomo and Akshay Krishnamurthy and Peter Bartlett and Sham Kakade},
  title   = {A Complete Characterization of Linear Estimators for Offline Policy Evaluation},
  journal = {Journal of Machine Learning Research},
  year    = {2023},
  volume  = {24},
  number  = {284},
  pages   = {1--50},
  url     = {http://jmlr.org/papers/v24/22-0341.html}
}

@InProceedings{chen2019information,
  title = 	 {Information-Theoretic Considerations in Batch Reinforcement Learning},
  author =       {Chen, Jinglin and Jiang, Nan},
  booktitle = 	 {Proceedings of the 36th International Conference on Machine Learning},
  pages = 	 {1042--1051},
  year = 	 {2019},
  editor = 	 {Chaudhuri, Kamalika and Salakhutdinov, Ruslan},
  volume = 	 {97},
  series = 	 {Proceedings of Machine Learning Research},
  month = 	 {09--15 Jun},
  publisher =    {PMLR},
  url = 	 {https://proceedings.mlr.press/v97/chen19e.html}
}

@misc{hu2023fastratesregretoffline,
      title={Fast Rates for the Regret of Offline Reinforcement Learning}, 
      author={Yichun Hu and Nathan Kallus and Masatoshi Uehara},
      year={2023},
      eprint={2102.00479},
      archivePrefix={arXiv},
      primaryClass={cs.LG},
      url={https://arxiv.org/abs/2102.00479}, 
      howpublished={\textit{arXiv preprint arXiv:2102.00479}}
}

@inproceedings{kumar2020cql,
 author = {Kumar, Aviral and Zhou, Aurick and Tucker, George and Levine, Sergey},
 booktitle = {Advances in Neural Information Processing Systems},
 editor = {H. Larochelle and M. Ranzato and R. Hadsell and M.F. Balcan and H. Lin},
 pages = {1179--1191},
 publisher = {Curran Associates, Inc.},
 title = {Conservative Q-Learning for Offline Reinforcement Learning},
 url = {https://proceedings.neurips.cc/paper_files/paper/2020/file/0d2b2061826a5df3221116a5085a6052-Paper.pdf},
 volume = {33},
 year = {2020}
}

@Article{necamp2020ihs,
author="NeCamp, Timothy
and Sen, Srijan
and Frank, Elena
and Walton, Maureen A
and Ionides, Edward L
and Fang, Yu
and Tewari, Ambuj
and Wu, Zhenke",
title="Assessing Real-Time Moderation for Developing Adaptive Mobile Health Interventions for Medical Interns: Micro-Randomized Trial",
journal="J Med Internet Res",
year="2020",
month="Mar",
day="31",
volume="22",
number="3",
pages="e15033",
issn="1438-8871",
doi="10.2196/15033",
url="http://www.jmir.org/2020/3/e15033/",
url="https://doi.org/10.2196/15033"
}

@article{athey2006nldid,
author = {Athey, Susan and Imbens, Guido W.},
title = {Identification and Inference in Nonlinear Difference-in-Differences Models},
journal = {Econometrica},
volume = {74},
number = {2},
pages = {431-497},
doi = {https://doi.org/10.1111/j.1468-0262.2006.00668.x},
url = {https://onlinelibrary.wiley.com/doi/abs/10.1111/j.1468-0262.2006.00668.x},
eprint = {https://onlinelibrary.wiley.com/doi/pdf/10.1111/j.1468-0262.2006.00668.x},
year = {2006}
}

@article{hsu2019monotonicity, title={TESTING GENERALIZED REGRESSION MONOTONICITY}, volume={35}, DOI={10.1017/S0266466618000439}, number={6}, journal={Econometric Theory}, author={Hsu, Yu-Chin and Liu, Chu-An and Shi, Xiaoxia}, year={2019}, pages={1146–1200}}

@inproceedings{
    patki2016sdv,
    title={The Synthetic data vault},
    author={Patki, Neha and Wedge, Roy and Veeramachaneni, Kalyan},
    booktitle={IEEE International Conference on Data Science and Advanced Analytics (DSAA)},
    year={2016},
    pages={399-410},
    doi={10.1109/DSAA.2016.49},
    month={Oct}
}


\newpage
\appendix


\setcounter{figure}{0}
\renewcommand{\thefigure}{A\arabic{figure}}
\setcounter{table}{0}
\renewcommand{\thetable}{A\arabic{table}}

\setcounter{theorem}{0}
\renewcommand{\thetheorem}{A\arabic{theorem}}
\setcounter{assumption}{0}
\renewcommand{\theassumption}{A\arabic{assumption}}

\setcounter{proposition}{0}
\renewcommand{\theproposition}{A\arabic{proposition}}

\setcounter{remark}{0}
\renewcommand{\theremark}{A\arabic{remark}}

\setcounter{lemma}{0}
\renewcommand{\thelemma}{A\arabic{lemma}}

\section{Table of Key Notation}

\begin{table}[H]
\centering
\caption{Summary of Key Notation}
\label{tab:notation}
\renewcommand{\arraystretch}{1.5} 
\small 
\begin{tabularx}{\textwidth}{@{} l X @{}} 
\toprule
\textbf{Symbol} & \textbf{Description} \\
\midrule

\rowcolor{gray!10} \multicolumn{2}{l}{\textit{Problem Setup \& Data-generating Mechanism}} \\
$Z, \mathcal{Z}$ & Sensitive attribute and its space $\mathcal{Z}=\{z^{(1)},...,z^{(K)}\}$. \\
$S_t, A_t, R_t$ & State, action, and reward at time $t$, with $S_t=(S_t^{(1)}, \dots, S_t^{(M)})$. \\
$\bar{S}_t, \bar{A}_t, \bar{R}_t$ & Historical state, action, and reward up to time $t$. \\
$U_t^S, U_t^A, U_t^R$ & Exogenous noise variables for state, action, and reward at time $t$. \\
$\bar{U}_t^S, \bar{U}_t^A, \bar{U}_t^R$ & Historical exogenous noise variables for state, action, and reward up to time $t$. \\
$H_t$ & $\{Z, \bar{S}_t, \bar{A}_{t-1}, \bar{R}_{t-1}\}$, history up to time $t$. \\

\midrule
\rowcolor{gray!10} \multicolumn{2}{l}{\textit{Counterfactual Fairness \& SCM}} \\
$S_t^{Z \leftarrow z}(\overline{U}_t(h_t))$ & Counterfactual state at time $t$ for an individual with observed history $h_t$ had their sensitive attribute been $z$ \textit{while experiencing the same action sequence as the observed action sequence $\bar{a}_{t-1}$}; sometimes written as $s_t^z$ for simplicity. \\
$A_t^{Z \leftarrow z}(\overline{U}_t(h_t))$ & Counterfactual action at time $t$ for an individual with observed history $h_t$ had their sensitive attribute been $z$ \textit{while experiencing the same action sequence as the observed action sequence $\bar{a}_{t-1}$}. \\
$R_t^{Z \leftarrow z}(\overline{U}_{t+1}(h_{t+1}))$ & Counterfactual reward at time $t$ for an individual with observed history $h_{t+1}$ had their sensitive attribute been $z$ \textit{while experiencing the same action sequence as the observed action sequence $\bar{a}_{t}$}; sometimes written as $r_t^z$ for simplicity. \\
$\hat{S}_t^{Z \leftarrow z}(\overline{U}_t(h_t))$, $\hat{R}_t^{Z \leftarrow z}(\overline{U}_{t+1}(h_{t+1}))$ & Estimated counterfactual states and rewards, respectively, under $Z=z$ for an individual with observed history $h_{t+1}$; sometimes written as $\hat{s}_t^z$ and $\hat{r}_t^z$, respectively, for simplicity. \\
$\tilde{s}_t$ & $\big(S_t^{Z \leftarrow z^{(1)}}(\overline{U}_t(h_t)), \dots, S_t^{Z \leftarrow z^{(K)}}(\overline{U}_t(h_t))\big)$, augmented state invariant to the observed sensitive attribute. \\
$\tilde{r}_t$ & $\sum_{k=1}^K \mathbb{P}(Z=z^{(k)})R_t^{Z \leftarrow z^{(k)}}(\overline{U}_{t+1}(h_{t+1}))$, augmented reward invariant to the observed sensitive attribute. \\
$\hat{\tilde{s}}_t, \hat{\tilde{r}}_t$ & Estimated augmented state and reward used for policy learning. \\

\midrule
\rowcolor{gray!10} \multicolumn{2}{l}{\textit{Proof of CFSMDM Performance Guarantees}} \\
$\hat{\pi}$ & Policy learned by CFSMDM, written as a stationary policy with respect to the estimated augmented state $\hat{\tilde{s}}_t$. \\
$\tilde{\pi}_n^{z}$ & Policy learned by CFSMDM, written as a history-dependent policy with respect to the true augmented state $\tilde{s}_t$. \\
$Q_0^{(i)}(\tau, z)$, $\hat{Q}_{0,n}^{(i)}(\tau, z)$ & True and estimated $\tau$-th conditional quantile of $S_0^{(i)}$ given $Z=z$. \\
$Q^{(i)}(\tau, z, s, a)$, $\hat{Q}_n^{(i)}(\tau, z, s, a)$ & True and estimated $\tau$-th conditional quantile of $S_t^{(i)}$ for $t \ge 1$ given $Z=z, S_{t-1}=s, A_{t-1}=a$. \\
$W(\tau, z, s, a)$, $\hat{W}_n(\tau, z, s, a)$ & True and estimated $\tau$-th conditional quantile of the reward $R_t$ given $Z=z, S_{t}=s, A_{t}=a$. \\
$\tau_{j,t}^{(i)}, \theta_{j,t}$ & True quantile levels of observed state ($i$-th dimension) and reward, respectively. \\
$\hat{\tau}_{j,t,n}^{(i)}, \hat{\theta}_{j,t,n}$ & Estimated quantile levels of observed state ($i$-th dimension) and reward, respectively. \\
$\|\cdot\|_{\infty}$ & Supremum norm (when applied to functions) or max norm (when applied to vectors). \\
\bottomrule
\end{tabularx}
\end{table}

\section{Comparison of Counterfactual Fairness and Other Fairness Metrics}
\label{sec:cf_comparison}
To further compare counterfactual fairness (CF) with other fairness criteria, notably demographic parity and equal opportunity, we present the following single-stage decision-making example \citep{wang2023adjusting}. Consider the setting where a university uses ML models to assist with making undergraduate admissions decisions. For simplicity, assume an applicant's features only include the biological sex (i.e. female or male), denoted by $S$, and entrance exam score, denoted by $E$. Suppose biologically female students tend to face systemic barriers in early education that negatively impact their entrance exam scores, which is a manifestation of historical bias, and the university wants to enforce some algorithmic fairness criteria to ensure that the admissions decision is fair to applicants of any biological sex. We discuss below what different fairness criteria would enforce in this setting.

\paragraph{Demographic parity:} If demographic parity is enforced, then the university would ensure that the proportion of female applicants being admitted is equal to the proportion of male applicants being admitted. Mathematically, if $\hat{Y}=1$ stands for "offered admission" and $\hat{Y}=0$ stands for "denied admission", then demographic parity guarantees
\begin{equation*}
    \mathbb{P}(\hat{Y}=1 | S=female) = \mathbb{P}(\hat{Y}=1 | S=male).
\end{equation*}
Demographic parity seeks to match the statistical proportions without considering the applicants' actual qualifications, so it could cause unfairness if applicants in one group are indeed more qualified than those from the other group. Also, demographic parity only focuses on fairness at the group level. As a result, the university can admit only females who faced fewer barriers, while leaving out the others who faced more barriers, without violating demographic parity; this might fail to enforce fairness as intended.

\paragraph{Equal opportunity:} If equal opportunity is enforced, then the university would ensure that among truly qualified applicants, the proportion of females being admitted is equal to the proportion of males being admitted. Suppose $\hat{Y}=1$ stands for "offered admission" and $\hat{Y}=0$ stands for "denied admission". Also, suppose $Y=1$ stands for "qualified for admission" and $Y=0$ stands for "not qualified for admission". Then, mathematically, equal opportunity guarantees 
\begin{equation*}
    \mathbb{P}(\hat{Y}=1 | S=female, Y=1) = \mathbb{P}(\hat{Y}=1 | S=male, Y=1).
\end{equation*}
While equal opportunity takes into account the actual qualifications, it disregards the mechanism by which group differences arise. Moreover, similar to demographic parity, equal opportunity is also group-based, and the university can systematically favor qualified female applicants who faced less barriers at the cost of other qualified female applicants who faced more barriers.

\paragraph{Counterfactual fairness:} If CF is enforced, then the university would ensure that each female applicant would have the same probability of being admitted had they been male, and each male applicant would have the same probability of being admitted had they been female. CF focuses on how each applicant's probability of being admitted would change in a counterfactual world where the applicant's biological sex is switched to the opposite value. In this counterfactual world, not only the applicant's biological sex changes, but the applicant's entrance exam score can also change because they might face a different level of educational barriers as a member of the opposite sex. Therefore, rather than match group-level metrics, CF directly models how the sensitive attribute affects each applicant's features. This enforces fairness at the individual level and corrects for historical barriers that each applicant experienced within the assumed causal graph.

The above example shows that, compared to other popular fairness criteria, CF has two notable distinctions. First, CF enforces fairness at an individual level rather than group level. This allows a counterfactually fair algorithm to make decisions that are fair for every individual, which mitigates the risk of the algorithm discriminating against some subgroups of individuals within the sensitive group based specifically on pathways originating from the sensitive attributes. Second, CF takes into account not only the different sensitive attributes but also how the sensitive attribute affects other covariates used for decision-making. Therefore, in addition to addressing the unfairness that directly results from the sensitive attributes, CF also corrects unfairness that is indirectly transmitted through other non-sensitive decision-making variables that are correlated with the sensitive attributes.

\section{Intuition Behind Sequential Marginal Distribution Mapping}
\label{sec:smdm_intuition}
In this section, we explain some intuitions for why sequential marginal distribution mapping works. We focus on the intuition for counterfactual state estimation using SMDM (similarly for counterfactual reward estimation). We consider the problem of estimating the $i$-th component of the counterfactual states at some $t \geq 1$ in a CMDP with state transition kernel
\begin{equation*}
S_t^{(i)} = g(Z, A_{t-1}, S_{t-1}) + U_t^{S^{(i)}}.
\end{equation*}
It is easy to check that this transition kernel satisfies Assumption \ref{ass:state_monotonicity}, so SMDM is applicable. 

Suppose the observed sensitive attribute is $z$, and we want to estimate the counterfactual under $z'$. For ease of presentation, for $\tilde{z} \in \{z, z'\}$, we write
\begin{align*}
&s_t^{\tilde{z},(i)} = S_t^{Z \leftarrow \tilde{z}, (i)}(\bar{U}(h_t)), \text{ } s_{t-1}^{\tilde{z},(i)} = S_t^{Z \leftarrow \tilde{z}, (i)}(\bar{U}(h_{t-1})), \\
&s_t^{\tilde{z}} = S_t^{Z \leftarrow \tilde{z}}(\bar{U}(h_t)), \text{ } s_{t-1}^{\tilde{z}} = S_t^{Z \leftarrow \tilde{z}}(\bar{U}(h_{t-1})).
\end{align*}
Also, let $u_t^{S^{(i)}}$ be the realized value of $U_t^{S^{(i)}}$ for our individual of interest. Note that the structural equation representations of $s_t^{z}$ and $s_t^{z'}$ are
\begin{equation*}
s_t^{z, (i)} = g(z, a_{t-1}, s_{t-1}^{z}) + u_t^{S^{(i)}}, \text{ } s_t^{z', (i)} = g(z', a_{t-1}, s_{t-1}^{z'}) + u_t^{S^{(i)}}.
\end{equation*}
Meanwhile, since $S_t^{(i)} = g(Z, A_{t-1}, S_{t-1}) + U_t^{S^{(i)}}$, the conditional distributions of $S_t^{(i)}$ are
\begin{align*}
&\mathbb{P}(S_t^{(i)} | Z=z, S_{t-1}=s_{t-1}^{z}, A_{t-1}=a_{t-1}) = \mathbb{P}(g(z, a_{t-1}, s_{t-1}^{z}) + U_t^{S^{(i)}}) \\
&\mathbb{P}(S_t^{(i)} | Z=z', S_{t-1}=s_{t-1}^{z'}, A_{t-1}=a_{t-1}) = \mathbb{P}(g(z', a_{t-1}, s_{t-1}^{z'}) + U_t^{S^{(i)}}).
\end{align*}
In this case, since $g(z, a_{t-1}, s_{t-1}^{z})$ and $g(z', a_{t-1}, s_{t-1}^{z'})$ are both constants, we can see that the quantile level of $s_t^{z, (i)}$ on $\mathbb{P}(S_t^{(i)} | Z=z, S_{t-1}=s_{t-1}^{z}, A_{t-1}=a_{t-1})$ and the quantile level of $s_t^{z', (i)}$ on $\mathbb{P}(S_t^{(i)} | Z=z', S_{t-1}=s_{t-1}^{z'}, A_{t-1}=a_{t-1})$ are both equal to the quantile level of the realized exogenous variable $u_t^{S^{(i)}}$ on $\mathbb{P}(U_t^{S^{(i)}})$, which is in line with Theorem \ref{thm:state_matching}.

More precisely, suppose $U_t^{S^{(i)}} \sim Uniform[-1, 1]$. Then 
\begin{align*}
&S_t^{(i)} | Z=z, S_{t-1}=s_{t-1}^{z}, A_{t-1}=a_{t-1} \sim Uniform[g(z, a_{t-1}, s_{t-1}^{z}) - 1, g(z, a_{t-1}, s_{t-1}^{z}) + 1] \\
&S_t^{(i)} | Z=z', S_{t-1}=s_{t-1}^{z'}, A_{t-1}=a_{t-1} \sim Uniform[g(z', a_{t-1}, s_{t-1}^{z'}) - 1, g(z', a_{t-1}, s_{t-1}^{z'}) + 1].
\end{align*}
Also, suppose $u_t^{S^{(i)}}=0.5$ so that it is on the $75$th percentile of $\mathbb{P}(U_t^{S^{(i)}})$. Then $s_t^{z, (i)} = g(z, a_{t-1}, s_{t-1}^{z}) + 0.5$, which is on the $75$th percentile of $\mathbb{P}(S_t^{(i)} | Z=z, S_{t-1}=s_{t-1}^{z}, A_{t-1}=a_{t-1})$. Similarly, $s_t^{z' (i)} = g(z', a_{t-1}, s_{t-1}^{z'}) + 0.5$, which is on the $75$th percentile of $\mathbb{P}(S_t^{(i)} | Z=z', S_{t-1}=s_{t-1}^{z'}, A_{t-1}=a_{t-1})$. This is in line with our previous claim that $s_t^{z, (i)}$ and $s_t^{z', (i)}$ are on the same quantile on their respective conditional distributions.

Finally, we note that the noise does not need to be additive as in the example above. The above analysis holds as long as the CMDP satisfies Assumption \ref{ass:state_monotonicity} (or Assumption \ref{ass:reward_monotonicity} if we are interested in the counterfactual rewards). It is the strict monotonicity of $S_t^{(i)}$ in $U_t^{S^{(i)}}$ that guarantees $s_t^{z, (i)}$ and $s_t^{z', (i)}$ are on the same quantile on their respective conditional distributions.

\section{Additional Discussion of the Single Scalar Noise Formulation}
\label{sec:additional_single_scalar_noise_discussion}

In our formulation of the underlying CMDP, each component of the state vector and the reward is affected by an exogenous noise variable. The noise variable for each state component and the reward need not to be the same, but each noise variable is required to be scalar-valued. When the state component or the reward is affected by multiple noise factors, the single scalar noise formulation may remain plausible as long as the aggregate effect of multiple noise sources can be captured by a single scalar. For example, a student’s exam score may depend on noise factors such as the room temperature, illness, and stressors other than the exam, whose combined effects may be summarized by a single scalar-valued noise variable representing the student’s overall level of physical and mental comfort during the exam. On the other hand, the single scalar noise formulation is less plausible when domain knowledge suggests that the noise factors cannot be reasonably aggregated.

Section \ref{sec: additional_single_scalar_noise_experiment} presents additional experiments that assess the sensitivity of the performance of CFSMDM to deviations from the single scalar noise formulation.

\section{Additional Discussion of Assumptions \ref{ass:state_monotonicity} and \ref{ass:reward_monotonicity}}
\label{sec:additional_monotonicity_discussion}

Assumption \ref{ass:state_monotonicity} assumes that each component of the state vector is strictly increasing in its corresponding scalar-valued noise variable. Assumption \ref{ass:reward_monotonicity} assumes that the reward is strictly increasing in its corresponding scalar-valued noise variable. As shown in Section \ref{sec:cfsmdm_cfsdp}, compared to the additivity assumption imposed by the CFSDP algorithm in \cite{wang2025counterfactuallyfairreinforcementlearning}, our Assumptions \ref{ass:state_monotonicity} and \ref{ass:reward_monotonicity} are strictly weaker, meaning that CFSMDM is theoretically applicable to a wider class of CMDPs than CFSDP. In particular, Assumptions \ref{ass:state_monotonicity} and \ref{ass:reward_monotonicity} accommodate CMDPs with location-scale transition kernels of the form $S_{t+1} = m(Z, S_t, A_t) + \sigma(Z, S_t, A_t)U_{t+1}$ where $\sigma(Z, S_t, A_t)>0$, as well as CMDPs with post-nonlinear transition kernels of the form $S_{t+1} = h(m(Z, S_t, A_t)+U_{t+1})$ where $h$ is strictly increasing. Neither class of CMDPs is theoretically accommodated by CFSDP.

The restrictiveness of Assumptions \ref{ass:state_monotonicity} and \ref{ass:reward_monotonicity} might be assessed based on domain knowledge and the interpretation of the noise variable. For example, Assumptions \ref{ass:state_monotonicity} and \ref{ass:reward_monotonicity} may be plausible when the noise variable represents an underlying severity level, such as when the observed state is a patient’s self-reported pain level and the noise variable is the actual pain level experienced by the patient. It may also be reasonable when the noise variable captures an external condition, or the collective effect of multiple external conditions, that affects all sensitive groups in the same direction, such as when the state is an athlete's performance and the noise variable is the level of comfort of the weather, so that better weather tends to improve the athlete's performance regardless of which sensitive group the athlete belongs to. On the other hand, monotonicity is less plausible when the noise variable represents factors that have opposite effects across sensitive groups.

The validity of the monotonicity assumption might be assessed empirically, provided that (i) the practitioner has a plausible interpretation of the noise variable, and (ii) the noise variable can be consistently estimated from the data, which may require additional modeling assumptions. Under these conditions, existing statistical tests might be adapted to examine whether the noise variable consistently shifts the state or reward in the same direction after conditioning on the observed history and sensitive attribute. For example, see \cite{hsu2019monotonicity}.

Section \ref{sec: additional_monotonicity_experiment} presents additional experiments that assess the sensitivity of the performance of CFSMDM to deviations from Assumptions \ref{ass:state_monotonicity} and \ref{ass:reward_monotonicity}.

\section{Pseudocodes of SMDM and CFSMDM}
\label{sec:pseudocode}

Algorithm \ref{alg:smdm} shows the pseudocode for sequential marginal distribution mapping (SMDM). Algorithm \ref{alg:cfsmdm} shows the pseudocode for counterfactually fair policy learning via SMDM (CFSMDM).

\begin{algorithm}
\caption{Sequential Marginal Distribution Mapping (SMDM)}\label{alg:smdm}

\textbf{Input:} Data for preprocessor training $\mathcal{D}' = \{s'_{j,0}, s'_{j,t}, z'_j, a'_{j,t-1}, r'_{j,t-1}: j = 1, \dots, N', t = 1, \dots, T'\}$. Original data to be preprocessed $\mathcal{D} = \{s_{j,0}, s_{j,t}, z_j, a_{j,t-1}, r_{j,t-1}: j = 1, \dots, N, t = 1, \dots, T\}$, where we denote $s_{j,t} = (s_{j,t}^{(1)}, \dots, s_{j,t}^{(M)})$.

\begin{algorithmic}[1]

\For{$i = 1, \dots, M$} \Comment{\textit{Step 1: Fit models of the environment dynamics.}}

    \State Fit models $\hat{Q}_0^{(i)}(\tau, z)$ and $\hat{Q}^{(i)}(\tau, z, s, a)$ using $\mathcal{D'}$ that estimates the conditional quantile functions $Q_0^{(i)}(\tau, z) = \inf\{x \in \mathcal{S}^{(i)}: \mathbb{P}(S_0^{(i)} \leq x|Z=z) \geq \tau\}$ and $Q^{(i)}(\tau, z, s, a) = \inf\{x \in \mathcal{S}^{(i)}: \mathbb{P}(S_t^{(i)} \leq x|Z=z, S_{t-1}=s, A_{t-1}=a) \geq \tau, t \geq 1\}$, respectively.

\EndFor

\State Fit model $\hat{W}(\tau, z, s, a)$ using $\mathcal{D'}$ that estimates the conditional quantile function $W(\tau, z, s, a) = \inf\{x \in \mathcal{R}: \mathbb{P}(R_t \leq x|Z=z, S_{t}=s, A_{t}=a) \geq \tau, t \geq 0\}$. 

\State $\forall z \in \mathcal{Z}$, estimate $\hat{\mathbb{P}}(Z=z)$ by the empirical mean using $\mathcal{D}'$.

\For{$j = 1, \dots, N$} 

    \For{$i = 1, \dots, M$} \Comment{\textit{Step 2: Preprocessing at $t=0$.}}

        \State Estimate $\tau_{j,0}^{(i)} = \mathbb{P}(S_0^{(i)} \leq s_{j,0}^{(i)}|Z=z_j)$ by $\hat{\tau}_{j,0}^{(i)} = \sup\{\tau \in (0,1): \hat{Q}_0^{(i)}(\tau, z_j) \leq s_{j,0}^{(i)}\}$.
        
        \State $\hat{s}_{j,0}^{Z \leftarrow z_j, (i)} = s_{j,0}^{(i)}$.
    
        \State $\forall z' \in \mathcal{Z} \backslash \{z_j\}$, set $\hat{s}_{j,0}^{Z \leftarrow z', (i)} = \hat{Q}_0^{(i)}(\hat{\tau}_{j,0}^{(i)}, z')$.

    \EndFor

    \State $\forall z \in \mathcal{Z}$, $\hat{s}_{j,0}^{Z \leftarrow z} = [\hat{s}_{j,0}^{Z \leftarrow z, (1)}, \dots, \hat{s}_{j,0}^{Z \leftarrow z, (M)}]$

    \State $\hat{\tilde{s}}_{j,0} = [\hat{s}_{j,0}^{Z \leftarrow z^{(1)}}, \dots, \hat{s}_{j,0}^{Z \leftarrow z^{(K)}}]^T$
    
    \For{$t = 1, \dots, T$} \Comment{\textit{Step 3: Preprocessing for $t \geq 1$.}}
    
        \State Estimate $\theta_{j,t-1} = \mathbb{P}(R_{t-1} \leq r_{j,t-1}|Z=z_j, S_{t-1}=s_{j,t-1}, A_{t-1}=a_{j,t-1})$ by $\hat{\theta}_{j,t-1} = \sup\{\theta \in (0,1): \hat{W}(\theta, z_j, s_{j,t-1}, a_{j,t-1}) \leq r_{j,t-1}\}$.
        
        \State $\hat{r}_{j,t-1}^{Z \leftarrow z_j} = r_{j,t-1}$.
        
        \State $\forall z' \in \mathcal{Z} \backslash \{z_j\}$, set $\hat{r}_{j,t-1}^{Z \leftarrow z'} = \hat{W}(\hat{\theta}_{j,t-1}, z', \hat{s}_{j,t-1}^{Z \leftarrow z'}, a_{j,t-1})$.

        \State $\hat{\tilde{r}}_{j,t-1} = \sum_{k=1}^{K} \hat{r}_{j,t-1}^{Z \leftarrow z^{(k)}} \hat{\mathbb{P}}(Z=z^{(k)})$

        \For{$i = 1, \dots, M$}
        
            \State Estimate $\tau_{j,t}^{(i)} = \mathbb{P}(S_t^{(i)} \leq s_{j,t}^{(i)}|Z=z_j, S_{t-1}=s_{j,t-1}, A_{t-1}=a_{j,t-1})$ by $\hat{\tau}_{j,t}^{(i)} = \sup\{\tau \in (0,1): \hat{Q}^{(i)}(\tau, z_j, s_{j, t-1}, a_{j,t-1}) \leq s_{j,t}^{(i)}\}$.
            
            \State $\hat{s}_{j,t}^{Z \leftarrow z_j, (i)} = s_{j,t}^{(i)}$.
            
            \State $\forall z' \in \mathcal{Z} \backslash \{z_j\}$, set $\hat{s}_{j,t}^{Z \leftarrow z',(i)} = \hat{Q}(\hat{\tau}_{j,t}^{(i)}, z', \hat{s
            }_{j,t-1}^{Z \leftarrow z'}, a_{j,t-1})$.

        \EndFor

        \State $\forall z \in \mathcal{Z}$, $\hat{s}_{j,t}^{Z \leftarrow z} = [\hat{s}_{j,t}^{Z \leftarrow z, (1)}, \dots, \hat{s}_{j,t}^{Z \leftarrow z, (M)}]$

        \State $\hat{\tilde{s}}_{j,t} = [\hat{s}_{j,t}^{Z \leftarrow z^{(1)}}, \dots, \hat{s}_{j,t}^{Z \leftarrow z^{(K)}}]^T$
        
    \EndFor
    
\EndFor

\end{algorithmic}

\textbf{Output:} Preprocessed experience tuples $\{\hat{\tilde{s}}_{j,0}, \hat{\tilde{s}}_{j,t}, a_{j,t-1}, \hat{\tilde{r}}_{j,t-1}: j = 1, \dots, N, t = 1, \dots, T\}$.

\end{algorithm}

\begin{algorithm}
\caption{Counterfactually Fair Policy Learning via SMDM (CFSMDM)}\label{alg:cfsmdm}

\textbf{Input:} Data for preprocessor training $\mathcal{D}' = \{s'_{j,0}, s'_{j,t}, z'_j, a'_{j,t-1}, r'_{j,t-1}: j = 1, \dots, N', t = 1, \dots, T'\}$. Original data to be preprocessed $\mathcal{D} = \{s_{j,0}, s_{j,t}, z_j, a_{j,t-1}, r_{j,t-1}: j = 1, \dots, N, t = 1, \dots, T\}$, where we denote $s_{j,t} = (s_{j,t}^{(1)}, \dots, s_{j,t}^{(M)})$.

\begin{algorithmic}[1]

\State Preprocess $\mathcal{D}$ using SMDM to obtain preprocessed experience tuples $\tilde{\mathcal{D}} = \{\hat{\tilde{s}}_{j,0}, \hat{\tilde{s}}_{j,t}, a_{j,t-1}, \hat{\tilde{r}}_{j,t-1}: j = 1, \dots, N, t = 1, \dots, T\}$.

\State Train a policy $\hat{\pi}$ using $\tilde{\mathcal{D}}$.

\end{algorithmic}

\textbf{Output:} Learned policy $\hat{\pi}$, which should be close to being counterfactually fair.

\end{algorithm}

\section{Connections Between CFSMDM and CFSDP}
\label{sec:cfsmdm_cfsdp}
\citet{wang2025counterfactuallyfairreinforcementlearning} proposed a data preprocessing method for performing counterfactually fair policy learning, which we call CFSDP. CFSDP also enforces CF by performing policy learning on estimates of the augmented states $\tilde{S}_t$ and the augmented rewards $\tilde{R}_t$, but it differs from CFSMDM in how the augmented states and rewards are estimated. Tables \ref{tab:cfsmdm_cfsdp_assumption_comparison} and \ref{tab:cfsdp_cfsmdm_components_comparison} provide a comprehensive comparison of the assumptions and major components of CFSMDM and CFSDP. In particular, as a result of the difference in estimation methods, while CFSMDM requires Assumptions \ref{ass:state_monotonicity} and \ref{ass:reward_monotonicity}, CFSDP makes an additivity assumption: for $t \geq 0$, the exogenous variables $U_t^S$ and $U_t^R$ are additive to $S_t$ and $R_t$, respectively.

\begin{table*}[t]
    \centering
    \caption{Assumptions for CFSMDM and their relationship to those for CFSDP. The first four assumptions are for the general validity of CFSMDM. The remaining assumptions are for theoretical analyses only.}
    \label{tab:cfsmdm_cfsdp_assumption_comparison}
    \small
    \renewcommand{\arraystretch}{1.2}
    \begin{tabularx}{\textwidth}{
        @{}
        >{\raggedright\arraybackslash}p{0.45\textwidth}
        >{\raggedright\arraybackslash}X
        @{}
    }
        \toprule
        \textbf{Assumption in CFSMDM} &
        \textbf{Inherited from CFSDP?} \\
        \midrule
        Stationary CMDP represented by an SCM &
        Yes \\

        Independent and identically distributed trajectories &
        Yes \\

        Assumption \ref{ass:state_monotonicity}: Strict monotonicity of the state in the noise variable &
        No, but is strictly weaker than CFSDP's additive noise assumption \\

        Assumption \ref{ass:reward_monotonicity}: Strict monotonicity of the reward in the noise variable &
        No, but is strictly weaker than CFSDP's additive noise assumption \\

        Assumption \ref{ass:space}: Boundedness of state and reward spaces &
        Bounded state space is not inherited; bounded reward space is inherited \\

        Assumption \ref{ass:fqi}.1: Linear and bounded Q function class &
        Yes \\

        Assumption \ref{ass:fqi}.2: Bellman completeness &
        Yes \\

        Assumption \ref{ass:fqi}.3: Feature coverage &
        Yes \\

        Assumption \ref{ass:lipschitz_conditional_quantile}: Lipschitz continuous conditional quantile functions and feature vector &
        Lipschitz continuity of conditional quantile functions is not inherited; Lipschitz continuous feature vector is inherited \\

        Assumption \ref{ass:margin}: Margin condition &
        Yes \\

        Independent samples for preprocessing and policy learning &
        No \\

        Independent and identically distributed oracle augmented transition tuples &
        Yes \\
        \bottomrule
    \end{tabularx}
\end{table*}

\begin{table*}[t]
    \centering
    \caption{Comparison of the main components of CFSDP and CFSMDM.}
    \label{tab:cfsdp_cfsmdm_components_comparison}
    \scriptsize
    \setlength{\tabcolsep}{4pt}
    \renewcommand{\arraystretch}{1.2}
    \begin{tabularx}{\textwidth}{
        @{}
        >{\raggedright\arraybackslash}p{0.16\textwidth}
        >{\raggedright\arraybackslash}X
        >{\raggedright\arraybackslash}X
        >{\raggedright\arraybackslash}X
        @{}
    }
        \toprule
        & \textbf{CFSDP}
        & \textbf{CFSMDM}
        & \textbf{Relation to CFSDP} \\
        \midrule

        Motivating example
        & Fairness issues in the PowerED study
        & Same fairness issues in the PowerED study
        & Inherited \\

        \addlinespace

        Causal/RL framework
        & CMDP represented by an SCM
        & CMDP represented by an SCM
        & Inherited/reformulated \\

        \addlinespace

        Counterfactual fairness definition
        & Action distribution remain unchanged after switching the sensitive attribute while holding the historical exogenous variables and observed past actions fixed
        & Same definition
        & Inherited \\

        \addlinespace

        Counterfactual abduction
        & Historical state and reward noises are deterministic functions of the observed history
        & Same deterministic abduction requirement; automatically satisfied if the strict monotonicity holds
        & Inherited \\

        \addlinespace

        Identification assumption
        & Additive noise assumption
        & Strict monotonicity of state components and reward their scalar exogenous noise
        & \textbf{Relaxed:} strict monotonicity is strictly weaker than additivity \\

        \addlinespace

        Counterfactual estimator
        & Sequential conditional mean-residual mapping: recover the noise variable (residual) and add it to conditional counterfactual mean
        & Sequential distribution mapping: match quantiles on the conditional distributions under observed and counterfactual sensitive attribute values
        & Changed, but sequential conditional mean-residual mapping can be seen as a shortcut to sequential distribution mapping when the additive noise assumption holds \\

        \addlinespace

        Restrictions on counterfactual estimation errors
        & Assumes counterfactual state and reward estimation errors are bounded above by a constant $\epsilon$ in the proof of performance guarantees
        & No restrictions
        & \textbf{Relaxed:} it can sometimes be restrictive to assume a constant bound of counterfactual estimation errors, because these errors might accumulate over time \\

        \addlinespace

        Assumptions for performance guarantee
        & Bounded linear Q-function class, Bellman completeness, feature coverage, Lipschitz features, bounded rewards, a margin condition for unfairness control, and independent and identically distributed oracle augmented transition tuples
        & Same assumptions, together with a bounded state space, Lipschitz conditional quantile functions, and independent samples for preprocessor training and policy learning
        & Mostly inherited; quantile-related regularity and independence requirement is added \\

        \addlinespace

        Performance guarantee
        & Suboptimality gap and counterfactual unfairness bounds depending on counterfactual estimation error, FQI regression error, and FQI initialization bias
        & Infinite-horizon suboptimality gap and per-step counterfactual unfairness bounds depending on counterfactual estimation errors, FQI regression error, and FQI initialization bias
        & Analogous guarantees extended to SMDM \\

        \addlinespace

        Applicable data-generating mechanisms
        & Additive-noise CMDPs
        & Additive-noise CMDPs, or nonadditive CMDPs satisfying strict monotonicity
        & Strictly broader applicability \\

        \bottomrule
    \end{tabularx}
\end{table*}

The following theorem states that the additivity assumption of CFSDP is strictly stronger than Assumptions \ref{ass:state_monotonicity} and \ref{ass:reward_monotonicity}, which means CFSMDM can be considered strictly more general than CFSDP.

\begin{theorem}
\label{thm:strict_weak}
The additivity assumption implies Assumptions \ref{ass:state_monotonicity} and \ref{ass:reward_monotonicity}, but Assumptions \ref{ass:state_monotonicity} and \ref{ass:reward_monotonicity} do not imply the additivity assumption in general.
\end{theorem}

Now we argue that the CFSDP method can be seen as a simplified version of the CFSMDM method, where the simplification is made possible by the additive noise assumption. We focus on the counterfactual states in the following argument, and the argument for the counterfactual rewards should be similar. Under the additivity assumption and an assumption that $U_t^{S^{(i)}}$ is scalar-valued with mean zero, for an individual with sensitive attribute $z^{(k)}$ and observed history $h_t$, we have
\begin{align*}
&S_t^{Z \leftarrow z^{(k)}, (i)}(\bar{U}_t(h_t)) \\
=&f^{S^{(i)}}(z^{(k)}, S_t^{Z \leftarrow z^{(k)}}(\bar{U}_t(h_t)), a_{t-1}) + U_t^{S^{(i)}} \\
=&\mathbb{E}[S_t^{Z \leftarrow z^{(k)}, (i)}(\bar{U}_t(h_t))|S_{t-1}=S_{t-1}^{Z \leftarrow z^{(k)}}(\bar{U}_t(h_t))] + U_t^{S^{(i)}} \\
=&\mathbb{E}[S_t^{(i)}|Z=z^{(k)}, S_{t-1}=S_{t-1}^{Z \leftarrow z^{(k)}}(\bar{U}_t(h_t)), A_{t-1}=a_{t-1}] + U_t^{S^{(i)}},
\end{align*}
where the second equality follows from the assumptions that $U_t^{S^{(i)}}$ is mean zero, and the third equality is due to exchangeability implied by Figure \ref{fig:cmdp_dag}. Assuming the conditional expectation and past counterfactual states are known, this shows that the only stochasticity in the conditional distribution $\mathbb{P}(S_t^{(i)}|Z=z^{(k)}, S_{t-1}=S_{t-1}^{Z \leftarrow z^{(k)}}(\bar{U}_{t-1}(h_{t-1})), A_{t-1}=a_{t-1})$ comes from $U_t^{S^{(i)}}$. Moreover, since $S_t^{Z \leftarrow z^{(k)}, (i)}(\bar{U}_t(h_t))$ is strictly increasing in $U_t^{S^{(i)}}$, the quantile level of $U_t^{S^{(i)}}$ in $\mathbb{P}(U_t^{S^{(i)}})$ is the same as the quantile level of $S_t^{Z \leftarrow z^{(k)} (i)}(\bar{U}_t(h_t))$ in $\mathbb{P}(S_t^{(i)}|Z=z^{(k)}, S_{t-1}=S_{t-1}^{Z \leftarrow z^{(k)}}(\bar{U}_{t-1}(h_{t-1})), A_{t-1}=a_{t-1})$. Note that CFSDP performs the following operation to recover the counterfactual state under some sensitive attribute $z^{(j)}$ with $j \neq k$:
\begin{align*}
&S_t^{Z \leftarrow z^{(j)}, (i)}(\bar{U}_t(h_t)) \\
=&S_t^{Z \leftarrow z^{(k)}, (i)}(\bar{U}_t(h_t)) \\
&-\mathbb{E}[S_t^{(i)}|Z=z^{(k)}, S_{t-1}=S_{t-1}^{Z \leftarrow z^{(k)}}(\bar{U}_{t-1}(h_{t-1})), A_{t-1}=a_{t-1}] \\
&+\mathbb{E}[S_t^{(i)}|Z=z^{(j)}, S_{t-1}=S_{t-1}^{Z \leftarrow z^{(j)}}(\bar{U}_{t-1}(h_{t-1})), A_{t-1}=a_{t-1}] \\
=& \mathbb{E}[S_t^{(i)}|Z=z^{(j)}, S_{t-1}=S_{t-1}^{Z \leftarrow z^{(j)}}(\bar{U}_{t-1}(h_{t-1})), A_{t-1}=a_{t-1}] + U_t^{S^{(i)}}.
\end{align*}
By this operation, CFSDP ensures $S_t^{Z \leftarrow z^{(k)} (i)}(\bar{U}_t(h_t))$ and $S_t^{Z \leftarrow z^{(j)} (i)}(\bar{U}_t(h_t))$ share the same value of $U_t^{S^{(i)}}$ and thereby the same quantile level on their respective conditional distributions. Therefore, the counterfactual state estimation method in CFSDP can be seen as a shortcut for sequential marginal distribution mapping that is specifically tailored to structural equation models satisfying the additive noise assumption. 

\section{Connections Between CFSMDM and FLAP}
\label{sec:cfsdmd_flap}
Consider CFSMDM with only the initial time step, where the agent takes an action based on the state and possibly the sensitive attribute, but no reward is produced. If we view the action taken as an ML prediction or classification, then the problem in this setting can be seen as a single-stage ML problem, and CFSMDM enforces the predictions made by the predictor are counterfactually fair under the definition given in \cite{kusner2018counterfactualfairness}. In this case, the initial-state counterfactual estimation procedure in CFSMDM coincides with the marginal distribution mapping procedure underlying the FLAP algorithm introduced by \cite{chen2024flap}. While CFSMDM constructs the augmented state as the vector of the counterfactual states under all sensitive attribute levels, FLAP uses the probability-weighted average of the counterfactual states. In this light, CFSMDM can be seen as a generalization of the marginal distribution mapping approach underlying the FLAP algorithm to the sequential decision-making setting.

\section{The FQI and FQE Algorithms}
\label{sec:fqi_fqe}
This section presents the pseudocodes of the FQI and FQE algorithms used in the numerical experiments and the real data analysis.

\begin{algorithm}[H]
    \caption{Fitted Q iteration \citep{riedmiller2005fqi}}

    \textbf{Input:} Dataset $\mathcal{D} = \{(s_i, a_i, r_i, s'_i) : i = 1, \dots, n\}$, function class $\mathcal{F}$.
    
    \begin{algorithmic}[1]
        
        \State Initialize $\hat{f}_0 \in \mathcal{F}$ randomly.
        \For{$b = 1, 2, \dots, B$}
            \State Compute target $y_i = r_i + \gamma \max_a \hat{f}_{b-1}(s'_i, a)$
            \State Build training dataset $\widetilde{D}_b = \{(s_i, a_i), y_i\}_{i=1}^n$
            \State Solve a supervised learning problem:
            \[ \hat{f}_b = \arg \min_{f \in \mathcal{F}} \frac{1}{n} \sum_{i=1}^{n} (f(s_i, a_i) - y_i)^2 \]
        \EndFor
        
    \end{algorithmic}

    \textbf{Output:} Estimated optimal Q function $\hat{f}_B$.
    
\end{algorithm}

\vspace{0.2cm}

\begin{algorithm}[H]
    \caption{Fitted Q Evaluation \citep{le2019Batch}}

    \textbf{Input:} Dataset $\mathcal{D} = \{(s_i, a_i, r_i, s'_i) : i = 1, \dots, n\}$. Function class $\mathcal{F}$. Deterministic policy $\pi$ to be evaluated.
    
    \begin{algorithmic}[1]
        
        \State Initialize $\hat{f}_0 \in \mathcal{F}$ randomly.
        \For{$b = 1, 2, \dots, B$}
            \State Compute target $y_i = r_i + \gamma \hat{f}_{b-1}(s'_i, \pi(s'_i)) \quad \forall i$
            \State Build training dataset $\widetilde{D}_b = \{(s_i, a_i), y_i\}_{i=1}^n$
            \State Solve a supervised learning problem:
            \[ \hat{f}_b = \arg \min_{f \in \mathcal{F}} \frac{1}{n} \sum_{i=1}^{n} (f(s_i, a_i) - y_i)^2 \]
        \EndFor
        
    \end{algorithmic}

    \textbf{Output:} $\widehat{V}^\pi(s) = \hat{f}_B(s, \pi(s)) \quad \forall s.$
    
\end{algorithm}

\section{Detailed Discussion of the Suboptimality Gap and Unfairness Bounds}
\label{sec:bound_discussion}

\paragraph{Suboptimality gap bound.} 
The first term in the bound controls the state and reward estimation error at the deployment stage up to time $\omega-1$; it tends to grow as $\omega$ increases, since a large $\omega$ increases $1 - \gamma^{\omega}$ and also requires taking the maximum over $M_{t,n}$ for more time steps. The second term in the bound covers the state and reward estimation errors for all time steps starting from time $\omega$; it decreases exponentially to $0$ as $\omega$ grows. The third term in the bound is a result of the effects on FQI of the state and reward estimation errors in the training trajectories (which have exactly $T$ time steps). The fourth term in the bound accounts for the one-step FQI regression error. The fifth term is related to the finite-iteration error in FQI, and it approaches $0$ geometrically as the number of FQI iterations increases. 

\paragraph{Unfairness bound.} The first term in the expectation term in the bound accounts for the FQI estimation error resulting from FQI regression error, initialization bias, and state and reward estimation errors in the training trajectories. The second term in the expectation term in the bound is related to the counterfactual state and reward estimation error at the deployment stage up to time $t$. Smaller FQI regression error and deployment-time state and reward estimation error will both result in lower unfairness. The rest of the bound adjusts for the possibility of large FQI estimation errors and large state and reward estimation errors.

\begin{remark}
Both the suboptimality gap and counterfactual unfairness bounds involve terms related to the probability $\mathbb{P}\left(\delta_{T,n} < \frac{\lambda_0}{4dL}\right)$, which depends on the uniform estimation errors of the state and reward conditional quantiles. For estimators such as linear quantile regression and quantile regression neural networks, it can be shown by concentration inequalities that, under certain conditions, this probability converges to 1 exponentially fast (at rate $1 - \exp{(-\Omega(n))}$) with respect to the sample size $n$.
\end{remark}

\section{Proofs of Theorems}
\label{sec:proofs}

\subsection{Proof of Theorems \ref{thm:state_matching} and \ref{thm:reward_matching}}
\label{subsec:proof_matching}

\paragraph{Theorem \ref{thm:state_matching}.} 
\textit{Suppose Assumption~\ref{ass:state_monotonicity} hold, and consider an individual with observed history $h_t=(z, \bar{s}_{t}, \bar{a}_{t-1}, \bar{r}_{t-1})$. Then, for $t=0$ and each $i \in \{1, \dots, M\}$, the quantile level of $S_0^{Z \leftarrow z', (i)}(\bar{U}_0(h_0))$ with respect to the law of $f_0^{S^{(i)}}(z', U^{S^{(i)}}_0)$ does not vary by $z' \in \mathcal{Z}$. }

\textit{Similarly, for later time steps $t\geq 1$ and each $i \in \{1, \dots, M\}$, the quantile level of $S_t^{Z \leftarrow z', (i)}(\bar{U}_t(h_t))$ with respect to the law of $f^{S^{(i)}}(z', a_{t-1}, S_{t-1}^{Z \leftarrow z'}(\bar{U}_{t-1}(h_{t-1})), U_t^{S^{(i)}})$ does not vary by $z' \in \mathcal{Z}$.}

\paragraph{Theorem \ref{thm:reward_matching}.} 
\textit{Suppose Assumption~\ref{ass:reward_monotonicity} hold, and consider an individual with observed history $h_t=(z, \bar{s}_{t}, \bar{a}_{t-1}, \bar{r}_{t-1})$. Then, for each time step $t \geq 0$, the quantile level of $R_t^{Z \leftarrow z'}(\bar{U}_{t+1}(h_{t+1}))$ with respect to the law of $f^{R}(z', a_{t}, S_{t}^{Z \leftarrow z'}(\bar{U}_{t}(h_{t})), U_t^{R})$ does not vary by $z' \in \mathcal{Z}$.}

\begin{proof}
We present a proof of Theorem \ref{thm:state_matching}. The proof of Theorem \ref{thm:reward_matching} is analogous. 

For Theorem \ref{thm:state_matching}, we first treat the $t=0$ case. Here, it suffices to show 
\begin{equation*}
\mathbb{P}(S_0^{(i)} \leq S_0^{Z \leftarrow z^{(k)}, (i)}(\bar{U}_0(h_0)) | Z=z^{(k)}) = \mathbb{P}(S_0^{(i)} \leq S_0^{Z \leftarrow z^{(j)}, (i)}(\bar{U}_0(h_0)) | Z=z^{(j)})
\end{equation*}
for every $k,j \in \{1, \dots, K\}$. For ease of presentation, for each $z \in \mathcal{Z}$ and $i \in \{1, \dots, M\}$, we use $s_0^{z,(i)}$ to denote $S_0^{Z \leftarrow z, (i)}(\bar{U}_0(h_0))$, and we use $u_0^{S^{(i)}}$ to denote the realized value of the exogenous variable $U_0^{S^{(i)}}$. Note that for each $i \in \{1, \dots, M\}$, we have 
\begin{align*}
&\mathbb{P}(S_0^{(i)} \leq s_0^{z^{(k)}, (i)} | Z=z^{(k)}) \\
=& \mathbb{P}(f_0^{S^{(i)}}(Z, U_0^{S^{(i)}}) \leq f_0^{S^{(i)}}(z^{(k)}, u_0^{S^{(i)}}) | Z=z^{(k)}) \\
=& \mathbb{P}(f_0^{S^{(i)}}(z^{(k)}, U_0^{S^{(i)}}) \leq f_0^{S^{(i)}}(z^{(k)}, u_0^{S^{(i)}})) \\
=& \mathbb{P}(U_0^{S^{(i)}} \leq u_0^{S^{(i)}}) \\
=& \mathbb{P}(f_0^{S^{(i)}}(z^{(j)}, U_0^{S^{(i)}}) \leq f_0^{S^{(i)}}(z^{(j)}, u_0^{S^{(i)}})) \\
=& \mathbb{P}(f_0^{S^{(i)}}(Z, U_0^{S^{(i)}}) \leq f_0^{S^{(i)}}(z^{(j)}, u_0^{S^{(i)}}) | Z=z^{(j)}) \\
=& \mathbb{P}(S_0^{(i)} \leq s_0^{z^{(j)}, (i)} | Z=z^{(j)}).
\end{align*}
In the derivation above, the first and sixth equalities are due to the structural equation representations of $S_0^{(i)}$, $s_0^{z^{(k)}, (i)}$, and $s_0^{z^{(j)}, (i)}$, and the third and fourth equalities are due to Assumption~\ref{ass:state_monotonicity}. This proves the $t=0$ case of Theorem \ref{thm:state_matching}.

Now we treat the $t \geq 1$ case. For this case, it suffices to show
\begin{align*}
&\mathbb{P}(S_t^{(i)} \leq S_t^{Z \leftarrow z^{(k)}, (i)}(\bar{U}_t(h_t)) | Z=z^{(k)}, S_{t-1}=S_{t-1}^{Z \leftarrow z^{(k)}}(\bar{U}_{t-1}(h_{t-1})), A_{t-1}=a_{t-1}) \notag \\
=&\mathbb{P}(S_t^{(i)} \leq S_t^{Z \leftarrow z^{(j)}, (i)}(\bar{U}_t(h_t)) | Z=z^{(j)}, S_{t-1}=S_{t-1}^{Z \leftarrow z^{(j)}}(\bar{U}_{t-1}(h_{t-1})), A_{t-1}=a_{t-1}).
\end{align*}
for $k,j \in \{1, \dots, K\}$. For ease of presentation, for each $t \geq 1$, $z \in \mathcal{Z}$, and $i \in \{1, \dots, M\}$, we use $s_t^{z, (i)}$ to denote $S_t^{Z \leftarrow z, (i)}(\bar{U}_t(h_t))$, and we use $u_t^{S^{(i)}}$ to denote the realized value of the exogenous variable $U_t^{S^{(i)}}$. Note that for each $i \in \{1, \dots, M\}$, we have 
\begin{align*}
&\mathbb{P}(S_t^{(i)} \leq s_t^{z^{(k)}, (i)} | Z=z^{(k)}, S_{t-1}=s_{t-1}^{z^{(k)}}, A_{t-1}=a_{t-1}) \\
=& \mathbb{P}(f^{S^{(i)}}(z^{(k)}, a_{t-1}, s_{t-1}^{z^{(k)}}, U_t^{S^{(i)}}) \leq f^{S^{(i)}}(z^{(k)}, a_{t-1}, s_{t-1}^{z^{(k)}}, u_t^{S^{(i)}})) \\
=& \mathbb{P}(U_t^{S^{(i)}} \leq u_t^{S^{(i)}}) \\
=& \mathbb{P}(f^{S^{(i)}}(z^{(j)}, a_{t-1}, s_{t-1}^{z^{(j)}}, U_t^{S^{(i)}}) \leq f^{S^{(i)}}(z^{(j)}, a_{t-1}, s_{t-1}^{z^{(j)}}, u_t^{S^{(i)}})) \\
=& \mathbb{P}(S_t^{(i)} \leq s_t^{z^{(j)}, (i)} | Z=z^{(j)}, S_{t-1}=s_{t-1}^{z^{(j)}}, A_{t-1}=a_{t-1}).
\end{align*}
In the derivation above, the first and fourth equalities are due to the structural equation representations of $S_t^{(i)}$, $s_t^{z^{(k)}, (i)}$, and $s_t^{z^{(j)}, (i)}$, and the second and third equalities are due to Assumption~\ref{ass:state_monotonicity}. This proves the $t \geq 1$ case of Theorem \ref{thm:state_matching}.

This completes the proof.
\end{proof}

\subsection{Proof of Theorem \ref{thm:strict_weak}}
\label{subsec:proof_weak}

\paragraph{Theorem \ref{thm:strict_weak}.} \textit{The additivity assumption implies Assumptions \ref{ass:state_monotonicity} and \ref{ass:reward_monotonicity}, but Assumptions \ref{ass:state_monotonicity} and \ref{ass:reward_monotonicity} do not imply the additivity assumption in general.}
\begin{proof}
We first show that the additivity assumption implies Assumptions \ref{ass:state_monotonicity} and \ref{ass:reward_monotonicity}. Suppose the additivity assumption is satisfied. In this case, we can write 
\begin{align*}
S_0^{(i)} =& f_0^{S^{(i)}}(Z, U_0^{S^{(i)}}) = g_0^{S^{(i)}}(Z) + U_0^{S^{(i)}} \text{ (for $i \in \{1, \dots, M\}$)}, \\
S_t^{(i)} =& f^{S^{(i)}}(Z, A_{t-1}, S_{t-1}, U_t^{S^{(i)}}) = g^{S^{(i)}}(Z, A_{t-1}, S_{t-1}) + U_t^{S^{(i)}} \\
&\text{ (for $i \in \{1, \dots, M\}$ and $t \geq 1$)}, \\
R_t =& f^R(Z, S_t, A_t, U_t^R) = g^R(Z, S_t, A_t) + U_t^R \text{ (for $t \geq 0$)}
\end{align*}
for some functions $\{f_0^{S^{(i)}}\}_{i=0}^M$, $\{f^{S^{(i)}}\}_{i=0}^M$, and $f^R$. Then Assumptions \ref{ass:state_monotonicity} and \ref{ass:reward_monotonicity} are obviously satisfied.

Now we show Assumptions \ref{ass:state_monotonicity} and \ref{ass:reward_monotonicity} do not imply the additivity assumption in general. Consider the CMDP
\begin{align*}
S_0^{(1)} =& (-0.45 + 1.5\delta Z + U_{0}^{S,(1)} + 0.5 U_{0}^{S,(3)})^{\frac{1}{3}}, \\
S_0^{(2)} =& (-0.75 + 2.5\delta Z + U_{0}^{S,(1)} + U_{0}^{S,(2)} + 0.5 U_{0}^{S,(3)})^{\frac{1}{3}}, \\
S_t^{(1)} =& (-0.45 + 0.3S_{t-1}^{(1)}(A_{t-1} - 0.5) + 0.3\delta S_{t-1}^{(1)}(Z - 0.5) + 0.45\delta (Z - 0.5)(A_{t-1} - 0.5) + \\
& 0.5U_{t}^{S,(1)} + 0.25 U_{t}^{S,(3)})^{\frac{1}{3}} \text{ (for $t \geq 1$)}, \\
S_t^{(2)} =& (-0.9 + 0.4S_{t-1}^{(1)}(A_{t-1} - 0.5) + 0.2S_{t-1}^{(2)}(A_{t-1} - 0.5) + 0.6\delta S_{t-1}^{(1)}(Z - 0.5) \\
&+ 0.2\delta S_{t-1}^{(2)}(Z - 0.5) + 0.9\delta (Z - 0.5)(A_{t-1} - 0.5) + 0.5U_{t}^{S,(1)} \\
&+ 0.9U_{t}^{S,(2)} + 0.5 U_{t}^{S,(3)})^{\frac{1}{3}} \text{ (for $t \geq 1$)}, \\
R_t = (&-0.3 + 0.2\delta S_t^{(1)}Z + 0.5S_t^{(1)}A_t + 0.2\delta S_t^{(2)}Z + 0.5S_t^{(2)}A_t - 1.0\delta ZA_t + 1.0U_{t}^R)^{\frac{1}{3}}.
\end{align*}
In this environment, we can view $U_0^{S, (1)} + 0.5U_0^{S, (3)}$ and $U_0^{S, (1)} + U_0^{S, (2)} + 0.5U_0^{S, (3)}$ as the scalar exogenous variables of $S_0^{(1)}$ and $S_0^{(2)}$, respectively. For $t \geq 1$, we can also view $0.5U_t^{S, (1)} + 0.25U_t^{S, (3)}$ and $0.5U_t^{S, (1)} + 0.9U_t^{S, (2)} + 0.5U_t^{S, (3)}$ as the scalar exogenous variables of $S_t^{(1)}$ and $S_t^{(2)}$, respectively. This CMDP satisfies Assumptions \ref{ass:state_monotonicity} and \ref{ass:reward_monotonicity} but does not satisfy the additivity assumption, which makes it a counterexample.
\end{proof}

\subsection{Proofs of Theorems \ref{thm:regret_bound} and \ref{thm:unfairness_bound}}
\label{subsec:proof_bounds}

\subsubsection{Notation}
We begin by introducing some notation. We first focus on notation for counterfactual state estimation. For convenience, we write  $s_t^{z,(i)}$ for $s_t^{Z \leftarrow z,(i)}(\bar{U}_t(h_t))$ and $s_t^{z}$ for $S_t^{Z \leftarrow z}(\bar{U}_t(h_t))$ for $t \geq 0$ and $i \in \{1, \dots, M\}$. For $i \in \{1, \dots, M\}$, let $Q_0^{(i)}(\omega, z)$ be the $\omega$-th quantile of the conditional distribution $\mathbb{P}(S_0^{(i)}|Z=z)$, and let $Q^{(i)}(\omega, z, s, a)$ be the $\omega$-th quantile of the conditional distribution $\mathbb{P}(S_t^{(i)}|Z=z, S_{t-1}=s, A_{t-1}=a)$ for $t \geq 1$. That is, 
\begin{align*}
    &Q_0^{(i)}(\omega, z) = \inf\{x \in \mathcal{S}: \mathbb{P}(S_0^{(i)} \leq x|Z=z) \geq \omega\}, \\
    &Q^{(i)}(\omega, z, s, a) = \inf\{x \in \mathcal{S}: \mathbb{P}(S_t^{(i)} \leq x|Z=z, S_{t-1}=s, A_{t-1}=a) \geq \omega\}.
\end{align*}
Note that the conditional quantile function is the same for time steps $t \geq 1$ because the environment is assumed to be stationary. Let $\hat{Q}_{0,n}^{(i)}(\omega, z)$ and $\hat{Q}_n^{(i)}(\omega, z, s, a)$ be estimators of $Q_0^{(i)}(\omega, z)$ and $Q^{(i)}(\omega, z, s, a)$, respectively. Moreover, for each individual $j$ and $t \geq 0$, let $\tau_{j,t}^{(i)}$ be the quantile level of the observed state $s_{j,t}^{(i)}$ on its conditional distribution under the observed sensitive attribute; that is, 
\begin{align*}
&\tau_{j,0}^{(i)} = \mathbb{P}(S_0^{(i)} \leq s_{j,0}^{(i)}|Z=z_j), \\
&\tau_{j,t}^{(i)} = \mathbb{P}(S_t^{(i)} \leq s_{j,t}^{(i)}|Z=z_j, S_{t-1}=s_{j,t-1}, A_{t-1}=a_{j,t-1}) \text{ (for $t \geq 1$)}
\end{align*}
where $z_j$ is the observed sensitive attribute of the $j$-th individual. Since $\tau_{j,t}^{(i)}$ is unobserved, let $\hat{\tau}_{j,t,n}^{(i)}$ be an estimator of $\tau_{j,t}^{(i)}$. For simplicity, we consider the case where the quantile level is estimated using the generalized inverse of the estimated conditional quantile functions. That is, 
\begin{align*}
    &\hat{\tau}_{j,0}^{(i)} = [\hat{Q}_{0,n}^{(i)}]^{-1}(s_{j,0}^{(i)}, z_j), \\
    &\hat{\tau}_{j,t}^{(i)} = [\hat{Q}_n^{(i)}]^{-1}(s_{j,t}^{(i)}, z_j, s_{j,t-1}, a_{j,t-1}) \text{ (for $t \geq 1$)}.
\end{align*}
We write $\tau_{j,t} = [\tau_{j,t}^{(1)}, \dots, \tau_{j,t}^{(M)}]^T$ and $\hat{\tau}_{j,t,n} = [\hat{\tau}_{j,t,n}^{(1)}, \dots, \hat{\tau}_{j,t,n}^{(M)}]^T$.

Now, for each individual $j$ and $t \geq 1$, we let the augmented state vector be 
\begin{align*}
&\tilde{s}_{j,0} = \begin{bmatrix} s_{j,0}^{z^{(1)}}(\tau_{j,0}) \\ \vdots \\ s_{j,0}^{z^{(K)}}(\tau_{j,0}) \end{bmatrix}, \text{ } \tilde{s}_{j,t} = \begin{bmatrix} s_{j,t}^{z^{(1)}}(\tau_{j,t}, s_{j,t-1}^{z^{(1)}}, a_{j,t-1}) \\ \vdots \\ s_{j,t}^{z^{(K)}}(\tau_{j,t}, s_{j,t-1}^{z^{(K)}}, a_{j,t-1}) \end{bmatrix},
\end{align*}
where, for each $z \in \{z^{(1)}, \dots, z^{(K)}\}$,
\begin{align*}
&s_{j,0}^z = \begin{bmatrix} Q_0^{(1)}(\tau_{j,0}^{(1)}, z) \\ \vdots \\ Q_0^{(M)}(\tau_{j,0}^{(M)}, z) \end{bmatrix}, \text{ }
s_{j,t}^z = \begin{bmatrix} Q^{(1)}(\tau_{j,t}^{(1)}, z, s_{j,t-1}^z, a_{j,t-1}) \\ \vdots \\ Q^{(M)}(\tau_{j,t}^{(M)}, z, s_{j,t-1}^z, a_{j,t-1}) \end{bmatrix}.
\end{align*}
Note that, under Assumption \ref{ass:state_monotonicity}, $Q_0^{(i)}(\omega, z)$ and $Q^{(i)}(\omega, z, s, a)$ are both strictly increasing in $\omega$. Therefore, for $i \in \{1, \dots, M\}$, $Q_0^{(i)}(\tau_{j,0}^{(i)}, z) = s_{j,0}^{z,(i)}$ and $Q^{(i)}(\tau_{j,t}^{(i)}, z, s_{j,t-1}^{z}, a_{j,t-1}) = s_{j,t}^{z,(i)}$ by Theorem \ref{thm:state_matching}. Since the above quantities are not fully observed, we estimate them by 
\begin{align*}
&\hat{\tilde{s}}_{j,0,n} = \begin{bmatrix} \hat{s}_{j,0,n}^{z^{(1)}}(\hat{\tau}_{j,0,n}) \\ \vdots \\ \hat{s}_{j,0,n}^{z^{(K)}}(\hat{\tau}_{j,0,n}) \end{bmatrix}, \text{ } \hat{\tilde{s}}_{j,t,n} = \begin{bmatrix} \hat{s}_{j,t,n}^{z^{(1)}}(\hat{\tau}_{j,t,n}, \hat{s}_{j,t-1,n}^{z^{(1)}}, a_{j,t-1}) \\ \vdots \\ \hat{s}_{j,t,n}^{z^{(K)}}(\hat{\tau}_{j,t,n}, \hat{s}_{j,t-1,n}^{z^{(K)}}, a_{j,t-1}) \end{bmatrix},
\end{align*}
where, for each $z \in \{z^{(1)}, \dots, z^{(K)}\}$,
\begin{align*}
&\hat{s}_{j,0,n}^z = \begin{bmatrix} \hat{Q}_{0,n}^{(1)}(\hat{\tau}_{j,0,n}^{(1)}, z) \\ \vdots \\ \hat{Q}_{0,n}^{(M)}(\hat{\tau}_{j,0,n}^{(M)}, z) \end{bmatrix}, \text{ }
\hat{s}_{j,t,n}^z = \begin{bmatrix} \hat{Q}_n^{(1)}(\hat{\tau}_{j,t,n}^{(1)}, z, \hat{s}_{j,t-1,n}^z, a_{j,t-1}) \\ \vdots \\ \hat{Q}_n^{(M)}(\hat{\tau}_{j,t,n}^{(M)}, z, \hat{s}_{j,t-1,n}^z, a_{j,t-1}) \end{bmatrix}.
\end{align*}

Now, we define similar notation for the rewards. For convenience, we write $r_t^{z}$ for $R_t^{Z \leftarrow z}(\bar{U}_{t+1}(h_{t+1}))$ for $t \geq 0$. Let $W(\omega, z, s, a)$ be the $\omega$-th quantile of the conditional distribution $\mathbb{P}(R_t|Z=z, S_{t}=s, A_{t}=a)$ for $t \geq 0$ (that is, $W(\omega, z, s, a) = \inf\{x \in \mathcal{R}: \mathbb{P}(R_t \leq x|Z=z, S_t=s, A_t=a) \geq \omega\}$), and let $\hat{W}_n(\omega, z, s, a)$ be an estimator of $W(\omega, z, s, a)$. Again, the conditional quantile function $W$ is fixed for all $t \geq 0$ because the environment is assumed to be stationary. Moreover, for each individual $j$ and $t \geq 0$, let $\theta_{j,t}$ be the quantile level of the observed reward $r_{j,t}$ on its conditional distribution under the observed sensitive attribute; that is, $\theta_{j,t} = \mathbb{P}(R_t \leq r_{j,t}|Z=z_j, S_t=s_{j,t}, A_t=a_{j,t})$ for $t \geq 0$. Since $\theta_{j,t}$ is usually unobserved, let $\hat{\theta}_{j,t,n}$ be an estimator of $\theta_{j,t}$. For simplicity, we consider the case where the quantile level is estimated using the generalized inverse of the estimated conditional quantile function. That is, $\hat{\theta}_{j,t} = \hat{W}_n^{-1}(r_{j,t}, z_j, s_{j,t}, a_{j,t})$. For each individual $j$ and $t \geq 0$, we let the augmented reward be 
\begin{align*}
&\tilde{r}_{j,t} = \sum_{k=1}^K \mathbb{P}(Z=z^{(k)}) W(\theta_{j,t}, z^{(k)}, s_{j,t}^{z^{(k)}}, a_t).
\end{align*}
Note that, under Assumption \ref{ass:reward_monotonicity}, $W(\omega, z, s, a)$ is strictly increasing in $\omega$, so Theorem \ref{thm:reward_matching} implies that $W(\theta_{j,t}, z, s_{j,t}^z, a_{j,t}) = r_{j,t}^z$. Since the counterfactual rewards $r_{j,t}^z$ are unobserved, we estimate the augmented rewards by 
\begin{align*}
&\hat{\tilde{r}}_{j,t,n} = \sum_{k=1}^K \hat{\mathbb{P}}_n(Z=z^{(k)}) \hat{W}_n(\hat{\theta}_{j,t,n}, z^{(k)}, \hat{s}_{j,t,n}^{z^{(k)}}, a_{j,t}).
\end{align*}
where $\mathbb{P}_n$ denotes the empirical probability measure.

For ease of notation, unless it is necessary to identify the individual associated with the variable, we omit the subscript $j$ for simplicity.

\begin{remark}
Recall that in our analysis, we consider estimating the quantile levels using the generalized inverse of the estimated conditional quantile functions. For the generalized inverse to be well-defined, techniques such as quantile rearrangement \citep{chernozhukov2010rearrangement} might be employed to ensure the monotonicity of the estimated conditional quantile functions.
\end{remark}

\subsubsection{Assumptions for Theoretical Analysis}
\label{subsubsec:assumptions}
In this section, we introduce the assumptions needed for our theoretical analysis. 

\begin{assumption}[Boundedness of State and Reward Spaces]
\label{ass:space}
The state space $\mathcal{S}$ is a bounded subset of $\mathbb{R}^M$. Then the space of the augmented state is $\tilde{\mathcal{S}} = \mathcal{S}^K$, and let $M_{\tilde{\mathcal{S}}} = \sup_{s,s' \in \tilde{\mathcal{S}}} \|s - s'\|_{\infty} < \infty$. The reward space $\mathcal{R}$ is also a bounded subset of $\mathbb{R}$, and $R_{max}$ is a positive constant such that $|r| \leq R_{max}$ for all $r \in \mathcal{R}$. Note that here the space $\tilde{\mathcal{R}}$ remains a bounded subset of $\mathbb{R}$.
\end{assumption}

\begin{assumption}
\label{ass:fqi}
Suppose the following hold:
\begin{enumerate}[leftmargin=*]
    \item (Class of the Q Function). \[ \mathcal{F} = \{w^T \phi(s, a) : w \in \mathbb{R}^d, \|w\|_1 \leq \mathcal{B}\}, \] where $\phi$ is a feature map $\tilde{\mathcal{S}} \times \mathcal{A} \to \mathbb{R}^d$ with $\|\phi(s, a)\|_2 \leq 1$ and component functions $\{\phi_i(s,a)\}_{i=1}^d$. Also, assume that $Q^*$ is estimated in FQI using ordinary least squares followed by Euclidean projection onto the ball $\{\omega: \|\omega\|_1 \leq \mathcal{B}\}$ at each iteration.
    \item (Completeness). Let $\mathcal{T}$ be the Bellman optimality operator. That is, for $f : \tilde{\mathcal{S}} \times \mathcal{A} \to \mathbb{R}$, \[ \mathcal{T}f(s, a) = r(s, a) + \gamma \mathbb{E}_{s' \sim P(\cdot|s, a)} \max_{a' \in \mathcal{A}} f(s', a'). \] Assume that $\mathcal{T}f \in \mathcal{F}$ for any $f \in \mathcal{F}$.
    \item (Feature Coverage). There exists a constant $\lambda_0 > 0$ such that the smallest eigenvalue of $\mathbb{E}_{s,a \sim \mu_b}[\phi(s, a)\phi(s, a)^T]$ is greater than or equal to $\lambda_0$. 
\end{enumerate}
\end{assumption}

\begin{assumption}[Lipschitz Continuous Conditional Quantile Functions and Feature Vector]
\label{ass:lipschitz_conditional_quantile}
The true conditional quantile functions $\{Q_0^{(i)}(\tau, z)\}_{i=1}^M$, $\{Q^{(i)}(\tau, z, s, a)\}_{i=1}^M$, and $W(\tau, z, s, a)$ are Lipschitz continuous in the quantile level $\tau$ with Lipschitz constant $H$. The true conditional quantile functions $\{Q^{(i)}(\tau, z, s, a)\}_{i=1}^M$, $W(\tau, z, s, a)$, and the feature vector components $\{\phi_i(s, a)\}_{i=1}^d$ of the Q function class $\mathcal{F}$ are Lipschitz continuous in the state $s$ with Lipschitz constant $L$ in the max norm. 
\end{assumption}

\begin{assumption}[Margin]
\label{ass:margin}
There exists some positive constant $\alpha>0$ such that
\begin{align*}
&\mathbb{P}_{\tilde{s}_t \sim \tilde{\pi}_n^zP_t} \left( \max_a Q^{*}(\tilde{s}_t, a) - \max_{a' \in \mathcal{A} \setminus \arg\max_a Q^{opt}(\tilde{s}_t, a)} Q^{*}(\tilde{s}_t, a') \leq \epsilon \middle| \Xi_n \right) = O(\epsilon^\alpha) \\
&\mathbb{P}_{\tilde{s}_t \sim \tilde{\pi}_n^zP_t} \left( \max_a Q^{*}(g_{t,n}^z(\tilde{h}_t), a) - \max_{a' \in \mathcal{A} \setminus \arg\max_a Q^{opt}(g_{t,n}^z(\tilde{h}_t), a)} Q^{*}(g_{t,n}^z(\tilde{h}_t), a') \leq \epsilon \middle| \Xi_n \right) \\
&= O(\epsilon^\alpha).
\end{align*}
for all $z \in \mathcal{Z}$ and $t \geq 0$, where $g_{t,n}^z$ is as defined in Section~\ref{sec:theory} of the main paper.
\end{assumption}

Assumption~\ref{ass:space} restricts the state and reward spaces to compact subsets of the Euclidean space. Assumption~\ref{ass:fqi} is needed to provide an upper bound for the estimation error of the optimal Q function in FQI, which is crucial for bounding the suboptimality gap and level of counterfactual unfairness. In particular, the completeness assumption is common in analyses of FQI \cite{hu2023fastratesregretoffline, chen2019information, perdomo2023linearOPE}, and the feature coverage assumption enables the convergence of the FQI estimator \cite{hu2023fastratesregretoffline}. Assumption~\ref{ass:lipschitz_conditional_quantile} requires that the true conditional quantile functions are Lipschitz continuous in the quantile level $\tau$ and the state $s$. It also requires that the feature vector of the Q function space $\mathcal{F}$ is Lipschitz continuous in the state $s$. Assumption~\ref{ass:margin} ensures that the optimal action can be distinguished from suboptimal ones with a nontrivial margin, which is needed for bounding the level of counterfactual unfairness.

\subsubsection{Uniform Bound of State and Reward Estimation Errors}
In this section, we derive bounds of the augmented state estimation error $\|\hat{\tilde{s}}_{j,t,n} - \tilde{s}_{j,t}\|_{\infty}$ and the augmented reward estimation error $|\hat{\tilde{r}}_{j,t,n} - \tilde{r}_{j,t}|$ that are uniform across all individuals regardless of their respective trajectories. Since state and reward estimation errors are a source of the suboptimality gap and counterfactual unfairness, bounding these errors is necessary for bounding the suboptimality gap and level of counterfactual unfairness.

\paragraph{Definition of the bounds.} Define \begin{align*}
    M_{0,n} =& \max_{i \in \{1, \dots, M\}} (\|\hat{Q}_{0,n}^{(i)} - Q_0^{(i)}\|_{\infty} + H \|[\hat{Q}_{0,n}^{(i)}]^{-1} - [Q^{(i)}_0]^{-1}\|_{\infty}) \\
    M_{t,n} =& \max_{i \in \{1, \dots, M\}} (\|\hat{Q}_n^{(i)} - Q^{(i)}\|_{\infty} + H \|[\hat{Q}_n^{(i)}]^{-1} - [Q^{(i)}]^{-1}\|_{\infty} + L M_{t-1,n}) \text{ (for $t \geq 1$)} \\
    D_{t,n} =& K \left(\max_{i \in \{1, \dots, M\}} (\|\hat{W}_n - W\|_{\infty} + H \|[\hat{W}_n]^{-1} - [W]^{-1}\|_{\infty} + L M_{t,n})\right) \\
    &+ R_{max} \sum_{z \in \mathcal{Z}} |\hat{\mathbb{P}}_n(Z=z) - \mathbb{P}(Z=z)| \text{ (for $t \geq 0$)}
    \\
    \delta_{t, n} =& \max_{t' \leq t} M_{t',n} + \max_{t' \leq t} D_{t',n} \text{ (for $t \geq 0$)},
\end{align*}
where the sup norm $\|\cdot\|_{\infty}$ is taken over all inputs of the corresponding function. Note that $M_{t,n}$, which will be shown later to bound the augmented state estimation error at time $t$, depends on the estimation accuracy of the conditional quantile functions for the states, the estimation accuracy of the quantile level of the observed state, and, if $t \geq 1$, the estimation accuracy of the previous time step's augmented state (which is reflected by $M_{t-1,n}$). Similarly, $D_{t,n}$, which will be shown later to bound the augmented reward estimation error at time $t$, depends on the estimation accuracy of the conditional quantile function for the reward, the estimation accuracy of the quantile level of the observed reward, the estimation accuracy of the current time step's augmented state (which is reflected by $M_{t,n}$), and the estimation accuracy of the distribution of the sensitive attribute in the population. Finally, $\delta_{t,n}$ depends on $M_{t',n}$ and $D_{t',n}$ for all time steps $t'$ before time $t$. For all $t \geq 0$, $M_{t,n}$, $D_{t,n}$, and $\delta_{t,n}$ are nonnegative and independent from the trajectory experienced by any individual.

Lemmas \ref{lem:state_bound} and \ref{lem:reward_bound} below establish uniform bounds for the estimation errors of the augmented states and rewards, respectively.

\begin{lemma}
\label{lem:state_bound}
Suppose Assumptions \ref{ass:state_monotonicity} and \ref{ass:lipschitz_conditional_quantile} hold. For each $t \geq 0$, $\|\hat{\tilde{s}}_{j,t,n} - \tilde{s}_{j,t}\|_{\infty} \leq M_{t,n}$ for all individuals $j$ regardless of the individual's trajectory.
\end{lemma}

\begin{proof}
We first show that $\|\hat{s}_{j,t,n}^{z} - s_{j,t}^{z}\|_{\infty} \leq M_{t,n}$ for 
$t \geq 0$ and any $z \in \mathcal{Z}$. Suppose the observed sensitive attribute is $z^*$. If $z=z^*$, then $\hat{s}_{t,n}^z = s_t^z$, so the statement holds trivially because $M_{t,n} \geq 0$ for all $t \geq 0$. Now we consider the scenario where $z \neq z^*$. In this scenario, we prove the statement by induction.

We first prove the case for $t = 0$. For each $i \in \{1, \dots, M\}$, we have 
\begin{align*}
& |\hat{s}_{j,0,n}^{z,(i)} - s_{j,0}^{z, (i)}| \\
=& |\hat{Q}_{0,n}^{(i)}(\hat{\tau}_{j,0,n}^{(i)},z) - Q_0^{(i)}(\tau_{j,0}^{(i)},z)| \\
\leq& |\hat{Q}_{0,n}^{(i)}(\hat{\tau}_{j,0,n}^{(i)},z) - Q_0^{(i)}(\hat{\tau}_{j,0,n}^{(i)},z)| + |Q_0^{(i)}(\hat{\tau}_{j,0,n}^{(i)},z) - Q_0^{(i)}(\tau_{j,0}^{(i)},z)| \\
\leq& \|\hat{Q}_{0,n}^{(i)} - Q_0^{(i)}\|_{\infty} + |Q_0^{(i)}(\hat{\tau}_{j,0,n}^{(i)},z) - Q_0^{(i)}(\tau_{j,0}^{(i)},z)| \\
\leq& \|\hat{Q}_{0,n}^{(i)} - Q_0^{(i)}\|_{\infty} + H |\hat{\tau}_{j,0,n}^{(i)} - \tau_{j,0}^{(i)}| \\
\leq& \|\hat{Q}_{0,n}^{(i)} - Q_0^{(i)}\|_{\infty} + H \|[\hat{Q}_{0,n}^{(i)}]^{-1} - [Q^{(i)}_0]^{-1}\|_{\infty}.
\end{align*}
Given the above result, we further have
\begin{align*}
\|\hat{s}_{j,0,n}^{z} - s_{j,0}^{z}\|_{\infty} =& \max_{i \in \{1, \dots, M\}} |\hat{s}_{j,0,n}^{z,(i)} - s_{j,0}^{z, (i)}| \\
\leq& \max_{i \in \{1, \dots, M\}} (\|\hat{Q}_{0,n}^{(i)} - Q_0^{(i)}\|_{\infty} + H \|[\hat{Q}_{0,n}^{(i)}]^{-1} - [Q^{(i)}_0]^{-1}\|_{\infty}) \\
=& M_{0,n},
\end{align*}
which proves the $t=0$ case.

Now we prove the inductive step. Let $t \geq 1$. For each $i \in \{1, \dots, M\}$, we have
\begin{align*}
& |\hat{s}_{j,t,n}^{z,(i)} - s_{j,t}^{z, (i)}| \\
=& |\hat{Q}_n^{(i)}(\hat{\tau}_{j,t,n}^{(i)},z, \hat{s}_{j,t-1,n}^z, a_{j,t-1}) - Q^{(i)}(\tau_{j,t}^{(i)},z, s_{j,t-1}^z, a_{j,t-1})| \\
\leq& |\hat{Q}_n^{(i)}(\hat{\tau}_{j,t,n}^{(i)},z, \hat{s}_{j,t-1,n}^z, a_{j,t-1}) - Q^{(i)}(\hat{\tau}_{j,t,n}^{(i)},z, \hat{s}_{j,t-1,n}^z, a_{j,t-1})| \\
&+ |Q^{(i)}(\hat{\tau}_{j,t,n}^{(i)},z, \hat{s}_{j,t-1,n}^z, a_{j,t-1}) - Q^{(i)}(\tau_{j,t}^{(i)},z, \hat{s}_{j,t-1,n}^z, a_{j,t-1})| \\
& + |Q^{(i)}(\tau_{j,t}^{(i)},z, \hat{s}_{j,t-1,n}^z, a_{j,t-1}) - Q^{(i)}(\tau_{j,t}^{(i)},z, s_{j,t-1}^z, a_{j,t-1})| \\
\leq& \|\hat{Q}_n^{(i)} - Q^{(i)}\|_{\infty} + H |\hat{\tau}_{j,t,n}^{(i)} - \tau_{j,t}^{(i)}| + L \|\hat{s}_{j,t-1,n}^z - s_{j,t-1}^z\|_{\infty} \\
\leq& \|\hat{Q}_n^{(i)} - Q^{(i)}\|_{\infty} + H \|[\hat{Q}_n^{(i)}]^{-1} - [Q^{(i)}]^{-1}\|_{\infty} + L M_{t-1,n}.
\end{align*}
Given the above result, we further have
\begin{align*}
& \|\hat{s}_{j,t,n}^{z} - s_{j,t}^{z}\|_{\infty} \\
=& \max_{i \in \{1, \dots, M\}} |\hat{s}_{j,t,n}^{z,(i)} - s_{j,t}^{z, (i)}| \\
\leq& \max_{i \in \{1, \dots, M\}} (\|\hat{Q}_n^{(i)} - Q^{(i)}\|_{\infty} + H \|[\hat{Q}_n^{(i)}]^{-1} - [Q^{(i)}]^{-1}\|_{\infty} + L M_{t-1,n}) \\
=& M_{t,n},
\end{align*}
which proves the inductive step. Therefore, $\|\hat{s}_{j,t,n}^{z} - s_{j,t}^{z}\|_{\infty} \leq M_{t,n}$ for 
$t \geq 0$ and any $z \in \mathcal{Z}$.

Finally, note that
\begin{align*}
&\|\hat{\tilde{s}}_{j,t,n} - \tilde{s}_{j,t}\|_{\infty} = \max_{z \in \mathcal{Z}} \|\hat{s}_{j,t,n}^{z} - s_{j,t}^{z}\|_{\infty} \leq M_{t,n},
\end{align*}
which completes the proof.
\end{proof}

\begin{lemma}
\label{lem:reward_bound}
Suppose Assumptions \ref{ass:state_monotonicity}-\ref{ass:reward_monotonicity}, \ref{ass:space}, and \ref{ass:lipschitz_conditional_quantile} hold. For each $t \geq 0$, $|\hat{\tilde{r}}_{j,t,n} - \tilde{r}_{j,t}| \leq D_{t,n}$ for all individuals $j$ regardless of the individual's trajectory.
\end{lemma}

\begin{proof}
We first show that $|\hat{r}_{j,t,n}^{z} - r_{j,t}^{z}| \leq \|\hat{W}_n - W\|_{\infty} + H \|[\hat{W}_n]^{-1} - [W]^{-1}\|_{\infty} + L M_{t,n}$ for $t \geq 0$ and any $z \in \mathcal{Z}$. Let $t\geq 0$. Suppose the observed sensitive attribute is $z^*$. If $z=z^*$, then $\hat{r}_{t,n}^z = r_t^z$, so the statement holds trivially. Now we consider the scenario where $z \neq z^*$. In this case, note that
\begin{align*}
& |\hat{r}_{j,t,n}^{z} - r_{j,t}^{z}| \\
=& |\hat{W}_n(\hat{\theta}_{j,t,n},z, \hat{s}_{j,t,n}^z, a_{j,t}) - W(\theta_{j,t},z, s_{j,t}^z, a_{j,t})| \\
\leq& |\hat{W}_n(\hat{\theta}_{j,t,n},z, \hat{s}_{j,t,n}^z, a_{j,t}) - W(\hat{\theta}_{j,t,n},z, \hat{s}_{j,t,n}^z, a_{j,t})| \\
&+ |W(\hat{\theta}_{j,t,n},z, \hat{s}_{j,t,n}^z, a_{j,t}) - W(\theta_{j,t},z, \hat{s}_{j,t,n}^z, a_{j,t})| \\
& + |W(\theta_{j,t},z, \hat{s}_{j,t,n}^z, a_{j,t}) - W(\theta_{j,t},z, s_{j,t}^z, a_{j,t})| \\
\leq& \|\hat{W}_n - W\|_{\infty} + H |\hat{\theta}_{j,t,n} - \theta_{j,t}| + L \|\hat{s}_{j,t,n}^z - s_{j,t}^z\|_{\infty} \\
\leq& \|\hat{W}_n - W\|_{\infty} + H \|[\hat{W}_n]^{-1} - [W]^{-1}\|_{\infty} + L M_{t,n}.
\end{align*}
Therefore, $|\hat{r}_{j,t,n}^{z} - r_{j,t}^{z}| \leq \|\hat{W}_n - W\|_{\infty} + H \|[\hat{W}_n]^{-1} - [W]^{-1}\|_{\infty} + L M_{t,n}$ for $t \geq 0$ and any $z \in \mathcal{Z}$.

Then,
\begin{align*}
& |\hat{\tilde{r}}_{j,t,n} - \tilde{r}_{j,t}| \\
\leq& |\sum_{z \in \mathcal{Z}} (\hat{\mathbb{P}}_n(Z=z) \hat{r}_{j,t,n}^{z} - \mathbb{P}(Z=z) r_{j,t}^{z})| \\
=& |\sum_{z \in \mathcal{Z}} \hat{\mathbb{P}}_n(Z=z) (\hat{r}_{j,t,n}^{z} - r_{j,t}^{z}) + \sum_{z \in \mathcal{Z}} (\hat{\mathbb{P}}_n(Z=z) - \mathbb{P}(Z=z)) r_{j,t}^{z}| \\
\leq& \sum_{z \in \mathcal{Z}} |\hat{\mathbb{P}}_n(Z=z)| | \hat{r}_{j,t,n}^{z} - r_{j,t}^{z}| + \sum_{z \in \mathcal{Z}} |\hat{\mathbb{P}}_n(Z=z) - \mathbb{P}(Z=z)| |r_{j,t}^{z}| \\
\leq& \sum_{z \in \mathcal{Z}} | \hat{r}_{j,t,n}^{z} - r_{j,t}^{z}| + R_{max} \sum_{z \in \mathcal{Z}} |\hat{\mathbb{P}}_n(Z=z) - \mathbb{P}(Z=z)| \\
\leq& \sum_{z \in \mathcal{Z}} (\|\hat{W}_n - W\|_{\infty} + H \|[\hat{W}_n]^{-1} - [W]^{-1}\|_{\infty} + L M_{t,n}) \\
&+ R_{max} \sum_{z \in \mathcal{Z}} |\hat{\mathbb{P}}_n(Z=z) - \mathbb{P}(Z=z)| \\
=& D_{t,n},
\end{align*}
which completes the proof.
\end{proof}

Now, we show that by combining $M_{t,n}$ and $D_{t,n}$ into a single bound, $\delta_{t,n}$ simultaneously bounds both the state and reward estimation errors up to some specified time point. 

\begin{lemma}
\label{lem:delta_convergence}
Suppose Assumptions \ref{ass:state_monotonicity}-\ref{ass:reward_monotonicity}, \ref{ass:space}, and \ref{ass:lipschitz_conditional_quantile} hold. Then, for all $t^* \geq 0$,
\begin{equation*}
\max_{t \leq t^*}\|\hat{\tilde{s}}_{j,t,n} - \tilde{s}_{j,t}\|_{\infty} + \max_{t \leq t^*}|\hat{\tilde{r}}_{j,t,n} - \tilde{r}_{j,t}| \leq \delta_{t^*,n}
\end{equation*}
\end{lemma}

\begin{proof}
By Lemmas \ref{lem:state_bound} and \ref{lem:reward_bound},
\begin{align*}
\max_{t \leq t^*}\|\hat{\tilde{s}}_{j,t,n} - \tilde{s}_{j,t}\|_{\infty} + \max_{t \leq t^*}|\hat{\tilde{r}}_{j,t,n} - \tilde{r}_{j,t}| \leq \max_{t \leq t^*}M_{t,n} + \max_{t \leq t^*}D_{t,n} = \delta_{t^*,n},
\end{align*}
which completes the proof.
\end{proof}

\subsubsection{Main Proofs of Theorems \ref{thm:regret_bound} and \ref{thm:unfairness_bound}}

\paragraph{Theorem \ref{thm:regret_bound}.} \textit{Suppose Assumptions~\ref{ass:state_monotonicity}-\ref{ass:reward_monotonicity} and \ref{ass:space}-\ref{ass:lipschitz_conditional_quantile} hold. Also, suppose the trajectory for each individual in the training set has exactly $T$ time steps. Let $\kappa$ be any positive constant satisfying $\kappa \log{n} \geq 1$. For any $\omega \in \mathbb{N}$ and $z \in \mathcal{Z}$, with probability at least $\left(1 - n^{-\kappa} - d \exp(-n\lambda_0/8)\right) \mathbb{P}\left(\delta_{T,n} < \frac{\lambda_0}{4dL}\right)$, we have $\delta_{T,n} < \frac{\lambda_0}{4dL}$ and 
\begin{align*}
    SG(\tilde{\pi}_n^z) &\le 2\mathcal{B}L\frac{1-\gamma^{\omega}}{1-\gamma}\max_{t \leq \omega-1}M_{t,n} + 2\mathcal{B}L\frac{\gamma^{\omega}}{1-\gamma}M_{\tilde{\mathcal{S}}} \\
    &\quad + \frac{2C_1 d L R_{\max} \delta_{T,n}}{\lambda_0(\lambda_0 - 4dL\delta_{T,n})(1-\gamma)^3} + \frac{2C_2 d R_{\max} \kappa \log(n)}{(1-\gamma)^3 \lambda_0 \sqrt{n}} + \frac{2\gamma^B R_{\max}}{(1-\gamma)^2},
\end{align*}
where $C_1$ and $C_2$ are positive constants, $B$ is the number of FQI iterations, $\gamma \in (0, 1)$ is the discount factor, and $\mathcal{B}$, $L$, $d$, $\lambda_0$, $R_{\text{max}}$, and $M_{\tilde{S}}$ are positive constants defined in Assumptions \ref{ass:space} and \ref{ass:fqi}.}

\begin{proof}
Let $\pi^*$ be the optimal policy in the Markov decision process with the augmented state and reward. By Theorem 2 of \cite{wang2025counterfactuallyfairreinforcementlearning}, $\pi^*$ is a stationary deterministic policy. Let $Q^*$ be the optimal Q function, and note that $Q^*$ is Lipschitz continuous in the sense that, for $s, s' \in \mathcal{S}$ and $a \in \mathcal{A}$, 
\begin{align*}
    |Q^*(s, a) - Q^*(s', a)| =& |(\phi(s,a) - \phi(s',a))^T \omega^*| \\
    \leq& \|\phi(s,a) - \phi(s',a)\|_{\infty} \|\omega^*\|_1 \\
    \leq& \mathcal{B} L \|s - s'\|_{\infty}.
\end{align*}
Also, let $\hat{f}$ be the Q function estimated by FQI, and let conditioning on $\Xi_n$ denote fixing the quantile estimators and empirical probabilities needed for preprocessing using SMDM. We have  
\begin{align}
    &SG(\tilde{\pi}_n^{z}) \nonumber \\
    =& \sum_{t=0}^{\infty} \gamma^t \mathbb{E}_{(\tilde{s}_t, a_t) \sim \tilde{\pi}_n^{z}P_t}[Q^*(\tilde{s}_t, \pi^*(\tilde{s}_t)) - Q^*(\tilde{s}_t, a_t) | \Xi_n] \text{ (note that $a_t = \tilde{\pi}_{t,n}^{z}(\tilde{h}_t)$)} \nonumber \\
    \leq& \sum_{t=0}^{\infty} \gamma^t \mathbb{E}_{(\tilde{s}_t, a_t) \sim \tilde{\pi}_n^{z}P_t}[Q^*(\tilde{s}_t, \pi^*(\tilde{s}_t)) - Q^*(\hat{\tilde{s}}_{t,n}, \pi^*(\tilde{s}_t)) + Q^*(\hat{\tilde{s}}_{t,n}, \pi^*(\tilde{s}_t)) - \hat{f}(\hat{\tilde{s}}_{t,n}, \pi^*(\tilde{s}_t)) \nonumber \\
    & + \hat{f}(\hat{\tilde{s}}_{t,n}, a_t) - Q^*(\hat{\tilde{s}}_{t,n}, a_t) + Q^*(\hat{\tilde{s}}_{t,n}, a_t) - Q^*(\tilde{s}_t, a_t) | \Xi_n] \nonumber \\
    \leq& \sum_{t=0}^{\infty} \gamma^t \mathbb{E}_{(\tilde{s}_t, a_t) \sim \tilde{\pi}_n^{z}P_t}[2\mathcal{B}L\|\hat{\tilde{s}}_{t,n} - \tilde{s}_t\|_{\infty} + 2\|Q^* - \hat{f}\|_{\infty} | \Xi_n] \nonumber \\
    =& 2\mathcal{B}L \sum_{t=0}^{\infty} \gamma^t \mathbb{E}_{(\tilde{s}_t, a_t) \sim \tilde{\pi}_n^{z}P_t}[\|\hat{\tilde{s}}_{t,n} - \tilde{s}_t\|_{\infty} | \Xi_n] + 2 \sum_{t=0}^{\infty} \gamma^t \mathbb{E}_{(\tilde{s}_t, a_t) \sim \tilde{\pi}_n^{z}P_t}[\|Q^* - \hat{f}\|_{\infty} | \Xi_n] \nonumber \\
    =& 2\mathcal{B}L \sum_{t=0}^{\infty} \gamma^t \mathbb{E}_{(\tilde{s}_t, a_t) \sim \tilde{\pi}_n^{z}P_t}[\|\hat{\tilde{s}}_{t,n} - \tilde{s}_t\|_{\infty} | \Xi_n] + \frac{2}{1-\gamma} \|Q^* - \hat{f}\|_{\infty} \nonumber \\
    =& 2\mathcal{B}L \left(\sum_{t=0}^{\omega-1} \gamma^t \mathbb{E}_{(\tilde{s}_t, a_t) \sim \tilde{\pi}_n^{z}P_t}[\|\hat{\tilde{s}}_{t,n} - \tilde{s}_t\|_{\infty} | \Xi_n] + \sum_{t=\omega}^{\infty} \gamma^t \mathbb{E}_{(\tilde{s}_t, a_t) \sim \tilde{\pi}_n^{z}P_t}[\|\hat{\tilde{s}}_{t,n} - \tilde{s}_t\|_{\infty} | \Xi_n]\right) \nonumber \\
    &+ \frac{2}{1-\gamma} \|Q^* - \hat{f}\|_{\infty} \nonumber \\
    \leq& 2\mathcal{B}L \left(\frac{1-\gamma^{\omega}}{1-\gamma}\max_{t \leq \omega-1}M_{t,n}+ \frac{\gamma^{\omega}}{1-\gamma}M_{\tilde{\mathcal{S}}}\right) + \frac{2}{1-\gamma} \|Q^* - \hat{f}\|_{\infty} \label{eq:regret_upper}
\end{align}
where the first equality is by the same argument as the one in the proof of Lemma 6.1 in \cite{kakade2002approximate}, and the first inequality is because $\hat{f}(\hat{\tilde{s}}_{t,n}, a_t) - \hat{f}(\hat{\tilde{s}}_{t,n}, \pi^*(\tilde{s}_t)) \geq 0$ by the greediness of $\hat{\pi}$ with respect to $\hat{f}(\hat{\tilde{s}}_{t,n}, \cdot)$. 

For the second term in \eqref{eq:regret_upper}, consider $n$ such that $n > d$. If $\delta_{T,n}  < \frac{\lambda_0}{4dL}$, then, by the same argument as the one in the proof of Theorem 3 in \cite{wang2025counterfactuallyfairreinforcementlearning}, 
\begin{align*}
& \|Q^* - \hat{f}\|_{\infty} \leq \sum_{t=0}^{B-1} \gamma^t \|\hat{f}_{B-t} - \mathcal{T}\hat{f}_{B-t-1}\|_{\infty} + \frac{\gamma^BR_{max}}{1-\gamma}, \\
& \mathbb{P}\left(\|\hat{f}_{B-t} - \mathcal{T}\hat{f}_{B-t-1}\|_{\infty} \leq \frac{C_1dLR_{max}\delta_{T,n}}{\lambda_0(\lambda_0 - 4dL\delta_{T,n})(1-\gamma)} + \frac{C_2dR_{max}\kappa \log(n)}{(1-\gamma)\lambda_0\sqrt{n}} \text{ } \forall t \middle| \Xi_n \right) \\
\geq& 1 - n^{-\kappa} - d \exp{(-n\lambda_0/8)}
\end{align*}
where $C_1$ and $C_2$ are positive constants and $\hat{f}_b$ denotes the estimated Q function after $b$ FQI iterations. This implies that
\begin{align}
    & \mathbb{P}\left(\|Q^* - \hat{f}\|_{\infty} \leq \frac{C_1dLR_{max}\delta_{T,n}}{\lambda_0(\lambda_0 - 4dL\delta_{T,n})(1-\gamma)^2} + \frac{C_2dR_{max}\kappa \log(n)}{(1-\gamma)^2\lambda_0\sqrt{n}} + \frac{\gamma^BR_{max}}{1-\gamma} \middle| \delta_{T,n}  < \frac{\lambda_0}{4dL}\right) \nonumber \\
    =& \mathbb{E}\Big[\mathbb{P}\Big(\|Q^* - \hat{f}\|_{\infty} \leq \frac{C_1dLR_{max}\delta_{T,n}}{\lambda_0(\lambda_0 - 4dL\delta_{T,n})(1-\gamma)^2} + \frac{C_2dR_{max}\kappa \log(n)}{(1-\gamma)^2\lambda_0\sqrt{n}} \nonumber \\
    &+ \frac{\gamma^BR_{max}}{1-\gamma} \Bigm| \delta_{T,n}  < \frac{\lambda_0}{4dL}, \Xi_n\Big) \Bigm| \delta_{T,n}  < \frac{\lambda_0}{4dL} \Big] \nonumber \\
    \geq& 1 - n^{-\kappa} - d \exp(-n\lambda_0/8). \label{eq:q_function_bound}
\end{align}

Combining the above results gives
\begin{align}
    &\mathbb{P}\Big(\Big\{SG(\tilde{\pi}_n^{z}) \leq \frac{2C_1dLR_{max}\delta_{T,n}}{\lambda_0(\lambda_0 - 4dL\delta_{T,n})(1-\gamma)^3} + \frac{2C_2dR_{max}\kappa \log(n)}{(1-\gamma)^3\lambda_0\sqrt{n}} + \frac{2\gamma^BR_{max}}{(1-\gamma)^2} \nonumber \\
    & + 2\mathcal{B}L\frac{1-\gamma^{\omega}}{1-\gamma}\max_{t \leq \omega-1}M_{t,n} + \frac{2\mathcal{B}L\gamma^{\omega}}{1-\gamma}M_{\tilde{\mathcal{S}}} \Big\} \cap \Big\{ \delta_{T,n}  < \frac{\lambda_0}{4dL} \Big\}\Big) \nonumber \\
    \geq& \mathbb{P}\Big(\Big\{2\mathcal{B}L\frac{1-\gamma^{\omega}}{1-\gamma}\max_{t \leq \omega-1}M_{t,n} + \frac{2\mathcal{B}L\gamma^{\omega}}{1-\gamma}M_{\tilde{\mathcal{S}}} + \frac{2}{1-\gamma} \|Q^* - \hat{f}\|_{\infty} \nonumber \\ & \leq \frac{2C_1dLR_{max}\delta_{T,n}}{\lambda_0(\lambda_0 - 4dL\delta_{T,n})(1-\gamma)^3} + \frac{2C_2dR_{max}\kappa \log(n)}{(1-\gamma)^3\lambda_0\sqrt{n}} + \frac{2\gamma^BR_{max}}{(1-\gamma)^2} \nonumber \\
    &+ 2\mathcal{B}L\frac{1-\gamma^{\omega}}{1-\gamma}\max_{t \leq \omega-1}M_{t,n} + \frac{2\mathcal{B}L\gamma^{\omega}}{1-\gamma}M_{\tilde{\mathcal{S}}}\Big\} \cap \Big\{ \delta_{T,n}  < \frac{\lambda_0}{4dL} \Big\}\Big) \nonumber \\
    \geq& \mathbb{P}\Big( \left\{ \|Q^* - \hat{f}\|_{\infty} \leq \frac{C_1dLR_{max}\delta_{T,n}}{\lambda_0(\lambda_0 - 4dL\delta_{T,n})(1-\gamma)^2} + \frac{C_2dR_{max}\kappa \log(n)}{(1-\gamma)^2\lambda_0\sqrt{n}} + \frac{\gamma^BR_{max}}{1-\gamma} \right\} \nonumber \\
    &\cap \left\{ \delta_{T,n}  < \frac{\lambda_0}{4dL} \right\} \Big) \nonumber \\
    =& \mathbb{P}\left(\|Q^* - \hat{f}\|_{\infty} \leq \frac{C_1dLR_{max}\delta_{T,n}}{\lambda_0(\lambda_0 - 4dL\delta_{T,n})(1-\gamma)^2} + \frac{C_2dR_{max}\kappa \log(n)}{(1-\gamma)^2\lambda_0\sqrt{n}} + \frac{\gamma^BR_{max}}{1-\gamma} \middle| \delta_{T,n}  < \frac{\lambda_0}{4dL}\right) \nonumber \\
    &\times \mathbb{P} \left(\delta_{T,n}  < \frac{\lambda_0}{4dL}\right) \nonumber \\
    \geq& \left(1 - n^{-\kappa} - d \exp(-n\lambda_0/8)\right) \mathbb{P} \left(\delta_{T,n}  < \frac{\lambda_0}{4dL}\right). \label{eq:regret_prob_lower}
\end{align}
where the first inequality is due to \eqref{eq:regret_upper} and the final inequality is due to \eqref{eq:q_function_bound}.

This completes the proof.
\end{proof}

\paragraph{Theorem \ref{thm:unfairness_bound}.} Suppose \textit{Assumptions~\ref{ass:state_monotonicity}-\ref{ass:reward_monotonicity} and \ref{ass:space}-\ref{ass:margin} hold. Also, suppose the trajectory for each individual in the training set has exactly $T$ time steps. Let $\xi_n = \left(\frac{C_1dLR_{max}\delta_{T,n}}{\lambda_0(\lambda_0 - 4dL\delta_{T,n})(1-\gamma)^2} + \frac{C_2dR_{max}\kappa \log(n)}{(1-\gamma)^2\lambda_0\sqrt{n}} + \frac{\gamma^BR_{max}}{1-\gamma}\right)I\left(\delta_{T,n} < \frac{\lambda_0}{4dL}\right)$ and $\eta_n = n^{-\kappa} + d\exp(-n\lambda_0/8)$, where $C_1, C_2$ are as defined in Theorem \ref{thm:regret_bound} and $\kappa$ is an arbitrary positive constant satisfying $\kappa \log{n} \geq 1$. For any $z', z'' \in \mathcal{Z}$ and $t \geq 0$, 
\begin{align*}
    \mathbb{E}\left[\|\tilde{\pi}^{z'}_{t,n}(\tilde{h}_t) - \tilde{\pi}^{z''}_{t,n}(\tilde{h}_t)\|_2\right] \leq& \mathbb{E} \left[O(\xi_n^{\alpha}) + O((\mathcal{B}L\delta_{t,n})^{\alpha})\right] + \sqrt{2}\left[\mathbb{P}\left(\delta_{T,n} \geq \lambda_0/(4dL)\right) + \eta_n\right]
\end{align*}
where $O$ stands for the big-$O$ notation, $B$ is the number of FQI iterations, $\gamma \in (0, 1)$ is the discount factor, and $L$, $d$, $\lambda_0$, $R_{max}$, and $\alpha$ are positive constants defined in Assumptions \ref{ass:space}, \ref{ass:fqi}, and \ref{ass:margin}, respectively.}

\begin{proof}
Let $\mathcal{A}_n$ represent the event that $\|Q^* - \hat{f}\| \leq \xi_n$, and let $\mathcal{B}_n$ represent the event that $\delta_{T,n}  < \frac{\lambda_0}{4dL}$. Also, let conditioning on $\Xi_n$ denote fixing the quantile estimators and empirical probabilities needed for preprocessing using SMDM. By the same argument as the one in the proof of Theorem 4 in \cite{wang2025counterfactuallyfairreinforcementlearning}, we have 
\begin{equation}
    \mathbb{E}_{(\tilde{s}_t,a_t) \sim \tilde{\pi}_n^{z'}P_t}[\|\tilde{\pi}^{z'}_{t,n}(\tilde{h}_t) - \tilde{\pi}^{z''}_{t,n}(\tilde{h}_t)\|_2 I\left(\mathcal{A}_n \cap \mathcal{B}_n\right) | \Xi_n] \leq O(\xi_n^{\alpha}) + O((\mathcal{B}L\delta_{t,n})^{\alpha}) \label{eq:unfairness_upper} 
\end{equation}
Then
\begin{align*}
    &\mathbb{E}_{(\tilde{s}_t,a_t) \sim \tilde{\pi}_n^{z'}P_t}\left[\|\tilde{\pi}^{z'}_{t,n}(\tilde{h}_t) - \tilde{\pi}^{z''}_{t,n}(\tilde{h}_t)\|_2 | \Xi_n\right] \\
    \leq& \mathbb{E}_{(\tilde{s}_t,a_t) \sim \tilde{\pi}_n^{z'}P_t}\left[\|\tilde{\pi}^{z'}_{t,n}(\tilde{h}_t) - \tilde{\pi}^{z''}_{t,n}(\tilde{h}_t)\|_2 I\left(\mathcal{A}_n \cap \mathcal{B}_n\right) \middle| \Xi_n \right] \\
    &+ \mathbb{E}_{(\tilde{s}_t,a_t) \sim \tilde{\pi}_n^{z'}P_t}\left[\|\tilde{\pi}^{z'}_{t,n}(\tilde{h}_t) - \tilde{\pi}^{z''}_{t,n}(\tilde{h}_t)\|_2 I\left(\mathcal{A}_n^C \cup \mathcal{B}_n^C\right) \middle| \Xi_n \right] \\
    \leq& \mathbb{E}_{(\tilde{s}_t,a_t) \sim \tilde{\pi}_n^{z'}P_t}\left[\|\tilde{\pi}^{z'}_{t,n}(\tilde{h}_t) - \tilde{\pi}^{z''}_{t,n}(\tilde{h}_t)\|_2 I\left(\mathcal{A}_n \cap \mathcal{B}_n\right) \middle| \Xi_n \right] + \sqrt{2} \mathbb{E}_{(\tilde{s}_t,a_t) \sim \tilde{\pi}_n^{z'}P_t}\left[I\left(\mathcal{A}_n^C \cup \mathcal{B}_n^C\right) \middle| \Xi_n \right] \\
    =& \mathbb{E}_{(\tilde{s}_t,a_t) \sim \tilde{\pi}_n^{z'}P_t}\left[\|\tilde{\pi}^{z'}_{t,n}(\tilde{h}_t) - \tilde{\pi}^{z''}_{t,n}(\tilde{h}_t)\|_2 I\left(\mathcal{A}_n \cap \mathcal{B}_n\right) \middle| \Xi_n \right] + \sqrt{2}\left[1 - \mathbb{P}\left(\mathcal{A}_n \cap \mathcal{B}_n\middle|\Xi_n\right) \right] \\
    \leq& O(\xi_n^{\alpha}) + O((\mathcal{B}L\delta_{t,n})^{\alpha}) + \sqrt{2}\left[1 - \mathbb{P}\left(\mathcal{A}_n\middle|\Xi_n, \mathcal{B}_n\right) \mathbb{P}\left(\mathcal{B}_n\middle|\Xi_n\right) \right] \\
    \leq& O(\xi_n^{\alpha}) + O((\mathcal{B}L\delta_{t,n})^{\alpha}) + \sqrt{2}\left[1 - (1 - \eta_n)I\left(\mathcal{B}_n\right) \right] \\
    =& O(\xi_n^{\alpha}) + O((\mathcal{B}L\delta_{t,n})^{\alpha}) + \sqrt{2}\left[I\left(\mathcal{B}_n^C\right) + \eta_n I\left(\mathcal{B}_n\right) \right], 
\end{align*}
where the fourth inequality holds both due to \eqref{eq:unfairness_upper} and because $\mathbb{P}\left(\mathcal{A}_n\middle|\Xi_n, \mathcal{B}_n\right) \geq 1 - \eta_n$ by the argument in the proof of Theorem 3 in \cite{wang2025counterfactuallyfairreinforcementlearning}. Marginalizing over $\Xi_n$ gives
\begin{align*}
    &\mathbb{E}\left[\|\tilde{\pi}^{z'}_{t,n}(\tilde{h}_t) - \tilde{\pi}^{z''}_{t,n}(\tilde{h}_t)\|_2\right] \\
    \leq& \mathbb{E}\left[ O(\xi_n^{\alpha}) + O((\mathcal{B}L\delta_{t,n})^{\alpha}) + \sqrt{2}\left[I\left(\mathcal{B}_n^C\right) + \eta_n I\left(\mathcal{B}_n\right) \right] \right] \\
    =& \mathbb{E} \left[O(\xi_n^{\alpha}) + O((\mathcal{B}L\delta_{t,n})^{\alpha})\right] + \sqrt{2}\left[\mathbb{P}\left(\mathcal{B}_n^C\right) + \eta_n \mathbb{P}\left(\mathcal{B}_n\right) \right] \\
    \leq& \mathbb{E} \left[O(\xi_n^{\alpha}) + O((\mathcal{B}L\delta_{t,n})^{\alpha})\right] + \sqrt{2}\left[\mathbb{P}\left(\mathcal{B}_n^C\right) + \eta_n \right],
\end{align*}
which completes the proof.
\end{proof}

\section{Additional Numerical Experiments}
\label{sec:additional_experiments}
This section presents some additional numerical experiments that evaluate the performance of CFSMDM against baseline methods. Unless otherwise specified, the methods being compared in the additional numerical experiments are ``Full'', ``Unaware'', ``Random'', FLAP\_M, ECOCF\_M, CFSDP, and CFSMDM, which follow the same definition as their counterparts in Section \ref{sec:experiments}. Also, similar to Section \ref{sec:experiments}, policy learning is performed using FQI, and we evaluate the performance of the methods based on the policy value and CF metric defined in Section \ref{sec:experiments}. More details about the implementation of the additional numerical experiments can be found in Section \ref{sec:details_experiments}.

\subsection{Additional Numerical Experiment on Fairness and Policy Value Under Varying $N$ and $\delta$}
\label{sec:additional_cf_value_experiment}

In this section, we conduct an additional numerical experiment to compare the level of counterfactual fairness and policy value achieved by CFSMDM and baseline methods under varying $N$ and $\delta$ using data generated from a known CMDP with \textit{additive} noise. Again, let $\delta$ denote the strength of impact of the sensitive attribute on the state and reward variables. This experiment aims to complement the experiment presented in section \ref{sec:experiments}, which compares the level of counterfactual fairness and policy value in a known CMDP with \textit{nonadditive} noise. 

We generate data from Data-generating CMDP 2 defined in Section \ref{sec:details_experiments}. This is a CMDP with univariate state variable and additive noise. Similar to Section \ref{sec:experiments}, we run two experiments with different objectives:
\begin{itemize}
    \item \textbf{Experiment 1':} This experiment evaluates how the policy value and the CF metric change with $N$, the number of individuals in the training dataset. In particular, for the training dataset, we fix $T=20, \delta=1.0$ and vary $N \in \{100, 200, 500, 1000, 2000\}$.
    \item \textbf{Experiment 2':} This experiment evaluates how the CF metric changes with $\delta$. In particular, for the training dataset, we fix $T=20, N=100$ and vary $\delta \in \{0.0, 0.5, 1.0, 1.5, 2.0\}$.
\end{itemize}
In both experiments, the training trajectories are generated using the ``Random'' policy as the behavioral policy. In particular, the sets of trajectories used for preprocessor training (if needed for the method of interest) and policy learning share the same sensitive attributes, while the remaining trajectory components are generated independently conditional on the shared sensitive attributes. The policy value resulting from the methods of interest is approximated by the discounted cumulative reward from a simulated trajectory with $10000$ individuals and $20$ horizons, generated from the known CMDP using the policy learned with the methods of interest. Similarly, the CF metric is calculated from a simulated dataset containing counterfactual trajectories of $10000$ individuals with $20$ horizons, generated from the known CMDP using the policy learned with the methods of interest. 

\paragraph{Results.} Figure~\ref{fig:results_additional_simulation} summarizes the experiment results. Panels (a)-(c) summarize the results from Experiment 1', and Panel (d) summarizes the results from Experiment 2'. In particular, panel (a) shows that CFSMDM and CFSDP achieves similar levels of performance in terms of fairness control, and both are significantly fairer than ``Full'' and ``Unaware''. Moreover, again, the CF metric achieved by CFSMDM decreases as the training sample size $N$ increases, which supports the consistency of CFSMDM. Meanwhile, while FLAP\_M and ECOCF\_M were fairer than ``Full'' and ``Unaware'', they still result in a noticeable level of unfairness. In panel (b), ``Full'' achieves the highest policy value because it uses the original states and rewards for reward maximization. CFSMDM and CFSDP achieve similar levels of policy values, but their policy values are lower than those of ``Full'' and ``Unaware'', which again demonstrates a tradeoff between fairness and policy value. Panel (c) shows that fairer methods tend to have lower policy values. In panel (d), CFSMDM remains relatively fair as we inject more unfairness into the underlying CMDP by increasing $\delta$, which shows that CFSMDM is robust against the level of unfairness present in the training trajectories.

\begin{figure}
    \centering
    \includegraphics[width=1\linewidth]{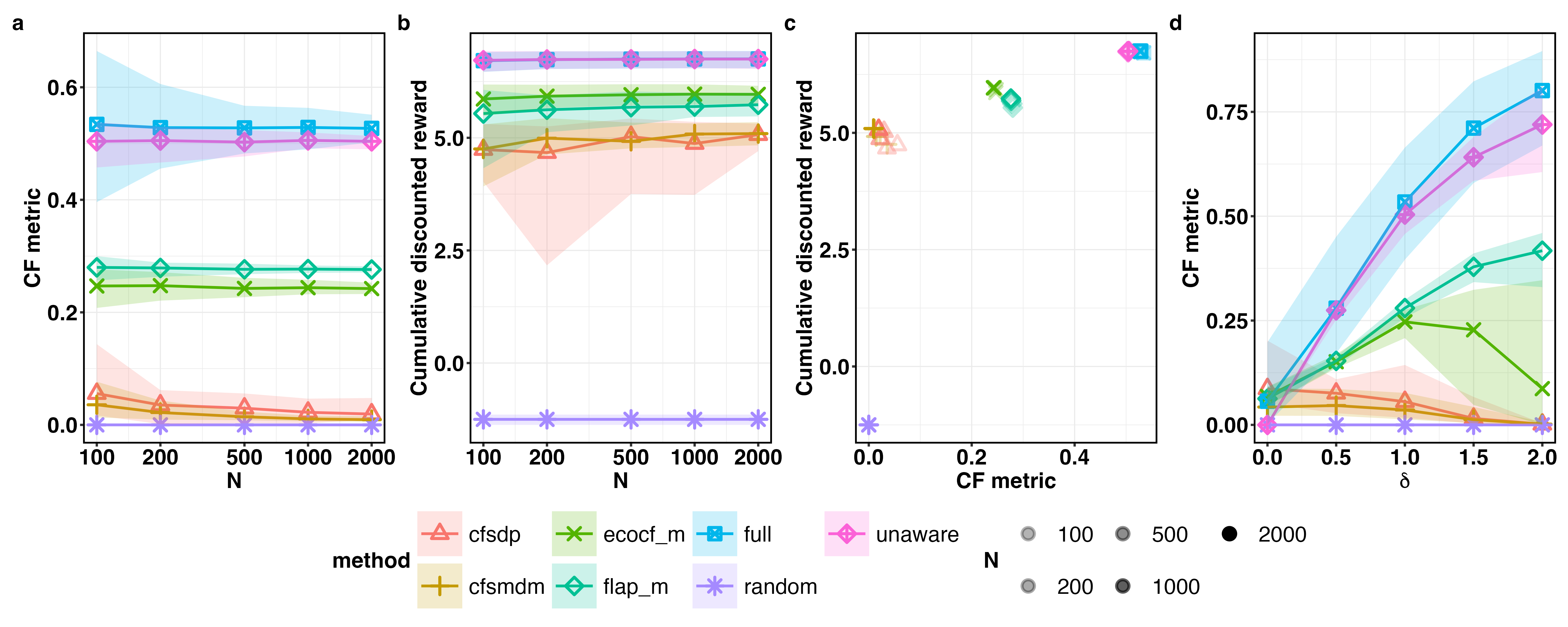}
    \caption{Experiment results under additive noise (CMDP 2): (a) CF metric vs. Sample size ($N$); (b) Policy value vs. Sample size ($N$); (c) Policy value vs. CF metric; (d) CF metric vs. Size of effect of $Z$ on $S_t$ and $R_t$ ($\delta$). Results are aggregated over $50$ replications with different random seeds. Point represents the empirical mean. Shaded area represents the band between $0.025$ and $0.975$ empirical quantiles.}
    \label{fig:results_additional_simulation}
\end{figure}

\subsection{Additional Numerical Experiment on Fairness Under Varying $T$}
\label{sec:additional_T_experiment}

In this section, we conduct an additional numerical experiment to compare the level of counterfactual unfairness achieved by CFSMDM and baseline methods under varying $T$ (number of horizons in the training trajectory). The experiment is repeated for both Data-generating CMDP 1, which has bivariate state variable and nonadditive noise, and Data-generating CMDP 2, which has univariate state variable and additive noise. Both CMDPs are defined in Section \ref{sec:details_experiments}. In particular, for Data-generating CMDP 1, we fix $N=500$, $\delta=2.0$, and vary $T \in \{10, 20, 30, 50, 100\}$; for Data-generating CMDP 2, we fix $N=100$, $\delta=1.0$, and vary $T \in \{10, 20, 30, 50, 100\}$.

For both CMDPs, the training trajectories are generated using the ``Random'' policy as the behavioral policy. In particular, the sets of trajectories used for preprocessor training (if needed for the method of interest) and policy learning share the same sensitive attributes, while the remaining trajectory components are generated independently conditional on the shared sensitive attributes. For each $T \in \{10, 20, 30, 50, 100\}$, the CF metric is calculated from a simulated dataset containing counterfactual trajectories of $10000$ individuals with $T$ horizons, generated from the known CMDP using the policy learned with the methods of interest.

\paragraph{Results.} Figure \ref{fig:results_simulation_T} summarizes the experiment results. In both CMDPs, the CF metric achieved by CFSMDM decreases slightly as $T$ increases. This shows that, although the counterfactual state and reward estimation errors might accumulate across time, CFSMDM controls the level of unfairness successfully for large levels of $T$. One possible explanation for this is that since the transition tuples used for SMDM training in the experiment are pooled across time, a larger $T$ provides more tuples for training the quantile models; this might improve the estimation accuracy of the counterfactual states and rewards.

\begin{figure}
    \centering
    \includegraphics[width=0.75\linewidth]{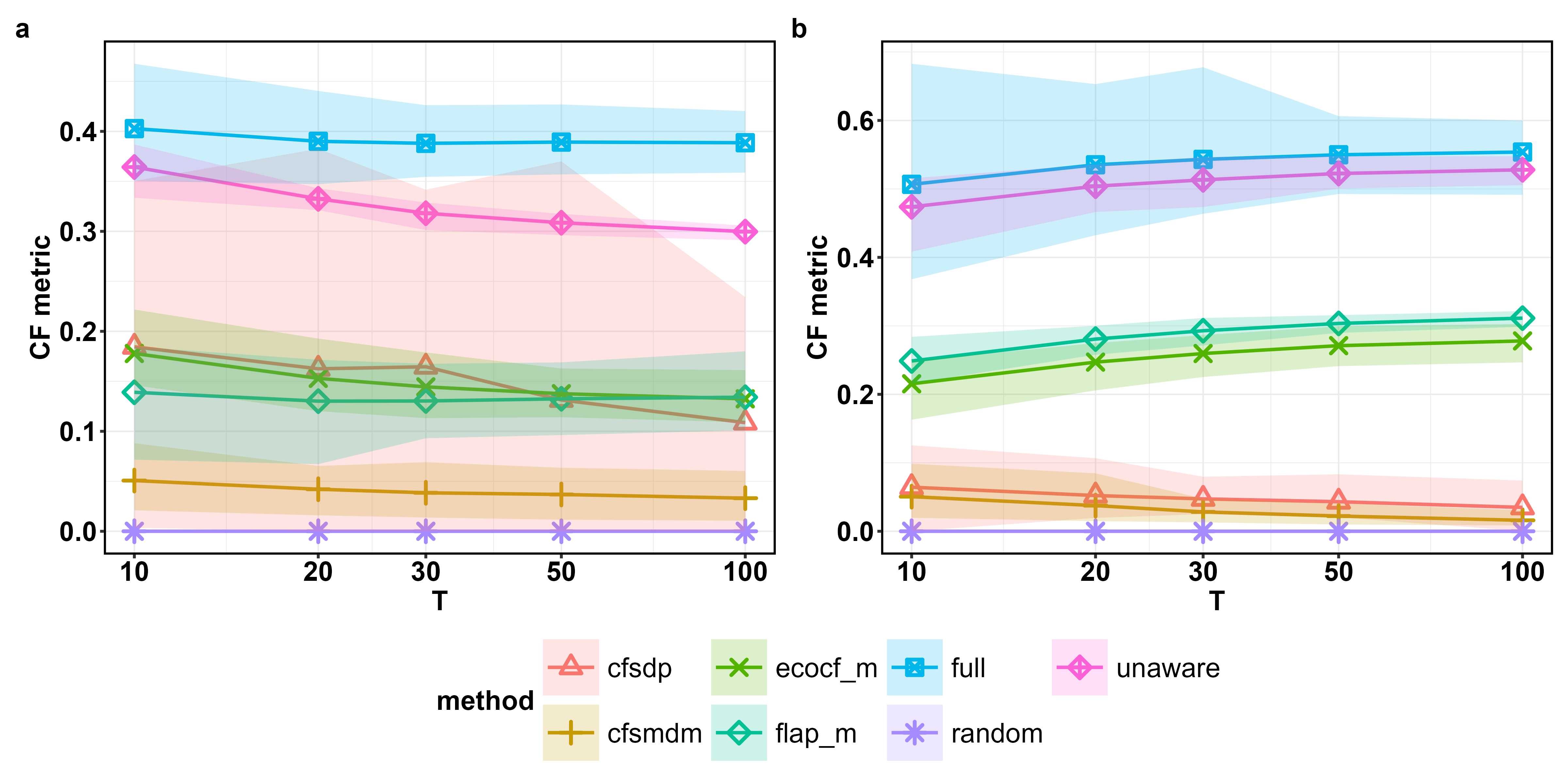}
    \caption{Level of counterfactual unfairness under varying levels of $T$. Panel (a) presents results under Data-generating CMDP 1. Panel (b) presents results under Data-generating CMDP 2. Results are aggregated over $50$ replications with different random seeds. Point represents the empirical mean. Shaded area represents the band between $0.025$ and $0.975$ empirical quantiles.}
    \label{fig:results_simulation_T}
\end{figure}

\subsection{Additional Numerical Experiments on Higher-dimensional CMDPs}
\label{sec:additional_hd_experiment}

In this section, we conduct an additional numerical experiment to compare the level of counterfactual fairness and policy value achieved by CFSMDM and baseline methods in higher-dimensional CMDPs. The data is generated from CMDP 3 defined in Section \ref{sec:details_experiments}, which has a 20-dimensional state variable. In this experiment, to improve the flexibility of the quantile model, CFSMDM is implemented using quantile regression neural network, and we denote it by ``CFSMDM\_QRNN'' to distinguish it from the linear quantile regression-based implementation used in most other experiments. Details on the implementation of ``CFSMDM\_QRNN'' can be found in Section \ref{sec:details_experiments}. We fix $T=20, \delta=1.0$ and vary $N \in \{500, 1000, 2000, 3000, 5000\}$. 

The training trajectories are generated using the ``Random'' policy as the behavioral policy. In particular, the sets of trajectories used for preprocessor training (if needed for the method of interest) and policy learning are generated independently. The policy value resulting from the methods of interest is approximated by the discounted cumulative reward from a simulated trajectory with $10000$ individuals and $20$ horizons, generated from the known CMDP using the policy learned with the methods of interest. Similarly, the CF metric is calculated from a simulated dataset containing counterfactual trajectories of $10000$ individuals with $20$ horizons, generated from the known CMDP using the policy learned with the methods of interest. 

\begin{figure}
    \centering
    \includegraphics[width=0.9\linewidth]{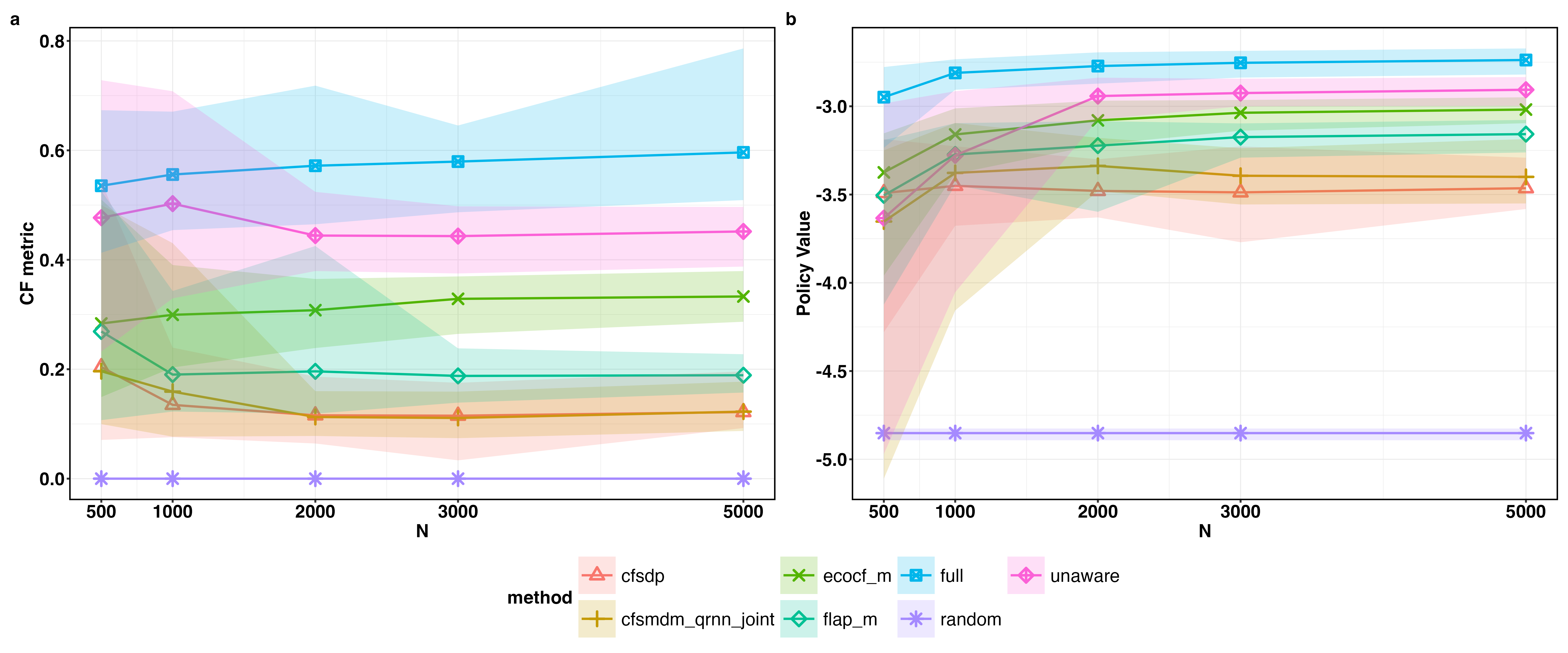}
    \caption{Experiment results in CMDP with 20-dimensional state (CMDP 3): (a) CF metric vs. Sample size ($N$); (b) Policy value vs. Sample size ($N$). Results are aggregated over $50$ replications with different random seeds. Point represents the empirical mean. Shaded area represents the band between $0.025$ and $0.975$ empirical quantiles.}
    \label{fig:results_hd}
\end{figure}

\paragraph{Results.} Figure \ref{fig:results_hd} shows the experiment results. Compared to the other experiments, a larger sample size is needed in CMDP 3 for CFSMDM to reach a relatively low CF metric, one possible explanation being that the high dimensionality tends to inflate the combined counterfactual estimation errors over all state components. In general, CFSMDM achieves a relatively low CF metric under moderate to large sample sizes in CMDP 3. The fairness control performance of CFSMDM in CMDP 3 is comparable to that of CFSDP and better than those of all other non-random baselines. 

\subsection{Additional Numerical Experiments on the Performance of CFSMDM based on Demographic Parity}
\label{sec: additional_dp_experiment}

In this section, we conduct additional numerical experiments to compare the fairness performance of CFSMDM and baseline methods with respect to demographic parity instead of counterfactual fairness. This experiment aims to assess the performance of CFSMDM when a fairness definition other than counterfactual fairness is used. The experiment is repeated for both CMDP 1 and CMDP 2. For CMDP 1, we fix $T=20, \delta=2.0$ and vary $N \in \{100, 200, 500, 1000, 2000\}$. For CMDP 2, we fix $T=20, \delta=1.0$ and vary $N \in \{100, 200, 500, 1000, 2000\}$.

The level of demographic parity is measured by a DP metric defined as follows:
\begin{equation*}
\text{DP Metric} = \frac{1}{T} \sum_{t=0}^{T-1} \left| \frac{1}{N_1} \sum_{i:Z_i=1} I(A_{i,t}=1) - \frac{1}{N_0} \sum_{i:Z_i=0} I(A_{i,t}=1) \right|,
\end{equation*}
where $N_1$ is the number of individuals with $Z=1$ and $N_0$ is the number of individuals with $Z=0$. A lower DP metric corresponds to a higher level of fairness defined by demographic parity. The training trajectories are generated using the ``Random'' policy as the behavioral policy. In particular, the sets of trajectories used for preprocessor training (if needed for the method of interest) and policy learning are generated independently. The DP metric is calculated from a simulated trajectory with $10000$ individuals and $20$ horizons, generated from the known CMDP using the policy learned with the methods of interest.

\paragraph{Results.} Figure \ref{fig:results_dp} shows the experiment results. Notably, CFSMDM achieves the highest level of demographic parity among the non-random methods in both CMDPs.

\begin{figure}
    \centering
    \includegraphics[width=0.9\linewidth]{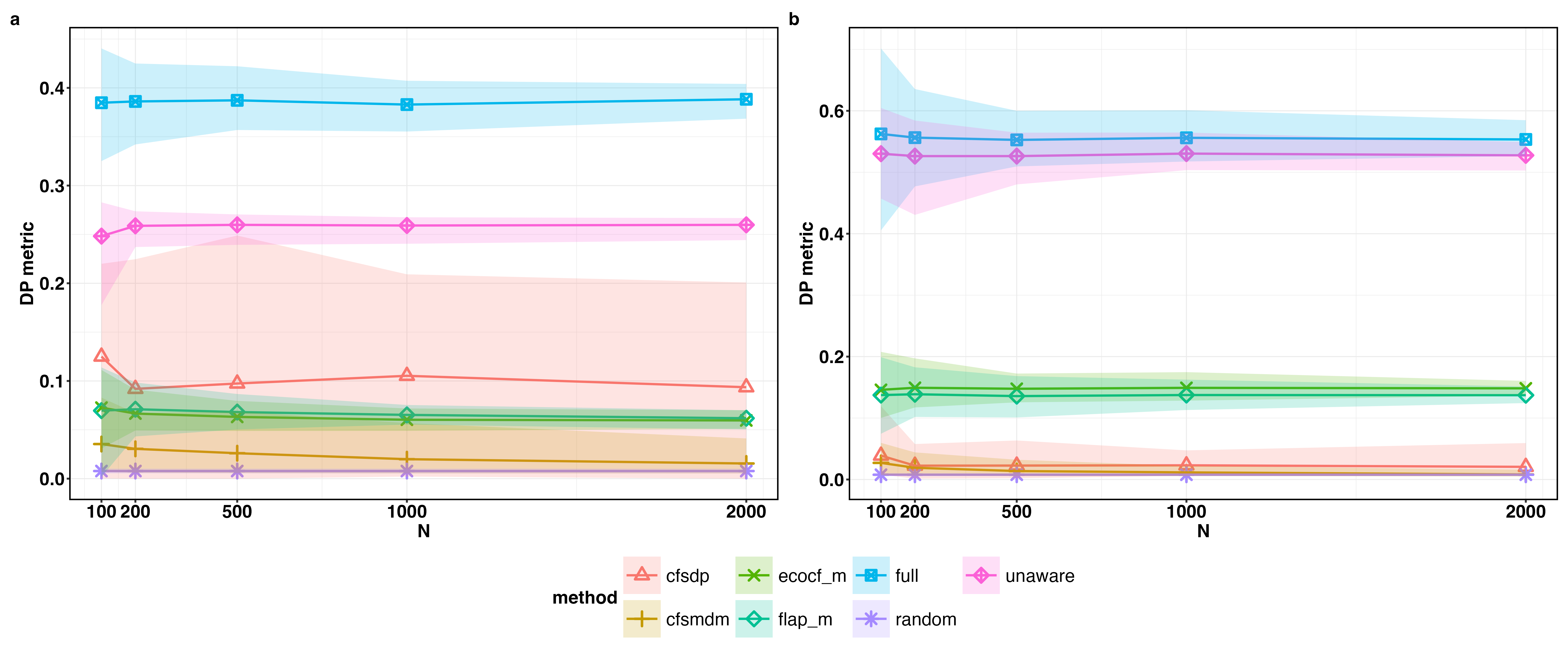}
    \caption{Level of demographic parity under varying levels of $N$. Panel (a) presents results under Data-generating CMDP 1. Panel (b) presents results under Data-generating CMDP 2. Results are aggregated over $50$ replications with different random seeds. Point represents the empirical mean. Shaded area represents the band between $0.025$ and $0.975$ empirical quantiles.}
    \label{fig:results_dp}
\end{figure}

\subsection{Additional Numerical Experiments on Sensitivity of CFSMDM to Violations of the Single Scalar Noise Formulation}
\label{sec: additional_single_scalar_noise_experiment}

CFSMDM requires that each component of the state variable and the reward in the underlying CMDP have a single scalar-valued exogenous noise variable. In this section, we conduct additional numerical experiments to compare the level of counterfactual fairness and policy value achieved by CFSMDM and baseline methods when this single scalar noise assumption is violated. The experiment is repeated for CMDPs 4 and 5, both of which are defined in Section \ref{sec:details_experiments}. For both CMDPs, a parameter $v$ controls the CMDP's degree of deviation from the single scalar noise formulation. A larger $v$ generally implies greater deviation, while $v=0$ means the CMDP satisfies the single scalar noise formulation. For CMDP 4, we fix $N=500, T=20, \delta=2.0$, and vary $v \in \{0.0, 0.1, 0.2, 0.5, 1.0, 1.5, 2.0\}$; for CMDP 5, we fix $N=500, T=20, \delta=1.0$, and vary $v \in \{0.0, 0.1, 0.2, 0.5, 1.0, 1.5, 2.0\}$

The training trajectories are generated using the ``Random'' policy as the behavioral policy. In particular, the sets of trajectories used for preprocessor training (if needed for the method of interest) and policy learning are generated independently. The policy value resulting from the methods of interest is approximated by the discounted cumulative reward from a simulated trajectory with $10000$ individuals and $20$ horizons, generated from the known CMDP using the policy learned with the methods of interest. Similarly, the CF metric is calculated from a simulated dataset containing counterfactual trajectories of $10000$ individuals with $20$ horizons, generated from the known CMDP using the policy learned with the methods of interest. 

\paragraph{Results.} Figure \ref{fig:results_single_scalar_noise_sensitivity} shows the experiment results. As the degree of deviation from the single scalar variable formulation increases, the mean CF metric for CFSMDM in CMDP 4 plateaus at around $0.1$, which indicates that CFSMDM's fairness control performance remains robust across the range of $v$ values tested in CMDP 4; the mean CF metric for CFSMDM in CMDP 5 increases at a mild rate and remains low under mild to moderate deviations. Across all $v$ values tested, CFSMDM is among the best performers in terms of fairness control among the non-random methods in both CMDPs. Deviation from the single scalar noise formulation also appears to have insignificant effect on the policy value resulting from CFSMDM relative to the other baselines.

\begin{figure}
    \centering
    \includegraphics[width=0.9\linewidth]{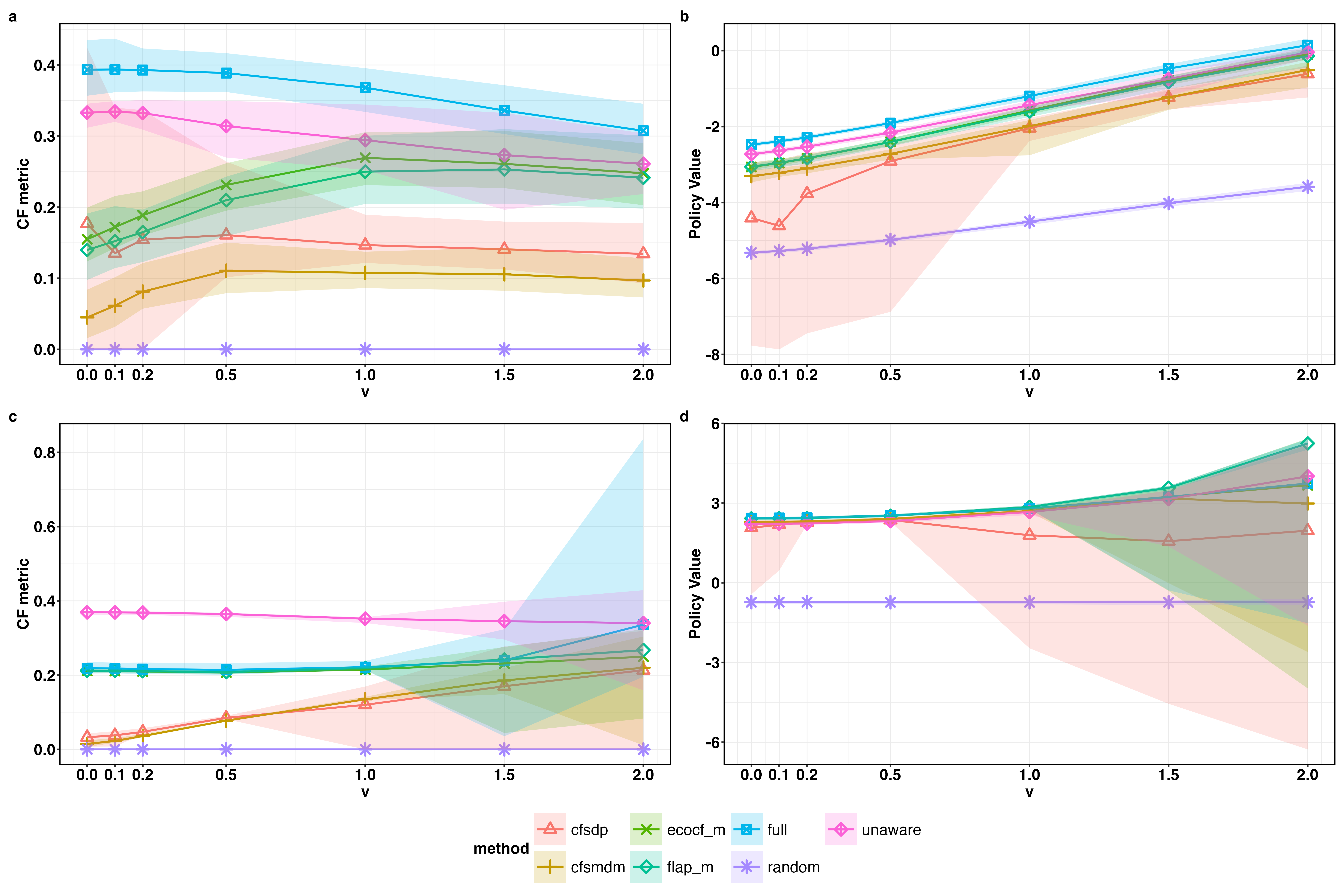}
    \caption{Experiment results for the sensitivity to violations of the single scalar noise formulation: (a) CF metric vs. $v$ under CMDP 4; (b) Policy value vs. $v$ under CMDP 4; (c) CF metric vs. $v$ under CMDP 5; (d) Policy value vs. $v$ under CMDP 5. Results are aggregated over $50$ replications with different random seeds. Point represents the empirical mean. Shaded area represents the band between $0.025$ and $0.975$ empirical quantiles.}
    \label{fig:results_single_scalar_noise_sensitivity}
\end{figure}

\subsection{Additional Numerical Experiments on Sensitivity of CFSMDM to Violations of the Strict Monotonicity Assumption}
\label{sec: additional_monotonicity_experiment}

CFSMDM assumes that each component of the state variable is strictly monotonic in its exogenous noise variable i.e. Assumptions \ref{ass:state_monotonicity} and \ref{ass:reward_monotonicity}. In this section, assuming the single scalar noise formulation is satisfied, we conduct additional numerical experiments to compare the level of counterfactual fairness and policy value achieved by CFSMDM and baseline methods when the strict monotonicity assumption is violated. The experiment is repeated for CMDPs 6 and 7, both of which are defined in Section \ref{sec:details_experiments}. For both CMDPs, a parameter $u$ controls the CMDP's degree of deviation from the strict monotonicity assumption. A larger $u$ generally implies greater deviation, while $u=0$ means the CMDP satisfies the strict monotonicity assumption. For CMDP 6, we fix $N=500, T=20, \delta=1.0$, and vary $u \in \{0.0, 0.1, 0.2, 0.5, 1.0, 1.5, 2.0\}$; for CMDP 7, we fix $N=500, T=20, \delta=2.0$, and vary $u \in \{0.0, 0.1, 0.2, 0.5, 1.0, 1.5, 2.0\}$

The training trajectories are generated using the ``Random'' policy as the behavioral policy. In particular, the sets of trajectories used for preprocessor training (if needed for the method of interest) and policy learning are generated independently. The policy value resulting from the methods of interest is approximated by the discounted cumulative reward from a simulated trajectory with $10000$ individuals and $20$ horizons, generated from the known CMDP using the policy learned with the methods of interest. Similarly, the CF metric is calculated from a simulated dataset containing counterfactual trajectories of $10000$ individuals with $20$ horizons, generated from the known CMDP using the policy learned with the methods of interest. 

\paragraph{Results.} Figure \ref{fig:results_monotonicity_sensitivity} shows the experiment results. As the degree of deviation from the strict monotonicity assumption increases, the mean CF metric for CFSMDM in CMDP 7 remains approximately below $0.1$ across the range of $u$ values tested; the CF metric for CFSMDM in CMDP 6 increases at a mild rate and remains low under mild to moderate deviations. Across all $u$ values tested, CFSMDM is among the best performers in terms of fairness control among the non-random methods in both CMDPs. Deviation from the strict monotonicity assumption also appears to have insignificant effect on the policy value resulting from CFSMDM relative to the other baselines.

\begin{figure}
    \centering
    \includegraphics[width=0.9\linewidth]{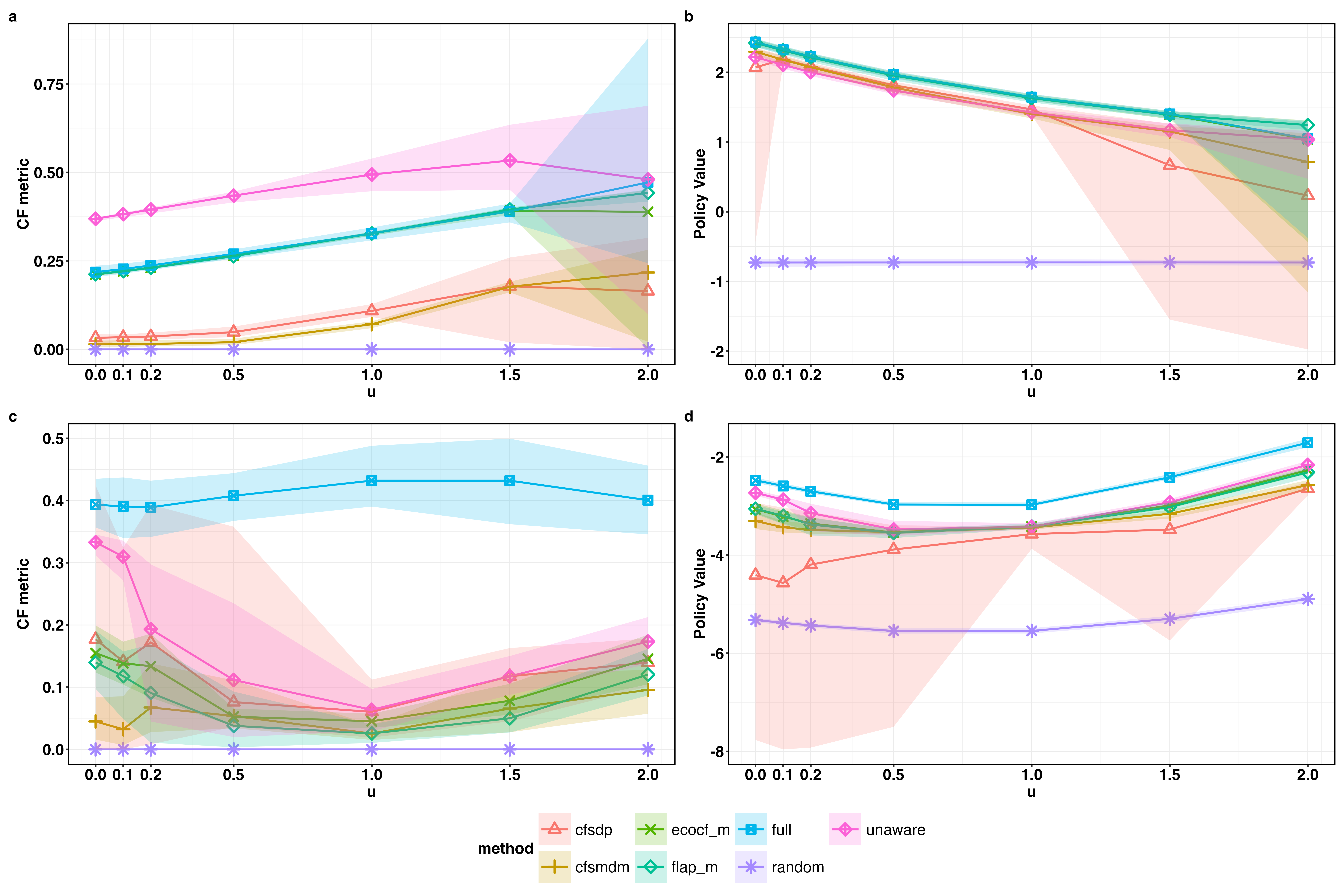}
    \caption{Experiment results for the sensitivity to violations of the strict monotonicity assumption: (a) CF metric vs. $u$ under CMDP 6; (b) Policy value vs. $u$ under CMDP 6; (c) CF metric vs. $u$ under CMDP 7; (d) Policy value vs. $u$ under CMDP 7. Results are aggregated over $50$ replications with different random seeds. Point represents the empirical mean. Shaded area represents the band between $0.025$ and $0.975$ empirical quantiles.}
    \label{fig:results_monotonicity_sensitivity}
\end{figure}

\subsection{Additional Numerical Experiments on the Performance of CFSMDM When Conservative Q Learning (CQL) is Used for Policy Learning}
\label{sec:additional_cql_experiments}

In this section, we conduct additional experiments to compare the performance of CFSMDM and baseline methods when policy learning is performed using CQL \citep{kumar2020cql} instead of FQI. The experiment is repeated for CMDP 1, which has bivariate state variable and nonadditive noise, and CMDP 2, which has univariate state variable and additive noise. Both CMDPs are defined in detail in Section \ref{sec:details_experiments}. 

Similar to Sections \ref{sec:experiments} and \ref{sec:additional_experiments}, we run two experiments with different objectives:

\begin{itemize}
    \item \textbf{Experiment 1'':} This experiment evaluates how the policy value and the CF metric change with $N$, the number of individuals in the training dataset. In particular, for the training dataset, we fix $T=20, \delta=2.0$ and vary $N \in \{100, 200, 500, 1000, 2000\}$ for CMDP 1; and we fix $T=20, \delta=1.0$ and vary $N \in \{100, 200, 500, 1000, 2000\}$ for CMDP 2.
    \item \textbf{Experiment 2'':} This experiment evaluates how the CF metric changes with $\delta$, which is the strength of impact of the sensitive attribute on the state and reward variables. In particular, for the training dataset, we fix $T=20, N=500$ and vary $\delta \in \{0.0, 0.5, 1.0, 1.5, 2.0\}$ for CMDP 1; and we fix $T=20, N=100$ and vary $\delta \in \{0.0, 0.5, 1.0, 1.5, 2.0\}$ for CMDP 2.
\end{itemize}

The training trajectories are generated using the ``Random'' policy as the behavioral policy. In particular, the sets of trajectories used for preprocessor training (if needed for the method of interest) and policy learning are generated independently. The policy value resulting from the methods of interest is approximated by the discounted cumulative reward from a simulated trajectory with $10000$ individuals and $20$ horizons, generated from the known CMDP using the policy learned with the methods of interest. Similarly, the CF metric is calculated from a simulated dataset containing counterfactual trajectories of $10000$ individuals with $20$ horizons, generated from the known CMDP using the policy learned with the methods of interest. 

\paragraph{Results.} Figures \ref{fig:results_m1_cql} and \ref{fig:results_m2_cql} show the experiment results under CMDP 1 and CMDP 2, respectively. For both CMDPs, the results are generally similar to the results obtained using FQI. In particular, in CMDP 1, CFSMDM achieves a high level of counterfactual fairness, while the fairness control performance of CFSDP is relatively undesirable and unstable, possibly due to the violation of the additive noise assumption; in CMDP 2, both CFSMDM and CFSDP achieves a high level of counterfactual fairness. The fairness control performance of CFSMDM also appears robust to high levels of $\delta$.

\begin{figure}
    \centering
    \includegraphics[width=1\linewidth]{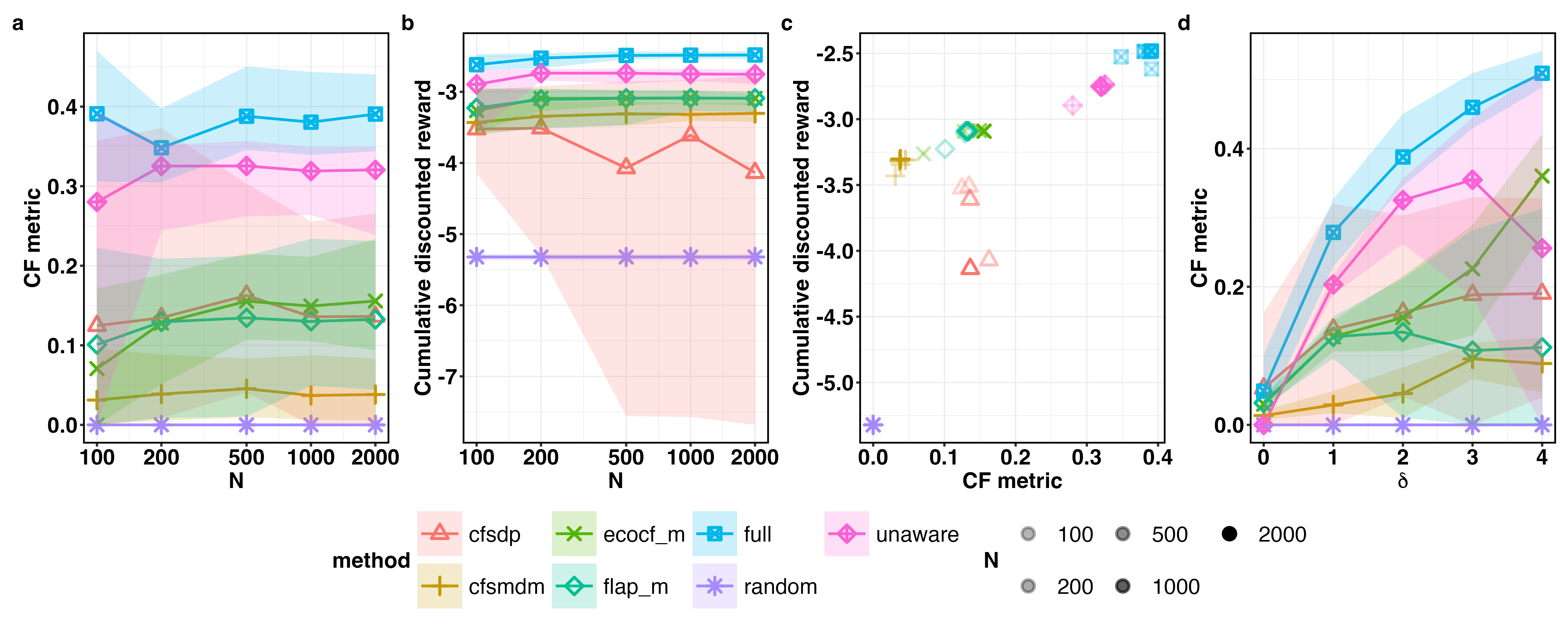}
    \caption{Experiment results using CQL for policy learning under nonadditive noise (CMDP 1): (a) CF metric vs. Sample size ($N$); (b) Policy value vs. Sample size ($N$); (c) Policy value vs. CF metric; (d) CF metric vs. Size of effect of $Z$ on $S_t$ and $R_t$ ($\delta$). Results are aggregated over $50$ replications with different random seeds. Point represents the empirical mean. Shaded area represents the band between $0.025$ and $0.975$ empirical quantiles.}
    \label{fig:results_m1_cql}
\end{figure}

\begin{figure}
    \centering
    \includegraphics[width=1\linewidth]{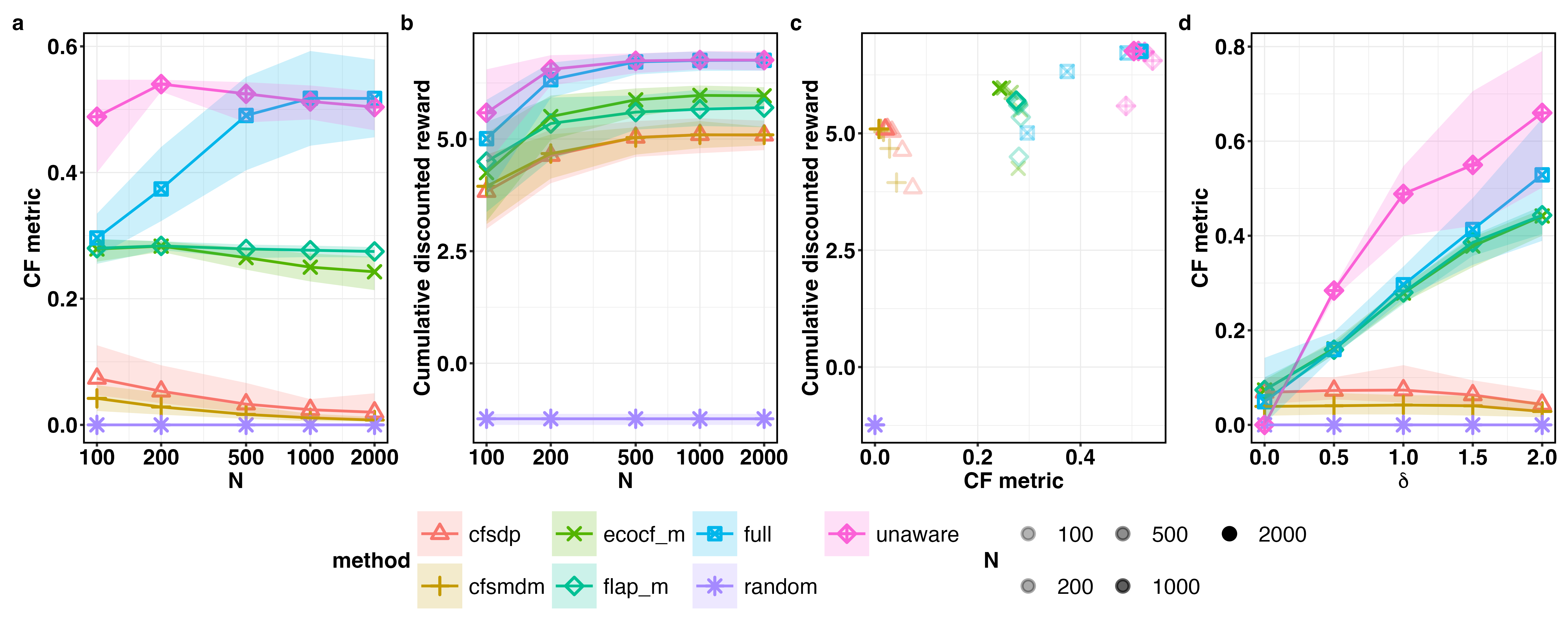}
    \caption{Experiment results using CQL for policy learning under additive noise (CMDP 2): (a) CF metric vs. Sample size ($N$); (b) Policy value vs. Sample size ($N$); (c) Policy value vs. CF metric; (d) CF metric vs. Size of effect of $Z$ on $S_t$ and $R_t$ ($\delta$). Results are aggregated over $50$ replications with different random seeds. Point represents the empirical mean. Shaded area represents the band between $0.025$ and $0.975$ empirical quantiles.}
    \label{fig:results_m2_cql}
\end{figure}

\subsection{Additional Numerical Experiments on Sensitivity of CFSMDM to the Choice of Quantile Estimator}
\label{sec:additional_quantile_estimator_experiment}

CFSMDM requires using a quantile estimator to estimate the conditional quantile function. Many different quantile estimators, which can often vary in terms of the algorithm used (e.g. linear quantile regression vs. quantile regression neural network) or the hyperparameter value, can be used. In this section, we conduct additional experiments to compare the performance of CFSMDM under different quantile estimators. In particular, we consider six different quantile estimators:
\begin{itemize}
    \item \textbf{qr\_linear\_500:} Linear quantile regression, fitted for a maximum of $500$ iterations.
    \item \textbf{qr\_linear\_1000:} Linear quantile regression, fitted for a maximum of $1000$ iterations.
    \item \textbf{qr\_linear\_2000:} Linear quantile regression, fitted for a maximum of $2000$ iterations.
    \item \textbf{qrnn\_50:} Quantile regression neural network, fitted for $50$ epochs.
    \item \textbf{qrnn\_100:} Quantile regression neural network, fitted for $100$ epochs.
    \item \textbf{qrnn\_200:} Quantile regression neural network, fitted for $200$ epochs.
\end{itemize}

The experiment is repeated for both Data-generating CMDP 1, which has bivariate state variable and nonadditive noise, and Data-generating CMDP 2, which has univariate state variable and additive noise. In particular, for Data-generating CMDP 1, we fix $N=1000$, $T=20$, $\delta=2.0$; for Data-generating CMDP 2, we fix $N=1000$, $T=20$, $\delta=1.0$. 

For both CMDPs, the training trajectories are generated using the ``Random'' policy as the behavioral policy. In particular, the sets of trajectories used for preprocessor training (if needed for the method of interest) and policy learning are generated independently. The policy value resulting from the methods of interest is approximated by the discounted cumulative reward from a simulated trajectory with $10000$ individuals and $20$ horizons, generated from the known CMDP using the policy learned with the methods of interest. Similarly, the CF metric is calculated from a simulated dataset containing counterfactual trajectories of $10000$ individuals with $20$ horizons, generated from the known CMDP using the policy learned with the methods of interest. 

\paragraph{Results.} Tables \ref{tab:qr_qrnn_comparison_m1} and \ref{tab:qr_qrnn_comparison_m2} show the experiment results for CMDP 1 and CMDP 2, respectively. Overall, all the quantile estimators tested are able to achieve a high level of counterfactual fairness, though linear quantile regression outperforms quantile regression neural network in both CMDPs. In both CMDPs, increasing the maximum allowed number of training iterations for linear quantile regression from $500$ to $2000$ appears to have negligible impact on fairness control performance, while increasing the number of training epochs for quantile regression neural network from $50$ to $200$ appears to slightly improve fairness control performance.

\begin{table}[ht]
\centering
\begin{tabular}{lcc}
\toprule
\textbf{Method} & \textbf{CF Metric} & \textbf{Policy Value} \\
\midrule
qr\_linear\_500
& 0.038(0.012)
& -3.309(0.063) \\

qr\_linear\_1000
& 0.038(0.013)
& -3.307(0.062) \\

qr\_linear\_2000
& 0.038(0.012)
& -3.307(0.064) \\

qrnn\_50
& 0.064(0.019)
& -3.308(0.101) \\

qrnn\_100
& 0.064(0.016)
& -3.324(0.107) \\

qrnn\_200
& 0.059(0.017)
& -3.315(0.109) \\
\bottomrule
\end{tabular}
\caption{CF metric and value of CFSMDM under different quantile estimators in Data-generating CMDP 1. Aggregated over $50$ replications with different random seeds. The number without parentheses is the empirical mean, and the number inside the parentheses is the standard deviation. All values are rounded to three decimal points.}
\label{tab:qr_qrnn_comparison_m1}
\end{table}

\begin{table}[ht]
\centering
\begin{tabular}{lcc}
\toprule
\textbf{Method} & \textbf{CF Metric} & \textbf{Policy Value} \\
\midrule
qr\_linear\_500
& 0.011(0.003)
& 5.094(0.154) \\

qr\_linear\_1000
& 0.011(0.003)
& 5.094(0.154) \\

qr\_linear\_2000
& 0.011(0.003)
& 5.094(0.154) \\

qrnn\_50
& 0.046(0.022)
& 5.086(0.362) \\

qrnn\_100
& 0.040(0.016)
& 5.120(0.297) \\

qrnn\_200
& 0.037(0.018)
& 5.116(0.269) \\
\bottomrule
\end{tabular}
\caption{CF metric and value of CFSMDM under different quantile estimators in Data-generating CMDP 2. Aggregated over $50$ replications with different random seeds. The number without parentheses is the empirical mean, and the number inside the parentheses is the standard deviation. All values are rounded to three decimal points.}
\label{tab:qr_qrnn_comparison_m2}
\end{table}

\subsection{Additional Numerical Experiments on Sensitivity of CFSMDM to the Number of Quantile Models}
\label{sec:additional_q_experiment}

CFSMDM requires estimating the conditional quantile function, which is often done by training quantile regression models for a grid of quantile levels in $(0,1)$. For example, in the previous experiments, the conditional quantile function is estimated by fitting quantile regression models for quantile levels $\tau \in \{0.01, 0.02, 0.03, \dots, 0.98, 0.99\}$. The number of quantile levels in this grid can be varied. Therefore, in this section, we conduct an additional numerical experiment to investigate the sensitivity of CFSMDM to the number of quantile levels in this grid. For notational convenience, we denote this number of quantile levels by $q$.

The experiment is repeated for both Data-generating CMDP 1, which has bivariate state variable and nonadditive noise, and Data-generating CMDP 2, which has univariate state variable and additive noise. Both CMDPs are defined in Section \ref{sec:details_experiments}. In particular, for Data-generating CMDP 1, we fix $N=500$, $T=20$, $\delta=2.0$, and vary $q \in \{9, 19, 49, 99, 199, 499\}$; for Data-generating CMDP 2, we fix $N=100$, $T=20$, $\delta=1.0$, and vary $q \in \{9, 19, 49, 99, 199, 499\}$. In our implementation, the $q$ quantile levels are evenly distributed on the interval $(0,1)$. For example, $q=49$ corresponds to $\tau \in \{0.02, 0.04, 0.06, \dots, 0.96, 0.98\}$, and $q=99$ corresponds to $\tau \in \{0.01, 0.02, 0.03, \dots, 0.98, 0.99\}$.

For both CMDPs, the training trajectories are generated using the ``Random'' policy as the behavioral policy. In particular, the sets of trajectories used for preprocessor training (if needed for the method of interest) and policy learning share the same sensitive attributes, while the remaining trajectory components are generated independently conditional on the shared sensitive attributes. The policy value resulting from the methods of interest is approximated by the discounted cumulative reward from a simulated trajectory with $10000$ individuals and $20$ horizons, generated from the known CMDP using the policy learned with the methods of interest. Similarly, the CF metric is calculated from a simulated dataset containing counterfactual trajectories of $10000$ individuals with $20$ horizons, generated from the known CMDP using the policy learned with the methods of interest. 

\paragraph{Results.} Tables \ref{tab:quantiles_metrics_m1} and \ref{tab:quantiles_metrics_m2} summarize the experiment results for Data-generating CMDP 1 and Data-generating CMDP 2, respectively. Overall, the performance of CFSMDM appears to be insensitive to the number of quantile levels fitted. Under Data-generating CMDP 1, the level of counterfactual unfairness is relatively low for all the $q$ levels tested. However, $q=9$ and $q=19$ resulted in higher unfairness than the other $q$ levels, possibly because a small $q$ degrades the quality of counterfactual state and reward estimation. $q=9$ also resulted in higher value than the other $q$ levels. Under Data-generating CMDP 2, the level of counterfactual unfairness and policy value are similar across all the $q$ levels tested. Also, the resulting policy is close to being counterfactually fair for all the $q$ levels tested.

\begin{table}
\centering
\begin{tabular}{lcc}
\hline
\textbf{$q$} & \textbf{CF Metric} & \textbf{Value} \\
\hline
$9$ & 0.0848 (0.0212) & -3.1881 (0.0911) \\
$19$ & 0.0528 (0.0152) & -3.2787 (0.0798) \\
$49$ & 0.0456 (0.0171) & -3.3039 (0.0840) \\
$99$ & 0.0444 (0.0173) & -3.3042 (0.0853) \\
$199$ & 0.0426 (0.0161) & -3.3035 (0.0791) \\
$499$ & 0.0430 (0.0159) & -3.3054 (0.0784) \\
\hline
\end{tabular}
\caption{CF metric and value of CFSMDM under varying levels of $q$ (number of quantile levels fitted) in Data-generating CMDP 1. Aggregated over $50$ replications with different random seeds. The number without parentheses is the empirical mean, and the number inside the parentheses is the standard deviation. All values are rounded to four decimal points.}
\label{tab:quantiles_metrics_m1}
\end{table}

\begin{table}[ht]
\centering
\begin{tabular}{lcc}
\hline
\textbf{$q$} & \textbf{CF Metric} & \textbf{Value} \\
\hline
$9$ & 0.0438 (0.0171) & 4.7583 (0.2918) \\
$19$ & 0.0394 (0.0173) & 4.7356 (0.5918) \\
$49$ & 0.0372 (0.0169) & 4.7817 (0.4774) \\
$99$ & 0.0368 (0.0204) & 4.7609 (0.8844) \\
$199$ & 0.0362 (0.0191) & 4.8301 (0.6153) \\
$499$ & 0.0347 (0.0162) & 4.7962 (0.6757) \\
\hline
\end{tabular}
\caption{CF metric and value of CFSMDM under varying levels of $q$ (number of quantile levels fitted) in Data-generating CMDP 2. Aggregated over $50$ replications with different random seeds. The number without parentheses is the empirical mean, and the number inside the parentheses is the standard deviation. All values are rounded to four decimal points.}
\label{tab:quantiles_metrics_m2}
\end{table}

\subsection{Additional Numerical Experiments on the Computing Time of CFSMDM}
\label{sec:additional_computing_time_experiment}

In this section, we conduct additional numerical experiments to compare the computing time of CFSMDM and baseline methods. We record the training time of CFSMDM and baseline methods under CMDP 1: the time a method takes to train the preprocessor and the FQI agent using a simulated trajectory with $N=500$, $T=20$, and $\delta=2.0$. Both wall-clock time and CPU time are recorded and reported. Please note that the actual computing time can vary if the experiment is run with a different code implementation, with a different setup (e.g. with another CMDP), or in a different computing environment.  

\paragraph{Results.} Table \ref{tab:computing_time_m1} shows the experiment results. In our setting, CFSMDM takes more training time than all other baselines except CFSDP, possibly reflecting the fact that CFSMDM is more complex than the faster methods. CFSMDM is slightly less time-consuming than CFSDP on the task tested.

\begin{table}[ht]
\centering
\begin{tabular}{lcc}
\toprule
\textbf{Method} & \textbf{CPU Time} & \textbf{Wall-clock Time} \\
\midrule
Random
& 0.000(0.000)
& 0.000(0.000) \\

Full
& 259.120(12.530)
& 261.575(12.709) \\

Unaware
& 253.095(11.732)
& 255.455(11.884) \\

CFSMDM
& 316.586(13.203)
& 319.503(13.405) \\

CFSDP
& 333.482(13.915)
& 336.492(14.110) \\

FLAP\_M
& 248.109(12.404)
& 250.430(12.606) \\

ECOCF\_M
& 255.242(13.998)
& 257.634(14.237) \\
\bottomrule
\end{tabular}
\caption{CPU time and wall-clock time needed for the training process in CMDP 1. Aggregated over $50$ replications with different random seeds. The number without parentheses is the empirical mean, and the number inside the parentheses is the standard deviation. All values are rounded to three decimal points.}
\label{tab:computing_time_m1}
\end{table}

\section{Application to the Intern Health Study Data}
\label{sec:rda_ihs}

In this section, we use CFSMDM and other baseline methods to analyze a dataset from the Intern Health Study \citep{necamp2020ihs}, which is a micro-randomized trial aimed at improving US medical interns' mental well-being through sequential mobile health interventions. The dataset consists of data from $1565$ US medical interns over a period of $26$ weeks. Each week, each participant either receives a push notification ($A_t=1$) or receives no push notifications ($A_t=0$). The state variable contains the participant's step count, sleep count, and mood score from the current week. The reward is the participant's mood score in the following week. We consider the biological sex as the sensitive attribute.

The methods being run in this section are ``Full'', ``Unaware'', ``Random'', FLAP\_M, ECOCF\_M, CFSDP, and CFSMDM\_QRNN. In particular, CFSDP estimates the transition kernel using a neural network, and CFSMDM\_QRNN estimates the conditional quantiles using a neural network (we call it ``CFSMDM\_QRNN'' instead of ``CFSMDM'' to distinguish it from the linear quantile regression-based implementation used in the PowerED data analysis). We are interested in the counterfactual unfairness level and policy value achieved by CFSMDM and the baseline methods. Details on the implementation of the real data analysis can be found in Section~\ref{sec:details_rda_ihs} of the Appendix.

Table \ref{tab:results_rda_ihs} summarizes the results of the analysis. Similar to the results for the PowerED dataset, ``Full'' has the worse level of counterfactual fairness, and ``Random'' is perfectly counterfactually fair but has made low-quality decisions as indicated by the low policy value. However, all of the other baselines achieve relatively high levels of counterfactual fairness, and their CF metrics are close to each other. In particular, despite achieving a high level of counterfactual fairness, CFSMDM does not notably outperform baselines other than ``Full''. One possible explanation for this is that the underlying environment does not contain very significant unfairness, potentially because the correlation between the biological sex and the medical interns' mood score, sleep count, and step count are weak. 

\begin{table}[t]
\centering
\small
\setlength{\tabcolsep}{12pt}
\renewcommand{\arraystretch}{1.2}
\begin{tabular}{lcc}
\hline
\textbf{Method} & \textbf{Unfairness} & \textbf{Policy value} \\
\hline
Full      & 0.14 (0.12) & 70.02 (0.44) \\
Random    & 0.00 (0.00) & 69.74 (0.35) \\
Unaware   & 0.07 (0.03) & 69.97 (0.49) \\
CFSDP     & 0.05 (0.02) & 69.96 (0.44) \\
ECOCF\_M  & 0.06 (0.02) & 69.90 (0.46) \\
FLAP\_M   & 0.06 (0.02) & 70.07 (0.50) \\
CFSMDM\_QRNN & 0.06 (0.02) & 69.96 (0.52) \\
\hline
\end{tabular}
\caption{CF metric and value of different methods when applied to the Intern Health Study dataset. The analysis is repeated for $50$ times using different random seeds. The number without parentheses is the empirical mean, and the number inside the parentheses is the standard deviation. All values are rounded to two decimal points.} 
\label{tab:results_rda_ihs}
\end{table}

\section{Semi-synthetic Experiments}
\label{sec:semisyn_experiments}

In this section, we conduct semi-synthetic experiments on the PowerED data and the Intern Health Study data. In the semi-synthetic experiments, we fit a generative model to the PowerED dataset to mimic its underlying transition dynamics and generate data from this known generative model. In this case, we have access to the ground truth under the generative model, and we can evaluate the policy value and level of counterfactual fairness following the strategy employed by the numerical experiments. 

For the PowerED dataset, we consider sex, education, age, and ethnicity, all binarized, as potential sensitive attributes, and we use $N=1000, T=11$ as the size of the training trajectory. For the Intern Health Study dataset, we consider the biological sex as the sensitive attribute, and we use $N=1000, T=25$ as the size of the training trajectory. For both datasets, the state variable and reward are defined in the same way as in the corresponding real data analyses. We use linear quantile regression-based CFSMDM for the PowerED semi-synthetic experiments, and use quantile regression neural network-based CFSMDM\_QRNN for the Intern Health Study semi-synthetic experiments. 

For both datasets, the training trajectories are generated using the ``Random'' policy as the behavioral policy. In particular, the sets of trajectories used for preprocessor training (if needed for the method of interest) and policy learning are generated independently. The policy value resulting from the methods of interest is approximated by the discounted cumulative reward from a simulated trajectory with $10000$ individuals and $11$ (for the PowerED dataset) or $25$ (for the Intern Health Study dataset) horizons, simulated from the generative model using the policy learned with the methods of interest. Similarly, the CF metric is calculated from a simulated dataset containing counterfactual trajectories of $10000$ individuals with $11$ (for the PowerED dataset) or $25$ (for the Intern Health Study dataset) horizons, simulated from the generative model using the policy learned with the methods of interest. More details on the implementation of the semi-synthetic experiments can be found in Section \ref{sec:details_semisyn}

Table \ref{tab:results_semisyn_powered} summarizes the results of the semi-synthetic experiment for the PowerED dataset, and Table \ref{tab:results_semisyn_ihs} summarizes the results of the semi-synthetic experiment for the Intern Health Study dataset. In the PowerED semi-synthetic experiment, for all of the sensitive attributes except ``Ethnicity'', CFSMDM achieves the highest level of counterfactual fairness among non-random baselines. On the other hand, for the Intern Health Study semi-synthetic experiment and the ``Ethnicity'' sensitive attribute in the PowerED semi-synthetic experiment, CFSMDM underperformed a few other baselines in terms of counterfactual fairness. One possible explanation is that the unfairness in the generative environments due to the corresponding sensitive attributes is not significant enough, so the fairness gain from CFSMDM is outweighed by the unfairness introduced by finite-sample counterfactual estimation errors. This demonstrates that in cases where the level of unfairness in the environment is low enough to be corrected by more naive methods, CFSMDM needs not always be the best choice due to the possibility of counterfactual estimation errors.

\begin{table}[ht]
\centering
\begin{tabular}{llcccc}
\toprule
\textbf{Metric} & \textbf{Method}
& \textbf{Education} & \textbf{Ethnicity}
& \textbf{Sex} & \textbf{Age} \\
\midrule
\multirow{7}{*}{Unfairness}
& Full     & 0.280 (0.107) & 0.288 (0.116) & 0.505 (0.174) & 0.575 (0.153) \\
& Random   & 0.000 (0.000) & 0.000 (0.000) & 0.000 (0.000) & 0.000 (0.000) \\
& Unaware  & 0.168 (0.026) & 0.083 (0.043) & 0.118 (0.021) & 0.155 (0.038) \\
& CFSDP    & 0.133 (0.039) & 0.125 (0.056) & 0.117 (0.030) & 0.140 (0.036) \\
& ECOCF\_M  & 0.151 (0.035) & 0.064 (0.030) & 0.112 (0.017) & 0.120 (0.016) \\
& FLAP\_M   & 0.151 (0.027) & 0.053 (0.020) & 0.119 (0.018) & 0.129 (0.024) \\
& CFSMDM   & 0.109 (0.028) & 0.091 (0.036) & 0.099 (0.017) & 0.113 (0.022) \\
\midrule
\multirow{7}{*}{Value}
& Full     & 39.038 (0.039) & 39.763 (0.049) & 40.224 (0.027) & 39.297 (0.043) \\
& Random   & 38.738 (0.018) & 39.237 (0.017) & 39.654 (0.017) & 38.922 (0.018) \\
& Unaware  & 39.038 (0.032) & 39.658 (0.023) & 40.197 (0.033) & 39.257 (0.040) \\
& CFSDP    & 38.964 (0.107) & 39.613 (0.118) & 40.153 (0.111) & 39.192 (0.058) \\
& ECOCF\_M  & 39.039 (0.032) & 39.672 (0.028) & 40.205 (0.034) & 39.251 (0.039) \\
& FLAP\_M   & 39.039 (0.029) & 39.678 (0.022) & 40.202 (0.038) & 39.244 (0.041) \\
& CFSMDM   & 38.995 (0.042) & 39.606 (0.080) & 40.183 (0.035) & 39.213 (0.050) \\
\bottomrule
\end{tabular}
\caption{CF metric and policy value from semi-synthetic experiments for the PowerED dataset under different sensitive attributes. Aggregated over $50$ replications with different random seeds. The number without parentheses is the empirical mean, and the number inside the parentheses is the standard deviation. All values are rounded to three decimal places.}
\label{tab:results_semisyn_powered}
\end{table}

\begin{table}[t]
\centering
\small
\begingroup
\setlength{\tabcolsep}{12pt}
\renewcommand{\arraystretch}{1.2}
\begin{tabular}{lcc}
\hline
\textbf{Method} & \textbf{Unfairness} & \textbf{Policy value} \\
\hline
Full      & 0.310 (0.186) & 68.836 (0.100) \\
Random    & 0.000 (0.000) & 68.487 (0.055) \\
Unaware   & 0.096 (0.065) & 68.844 (0.081) \\
CFSDP     & 0.015 (0.012) & 68.818 (0.100) \\
ECOCF\_M  & 0.019 (0.013) & 68.823 (0.102) \\
FLAP\_M   & 0.017 (0.011) & 68.831 (0.097) \\
CFSMDM\_QRNN & 0.039 (0.028) & 68.811 (0.105) \\
\hline
\end{tabular}
\endgroup
\caption{CF metric and policy value from semi-synthetic experiments for the Intern Health Study dataset when the sensitive attribute is biological sex. Aggregated over $50$ replications with different random seeds. The number without parentheses is the empirical mean, and the number inside the parentheses is the standard deviation. All values are rounded to three decimal points.}
\label{tab:results_semisyn_ihs}
\end{table}

\section{Additional Implementation Details of Numerical Experiments}
\label{sec:details_experiments}

\paragraph{Data-generating CMDP 1:} We define below Data-generating CMDP 1, which has a bivariate state variable and nonadditive noise:
\begin{align*}
S_0^{(1)} =& (-0.45 + 1.5\delta Z + U_{0}^{S,(1)} + 0.5 U_{0}^{S,(3)})^{\frac{1}{3}}, \\
S_0^{(2)} =& (-0.75 + 2.5\delta Z + U_{0}^{S,(1)} + U_{0}^{S,(2)} + 0.5 U_{0}^{S,(3)})^{\frac{1}{3}}, \\
S_t^{(1)} =& (-0.45 + 0.3S_{t-1}^{(1)}(A_{t-1} - 0.5) + 0.3\delta S_{t-1}^{(1)}(Z - 0.5) + 0.45\delta (Z - 0.5)(A_{t-1} - 0.5) \\
&+ 0.5U_{t}^{S,(1)} + 0.25 U_{t}^{S,(3)})^{\frac{1}{3}} \text{ (for $t \geq 1$)}, \\
S_t^{(2)} =& (-0.9 + 0.4S_{t-1}^{(1)}(A_{t-1} - 0.5) + 0.2S_{t-1}^{(2)}(A_{t-1} - 0.5) + 0.6\delta S_{t-1}^{(1)}(Z - 0.5) \\
&+ 0.2\delta S_{t-1}^{(2)}(Z - 0.5) + 0.9\delta (Z - 0.5)(A_{t-1} - 0.5) + 0.5U_{t}^{S,(1)} \\
&+ 0.9U_{t}^{S,(2)} + 0.5 U_{t}^{S,(3)})^{\frac{1}{3}} \text{ (for $t \geq 1$)}, \\
R_t = (&-0.3 + 0.2\delta S_t^{(1)}Z + 0.5S_t^{(1)}A_t + 0.2\delta S_t^{(2)}Z + 0.5S_t^{(2)}A_t - 1.0\delta ZA_t + 1.0U_{t}^R)^{\frac{1}{3}},
\end{align*}
where $U_t^{S,(1)}, U_t^{S,(2)}, U_t^{S,(3)}, U_t^{R} \overset{\text{i.i.d.}}{\sim} N(0,1)$ for $t \geq 0$. In this environment, we can view $U_0^{S, (1)} + 0.5U_0^{S, (3)}$ and $U_0^{S, (1)} + U_0^{S, (2)} + 0.5U_0^{S, (3)}$ as the scalar exogenous variables of $S_0^{(1)}$ and $S_0^{(2)}$, respectively. For $t \geq 1$, we can also view $0.5U_t^{S, (1)} + 0.25U_t^{S, (3)}$ and $0.5U_t^{S, (1)} + 0.9U_t^{S, (2)} + 0.5U_t^{S, (3)}$ as the scalar exogenous variables of $S_t^{(1)}$ and $S_t^{(2)}$, respectively. Note that $\delta$ is a constant that controls the strength of impact of the sensitive attribute on the state and reward variables; in particular, there will be no fairness issues with the CMDP if $\delta=0$. In this setting, the exogenous variables are not additive to the state or the reward, but Assumptions \ref{ass:state_monotonicity} and \ref{ass:reward_monotonicity} are both satisfied. Therefore, in this setting, CFSMDM should result in approximately counterfactually fair policies, while there is no guarantee for the performance of CFSDP.

\begin{remark}
The conditional quantile functions in the setting given above are not linear. However, for simplicity, we estimate the conditional quantile functions in CFSMDM using linear quantile regression with a flexible model specification, where the model specification can be found in the paragraphs below. As shown in Figure~\ref{fig:results_simulation}, the performance of CFSMDM remains robust despite some degree of model misspecification.
\end{remark}

\paragraph{Data-generating CMDP 2:} We define below Data-generating CMDP 2, which has a univariate state variable and additive noise:
\begin{align*}
S_0 =& -0.3 + 1.0 \delta Z + U_0^S, \\
S_t =& -0.3 + 1.0 \delta (Z - 0.5) + 0.5S_{t-1} + 0.4(A_{t-1} - 0.5) + 0.3 S_{t-1}(A_{t-1} - 0.5) \\
&+ 0.3 \delta S_{t-1} (Z - 0.5) + 0.4 \delta (Z - 0.5)(A_{t-1} - 0.5) + U_t^S \text{ (for $t \geq 1$)}, \\
R_t =& -0.3 + 0.3S_t + 0.5 \delta Z + 0.5A_t + 0.2 \delta S_tZ + 0.7S_tA_t - 1.0 \delta ZA_t + U_t^R,
\end{align*}
where $U_t^S, U_t^R \overset{\text{i.i.d.}}{\sim} N(0,1)$ for $t \geq 0$. Note that $\delta$ is a constant that controls the strength of impact of the sensitive attribute on the state and reward variables; in particular, there will be no fairness issues with the CMDP if $\delta=0$. In this setting, since the exogenous variables of the state are additive to the state, both CFSDP and CFSMDM should result in approximately counterfactually fair policies.

\paragraph{Data-generating CMDP 3:}
We define below Data-generating CMDP 3, which has a 20-dimensional
state variable, a two-dimensional sensitive attribute
$Z=(Z^{(1)},Z^{(2)})^\top$, and nonadditive noise. Let the coefficient matrix for the initial state be
\[
\Gamma_0=(\gamma_{0,ij})=
\begin{pmatrix}
-0.2&0.4&0.5&0.3&0.2\\
-0.2&0.4&0.2&0.5&0.9\\
-0.5&0.4&0.2&0.1&0.1\\
-0.1&0.2&0.1&0.6&0.4\\
-0.1&0.5&0.5&0.4&0.1\\
-0.7&0.3&0.4&0.1&0.3\\
-0.5&0.3&0.2&0.5&0.2\\
-0.1&0.3&0.2&0.1&0.4\\
-0.1&0.1&0.1&0.4&0.4\\
-0.3&0.2&0.3&0.2&1.0\\
-0.4&0.4&0.3&0.4&0.3\\
-0.2&0.2&0.3&0.4&0.5\\
-0.2&0.5&0.4&0.3&0.4\\
-0.3&0.5&0.1&0.3&0.3\\
-0.4&0.3&0.2&0.1&0.4\\
-0.5&0.6&0.2&0.3&0.4\\
-0.1&0.2&0.3&0.2&0.3\\
-0.8&0.3&0.1&0.2&1.0\\
-0.1&0.2&0.3&0.5&0.4\\
-0.9&0.1&0.5&0.1&0.5
\end{pmatrix},
\]
and the coefficient matrix for the transition equation be
\begin{center}
\resizebox{0.9\textwidth}{!}{$\displaystyle
\Gamma=(\gamma_{ij})=
\begin{pmatrix}
-0.7&0.5&0.6&0.2&0.3&0.66&0.56&0.4&0.5&0.2&0.1&0.2&0.5&0.6\\
-0.3&0.2&0.4&0.5&0.3&0.46&0.46&0.2&0.3&0.1&0.3&0.1&0.4&0.2\\
-0.6&0.4&0.6&0.6&0.3&0.26&0.66&0.1&0.1&0.4&0.3&0.7&0.2&0.5\\
-0.4&0.6&0.6&0.7&0.5&0.26&0.36&0.1&0.1&0.2&0.2&0.2&0.1&0.4\\
-0.5&0.2&0.4&0.5&0.2&0.26&0.66&0.3&0.6&0.4&0.4&0.1&0.5&0.1\\
-0.7&0.2&0.5&0.6&0.5&0.56&0.36&0.1&0.2&0.5&0.1&0.3&0.1&0.1\\
-0.8&0.6&0.5&0.7&0.3&0.36&0.46&0.6&0.3&0.5&0.1&0.2&0.6&0.4\\
-0.5&0.6&0.6&0.4&0.3&0.36&0.56&0.1&0.4&0.1&0.3&0.3&0.1&0.8\\
-0.9&0.6&0.4&0.5&0.2&0.66&0.46&0.4&0.2&0.3&0.3&0.2&0.2&0.3\\
-0.3&0.5&0.3&0.4&0.2&0.46&0.26&0.7&0.5&0.4&0.1&0.2&0.4&0.1\\
-0.4&0.2&0.5&0.6&0.6&0.46&0.16&0.1&0.1&0.1&0.2&0.2&0.3&0.5\\
-0.6&0.3&0.6&0.7&0.6&0.56&0.16&0.5&0.2&0.3&0.4&0.2&0.1&0.2\\
-0.5&0.4&0.6&0.6&0.3&0.26&0.36&0.2&0.1&0.5&0.1&0.1&0.5&0.3\\
-0.8&0.2&0.3&0.7&0.4&0.36&0.26&0.3&0.4&0.4&0.5&0.6&0.2&0.7\\
-0.7&0.4&0.6&0.4&0.5&0.26&0.46&0.1&0.5&0.1&0.4&0.2&0.6&0.2\\
-0.5&0.2&0.2&0.6&0.5&0.36&0.36&-0.4&0.1&0.2&0.1&0.3&0.6&0.1\\
-0.3&0.4&0.2&0.6&0.2&0.36&0.26&0.4&0.5&0.4&0.2&0.2&0.4&0.5\\
-0.8&0.4&0.5&0.3&0.2&0.26&0.26&0.4&0.3&0.5&0.1&0.3&0.1&0.3\\
-0.7&0.2&0.3&0.5&0.2&0.36&0.26&0.1&0.2&0.1&0.3&0.1&0.4&0.6\\
-0.5&0.4&0.3&0.7&0.4&0.26&0.36&0.4&0.4&0.5&0.2&0.2&0.1&0.4
\end{pmatrix}.
$}
\end{center}

Let
$i^+=i+1$ for $i=1,\ldots,19$ and $20^+=1$, and define
\[
\overline{S}_t=\frac{1}{20}\sum_{j=1}^{20}S_t^{(j)}.
\]
For $i=1,\ldots,20$, we define
\begin{align*}
S_0^{(i)}
=
\Big(
&\gamma_{0,i1}
+\delta\gamma_{0,i2}Z^{(1)}
+\delta\gamma_{0,i3}Z^{(2)}
+\gamma_{0,i4}U_0^{S,(i)}
+\gamma_{0,i5}U_0^{S,(c)}
\Big)^{\frac{1}{3}},
\end{align*}
and, for $t\geq1$,
\begin{align*}
S_t^{(i)}
=
\Big(
&\gamma_{i1}
+\gamma_{i2}S_{t-1}^{(i)}(A_{t-1}-0.5)
+\delta\gamma_{i3}S_{t-1}^{(i)}(Z^{(2)}-0.5)\\
&+\gamma_{i4}S_{t-1}^{(i^+)}(A_{t-1}-0.5)
+\delta\gamma_{i5}S_{t-1}^{(i^+)}(Z^{(1)}-0.5)\\
&+\delta\gamma_{i6}(Z^{(1)}-0.5)(A_{t-1}-0.5)
+\delta\gamma_{i7}(Z^{(2)}-0.5)(A_{t-1}-0.5)\\
&+\delta\gamma_{i8}(Z^{(1)}-0.5)
+\delta\gamma_{i9}(Z^{(2)}-0.5)
+\gamma_{i,10}(A_{t-1}-0.5)\\
&+\gamma_{i,11}S_{t-1}^{(i)}
+\gamma_{i,12}S_{t-1}^{(i^+)}
+\gamma_{i,13}U_t^{S,(i)}
+\gamma_{i,14}U_t^{S,(c)}
\Big)^{\frac{1}{3}}.
\end{align*}
The reward is generated according to
\begin{align*}
R_t
=
\Big(
&-0.6
+0.5\delta Z^{(1)}
+0.5\delta Z^{(2)}
+0.3\overline{S}_t
+0.4A_t\\
&+0.4\delta\overline{S}_tZ^{(1)}
+0.4\delta\overline{S}_tZ^{(2)}
+0.2\overline{S}_tA_t\\
&-0.7\delta Z^{(1)}A_t
-0.7\delta Z^{(2)}A_t
+U_t^R
\Big)^{\frac{1}{3}}.
\end{align*}

Here, $U_t^{S,(1)}, \dots,  U_t^{S,(20)}, U_t^{S,(c)}, U_t^{R} \overset{\text{i.i.d.}}{\sim} N(0,1)$ for $t \geq 0$. In this
environment, we can view
$\gamma_{0,i4}U_0^{S,(i)}+\gamma_{0,i5}U_0^{S,(c)}$
as the scalar exogenous variable of $S_0^{(i)}$. For $t\geq1$, we can
similarly view
$\gamma_{i,13}U_t^{S,(i)}+\gamma_{i,14}U_t^{S,(c)}$
as the scalar exogenous variable of $S_t^{(i)}$. Note that $\delta$ is
a constant that controls the strength of impact of the sensitive
attributes on the state and reward variables; in particular, there
will be no fairness issues with the CMDP if $\delta=0$. In this setting, the additive noise assumption is not satisfies, but the strict monotonicity assumption is satisfied.

\paragraph{Data-generating CMDP 4:}
We define below Data-generating CMDP 4:
\begin{align*}
S_0^{(1)}
&=
\left(
-0.45+1.5\delta Z
+U_0^{S,(1)}
+0.5U_0^{S,(3)}
\right)^{\frac13}
+v U_0^{S,(3)},\\
S_0^{(2)}
&=
\left(
-0.75+2.5\delta Z
+U_0^{S,(1)}
+U_0^{S,(2)}
+0.5U_0^{S,(3)}
\right)^{\frac13}
+v U_0^{S,(3)},\\
S_t^{(1)}
&=
\left(
-0.45
+0.3S_{t-1}^{(1)}(A_{t-1}-0.5)
+0.3\delta S_{t-1}^{(1)}(Z-0.5)\right.\\
&\qquad\left.
+0.45\delta(Z-0.5)(A_{t-1}-0.5)
+0.5U_t^{S,(1)}
+0.25U_t^{S,(3)}
\right)^{\frac13}
+v U_t^{S,(3)}
\qquad (t\geq1),\\
S_t^{(2)}
&=
\left(
-0.9
+0.4S_{t-1}^{(1)}(A_{t-1}-0.5)
+0.2S_{t-1}^{(2)}(A_{t-1}-0.5)\right.\\
&\qquad
+0.6\delta S_{t-1}^{(1)}(Z-0.5)
+0.2\delta S_{t-1}^{(2)}(Z-0.5)\\
&\qquad\left.
+0.9\delta(Z-0.5)(A_{t-1}-0.5)
+0.5U_t^{S,(1)}
+0.9U_t^{S,(2)}
+0.5U_t^{S,(3)}
\right)^{\frac13}\\
&\qquad
+v U_t^{S,(3)}
\qquad (t\geq1),\\
R_t
&=
\left(
-0.3
+0.2\delta S_t^{(1)}Z
+0.5S_t^{(1)}A_t
+0.2\delta S_t^{(2)}Z
+0.5S_t^{(2)}A_t
-\delta ZA_t
+U_t^R
\right)^{\frac13}
+v U_t^R,
\end{align*}
where $U_t^{S,(1)}, U_t^{S,(2)}, U_t^{S,(3)}, U_t^{R} \overset{\text{i.i.d.}}{\sim} N(0,1)$ for $t \geq 0$. Here, $v$ is a constant controlling the magnitude of the deviation
from the single scalar noise structure. When $v=0$, Data-generating CMDP 4
has single scalar noise for each component of its state variable and the reward.

\paragraph{Data-generating CMDP 5:}
We define below Data-generating CMDP 5:
\begin{align*}
S_0
&=
-0.3+\delta Z+U_0^{S,(1)},\\
S_t
&=
-0.3
+\delta(Z-0.5)
+0.5v S_{t-1}U_t^{S,(2)}
+0.4(A_{t-1}-0.5)\\
&\qquad
+0.3S_{t-1}(A_{t-1}-0.5)
+0.3\delta S_{t-1}(Z-0.5)
+0.4\delta(Z-0.5)(A_{t-1}-0.5)
+U_t^{S,(1)} \\
&\qquad (t\geq1),\\
R_t
&=
-0.3
+0.3v S_tU_t^R
+0.5\delta Z
+0.5A_t
+0.2\delta S_tZ
+0.7S_tA_t
-\delta ZA_t
+U_t^R,
\end{align*}
where $U_t^S, U_t^R \overset{\text{i.i.d.}}{\sim} N(0,1)$ for $t \geq 0$. Here, $v$ is a constant controlling the magnitude of the deviation
from the single scalar noise structure. When $v=0$, Data-generating CMDP 5
has single scalar noise for each component of its state variable and the reward.

\paragraph{Data-generating CMDP 6:}
We define below Data-generating CMDP 6:
\begin{align*}
S_0
&=
-0.3+\delta Z+U_0^S,\\
S_t
&=
-0.3
+\delta(Z-0.5)
-0.5u|S_{t-1}|U_t^S
+0.4(A_{t-1}-0.5)\\
&\qquad
+0.3S_{t-1}(A_{t-1}-0.5)
+0.3\delta S_{t-1}(Z-0.5)
+0.4\delta(Z-0.5)(A_{t-1}-0.5)
+U_t^S \\
&\qquad (t\geq1),\\
R_t
&=
-0.3
-0.3u|S_t|U_t^R
+0.5\delta Z
+0.5A_t
+0.2\delta S_tZ
+0.7S_tA_t
-\delta ZA_t
+U_t^R.
\end{align*}
where $U_t^S, U_t^R \overset{\text{i.i.d.}}{\sim} N(0,1)$ for $t \geq 0$. Here, $u$ is a constant controlling the magnitude of the deviation
from the strict monotonicity assumption. When $u=0$, Data-generating CMDP 6
satisfies the strict monotonicity assumption.

\paragraph{Data-generating CMDP 7:}
We define below Data-generating CMDP 7:
\begin{align*}
S_0^{(1)}
&=
\left(
-0.45+1.5\delta Z
+U_0^{S,(1)}
+0.5U_0^{S,(3)}
\right)^{\frac13}
-u\left(U_0^{S,(1)}+0.5U_0^{S,(3)}\right),\\
S_0^{(2)}
&=
\left(
-0.75+2.5\delta Z
+U_0^{S,(1)}
+U_0^{S,(2)}
+0.5U_0^{S,(3)}
\right)^{\frac13}\\
&\qquad
-u\left(U_0^{S,(1)}+U_0^{S,(2)}
+0.5U_0^{S,(3)}\right),\\
S_t^{(1)}
&=
\left(
-0.45
+0.3S_{t-1}^{(1)}(A_{t-1}-0.5)
+0.3\delta S_{t-1}^{(1)}(Z-0.5)\right.\\
&\qquad\left.
+0.45\delta(Z-0.5)(A_{t-1}-0.5)
+0.5U_t^{S,(1)}
+0.25U_t^{S,(3)}
\right)^{\frac13}\\
&\qquad
-u\left(0.5U_t^{S,(1)}+0.25U_t^{S,(3)}\right)
\qquad (t\geq1),\\
S_t^{(2)}
&=
\left(
-0.9
+0.4S_{t-1}^{(1)}(A_{t-1}-0.5)
+0.2S_{t-1}^{(2)}(A_{t-1}-0.5)\right.\\
&\qquad
+0.6\delta S_{t-1}^{(1)}(Z-0.5)
+0.2\delta S_{t-1}^{(2)}(Z-0.5)\\
&\qquad\left.
+0.9\delta(Z-0.5)(A_{t-1}-0.5)
+0.5U_t^{S,(1)}
+0.9U_t^{S,(2)}
+0.5U_t^{S,(3)}
\right)^{\frac13}\\
&\qquad
-u\left(
0.5U_t^{S,(1)}
+0.9U_t^{S,(2)}
+0.5U_t^{S,(3)}
\right)
\qquad (t\geq1),\\
R_t
&=
\left(
-0.3
+0.2\delta S_t^{(1)}Z
+0.5S_t^{(1)}A_t
+0.2\delta S_t^{(2)}Z
+0.5S_t^{(2)}A_t
-\delta ZA_t
+U_t^R
\right)^{\frac13}
-u U_t^R,
\end{align*}
where $U_t^{S,(1)}, U_t^{S,(2)}, U_t^{S,(3)}, U_t^{R} \overset{\text{i.i.d.}}{\sim} N(0,1)$ for $t \geq 0$. Here, $u$ is a constant controlling the magnitude of the deviation
from the strict monotonicity assumption. When $u=0$, Data-generating CMDP 7
satisfies the strict monotonicity assumption.

\paragraph{Implementation of FLAP\_M:} The Fair Learning through data preprocessing (FLAP) algorithm is introduced by \cite{chen2024flap}. It is a data preprocessing method that aims to achieve counterfactual fairness in cross-sectional prediction tasks. Our FLAP\_M baseline adapts the FLAP algorithm (with marginal distribution mapping) to the sequential decision-making setting, which is detailed in Algorithm \ref{alg:flap_m}. Specifically, the quantile levels and quantiles for each time step $t$ and dimension $i$, which are required by FLAP\_M, are estimated using the empirical cumulative distribution function (CDF) of the conditional distribution (denoted by $\hat{\mathbb{P}}_{t,n}^{(i)}(\cdot|Z)$) and its inverse, respectively. 

\begin{algorithm}
\caption{FLAP\_M: FLAP with single-stage marginal distribution mapping, adapted to the sequential decision-making setting}\label{alg:flap_m}

\textbf{Input:} Original data $\mathcal{D} = \{s_{j,0}, s_{j,t}, z_j, a_{j,t-1}, r_{j,t-1}: j = 1, \dots, N, t = 1, \dots, T\}$, where $s_{j,t} = (s_{j,t}^{(1)}, \dots, s_{j,t}^{(M)})$.

\begin{algorithmic}[1]

\For{$t = 0, \dots, T$} 

    \For{$i = 1, \dots, M$} 
    
        \State For each $z \in \mathcal{Z}$, fit empirical CDF $\hat{\mathbb{P}}_{t,n}^{(i)}(S_t^{(i)}|Z)$ using $\{s_{j,t}^{(i)}: \text{$j$ such that $z_j = z$}\}$.
    
    \EndFor

\EndFor

\For{$j = 1, \dots, N$} 
    
    \For{$t = 0, \dots, T$} 

        \For{$i = 1, \dots, M$}

            \State $s_{j,t}^{pro, (i)} = \sum_{z \in \mathcal{Z}} \hat{\mathbb{P}}_n(Z = z) [\hat{\mathbb{P}}_{t,n}^{(i)}]^{-1}(\hat{\mathbb{P}}_{t,n}^{(i)}(s_{j,t}^{(i)}|Z=z_j)|Z=z)$
            
        \EndFor

        \State $s_{j,t}^{pro} = (s_{j,t}^{pro, (1)}, \dots, s_{j,t}^{pro, (M)})$
        
    \EndFor
    
\EndFor

\State Train a policy $\hat{\pi}$ using FQI on $\mathcal{D}' = \{s_{j,0}^{pro}, s_{j,t}^{pro}, a_{j,t-1}, r_{j,t-1}: j = 1, \dots, N, t = 1, \dots, T\}$

\end{algorithmic}

\textbf{Output:} Learned policy $\hat{\pi}$.

\end{algorithm}

\paragraph{Implementation of ECOCF\_M:} \cite{wang2023adjusting} proposes a postprocessing method that aims to achieve counterfactual fairness in cross-sectional prediction tasks. ECOCF\_M implements an adaptation of this method to the sequential decision-making setting. In particular, given training dataset $\mathcal{D} = \{s_{j,t}, z_j, a_{j,t}, r_{j,t}: j = 1, \dots, N, t = 0, \dots, T\}$ with $s_{j,t} = (s_{j,t}^{(1)}, \dots, s_{j,t}^{(M)})$, ECOCF\_M trains a policy using FQI with the sensitive attribute being included in the state variable. At deployment time, for each time step $t$, decisions for each individual $j \in \{1, \dots, N\}$ are made via the following procedure:
\begin{enumerate}
    \item For each $z \in \mathcal{Z}$, find 
    \begin{equation*}
        s_t(z) = \big([\hat{\mathbb{P}}_{t,n}^{(1)}]^{-1}(\hat{\mathbb{P}}_{t,n}^{(1)}(s_{j,t}^{(1)}|Z=z_j)|Z=z), \dots, [\hat{\mathbb{P}}_{t,n}^{(M)}]^{-1}(\hat{\mathbb{P}}_{t,n}^{(M)}(s_{j,t}^{(M)}|Z=z_j)|Z=z)\big),
    \end{equation*}
    where $\hat{\mathbb{P}}_{t,n}^{(i)}(\cdot|Z)$ is the empirical CDF of $S_t^{(i)}$ following the same definition as that in FLAP\_M.
    
    \item For $a \in \mathcal{A}$, calculate the estimated counterfactually fair Q value by 
    \begin{equation*}
        \hat{f}_{cf}(s_{j,t},a) = \sum_{z' \in \mathcal{Z}} p(z') \sum_{z \in \mathcal{Z}} p(z) \hat{f}(z, s_t(z'), a)
    \end{equation*}
    where $p(z) = \hat{\mathbb{P}}_n(Z=z)$ is the empirical population-level probability of the observed sensitive attribute taking the value $z$, and $\hat{f}: \mathcal{Z} \times \mathcal{S} \times \mathcal{A} \rightarrow \mathbb{R}$ is the Q function estimated by FQI.

    \item The decision (or action) is made by
    \begin{equation*}
        a_{j,t} = \arg\max_{a \in \mathcal{A}} \hat{f}_{cf}(s_{j,t}, a).
    \end{equation*}
\end{enumerate}

\paragraph{Implementation of CFSDP:} For the experiment presented in Section \ref{sec:additional_hd_experiment}, our implementation of CFSDP estimates the environment's transition kernel using a neural network with hidden layers $[512, 256]$. The neural network is trained using a maximum of $1000$ epochs, a learning rate of $0.0001$, and a batch size of $128$. The mean squared error is used as the loss function, and the Adam optimizer \citep{kingma2015adam} is also employed. During training, $80\%$ of the data is used as the training set, and the remaining $20\%$ is used as the testing set to test for convergence. Early stopping will be triggered if the test loss barely improves for $10$ consecutive epochs. 

For the other experiments involving CFSDP, our implementation of CFSDP estimates the environment's transition kernel using a neural network with hidden layers $[64, 64]$. The neural network is trained using a maximum of $1000$ epochs, a learning rate of $0.005$, and a batch size of $512$. The mean squared error is used as the loss function, and the Adam optimizer \citep{kingma2015adam} is also employed. No early stopping is enforced.

\paragraph{Implementation of CFSMDM:} Our implementation of CFSMDM estimates the conditional quantiles using quantile regression \citep{koenker1978regression, koenker2005qrtextbook}. For experiments using Data-generating CMDP 1, CMDP 4, or CMDP 7, the quantile regression is fit with the following specification:
\begin{align*}
S_0^{(1)} \sim& 1 + Z, \\
S_0^{(2)} \sim& 1 + Z, \\
S_t^{(1)} \sim& 1 + Z + S_{t-1}^{(1)} + (S_{t-1}^{(1)})^2 + (S_{t-1}^{(1)})^3 + A_{t-1} \\
&+ ZS_{t-1}^{(1)} + ZA_{t-1} + S_{t-1}^{(1)}A_{t-1} \text{ (for $t \geq 1$)}, \\
S_t^{(2)} \sim& 1 + Z + S_{t-1}^{(1)} + (S_{t-1}^{(1)})^2 + (S_{t-1}^{(1)})^3 + S_{t-1}^{(2)} + (S_{t-1}^{(2)})^2 + (S_{t-1}^{(2)})^3 + A_{t-1} \\
&+ ZS_{t-1}^{(1)} + ZS_{t-1}^{(2)} + ZA_{t-1} + S_{t-1}^{(1)}A_{t-1} + S_{t-1}^{(2)}A_{t-1} \text{ (for $t \geq 1$)}, \\
R_t \sim& 1 + Z + S_{t}^{(1)} + (S_{t}^{(1)})^2 + (S_{t}^{(1)})^3 + S_{t}^{(2)} + (S_{t}^{(2)})^2 + (S_{t}^{(2)})^3 + A_{t} \\
&+ ZS_{t}^{(1)} + ZS_{t}^{(2)} + ZA_{t} + S_{t}^{(1)}A_{t} + S_{t}^{(2)}A_{t} \text{ (for $t \geq 0$)}, \\
\end{align*}
For experiments using Data-generating CMDP 2, CMDP 5, or CMDP 6, the quantile regression is fit with the following specification:
\begin{align*}
S_0 \sim& 1 + Z, \\
S_t \sim& 1 + Z + S_{t-1} + A_{t-1} + ZS_{t-1} + ZA_{t-1} + S_{t-1}A_{t-1} \text{ (for $t \geq 1$)}, \\
R_t \sim& 1 + Z + S_{t} + A_{t} + ZS_{t} + ZA_{t} + S_{t}A_{t} \text{ (for $t \geq 0$)}.
\end{align*}
In both cases, the quantile regression is implemented via the \texttt{QuantReg} class of the \texttt{statsmodels} library \citep{seabold2010statsmodels}. Unless otherwise specified, the conditional quantile function is estimated by training quantile regression models for quantile levels $\tau \in \{0.01, 0.02, 0.03, \dots, 0.98, 0.99\}$ with a maximum of $2000$ training iterations. In addition, CFSMDM requires knowledge of the quantile level of the observed state component or reward on its conditional distribution, and this quantity is estimated as the quantile level in $\{0.01, \dots, 0.99\}$ whose corresponding quantile is the closest to the observed state component or reward. 

\paragraph{Implementation of CFSMDM\_QRNN:} Our implementation of CFSMDM\_QRNN estimates the conditional quantiles using quantile regression neural network \citep{taylor2000qrnn, oscar2022qrrelu}. The conditional quantile function is estimated by estimating quantile levels $\tau \in \{0.01, 0.02, 0.03, \dots, 0.98, 0.99\}$. In particular, the quantiles are jointly estimated by a neural network using the loss function described in Section 2.1 of \citep{oscar2022qrrelu}. In the experiment presented in Section \ref{sec:additional_hd_experiment}, the neural network uses hidden layers $[512, 256]$, a learning rate of $0.0001$, and a batch size of $128$. It trained for $500$ epochs. During training, $80\%$ of the data is used as the training set, and the remaining $20\%$ is used as the testing set to test for convergence. Early stopping will be triggered if the test loss barely improves for $10$ consecutive epochs. In the experiment presented in Section \ref{sec:additional_quantile_estimator_experiment}, the neural network uses hidden layers $[128, 128]$, a learning rate of $0.01$, and a batch size of $128$. No early stopping is enforced.

\paragraph{Implementation of FQI:} Unless otherwise specified, FQI is employed in the numerical experiments for policy learning for all the methods except "Random". Our implementation of FQI runs for $200$ training iterations (except that FQI runs for $500$ training iterations for the experiment in Section \ref{sec:additional_T_experiment}) and uses a discount factor of $\gamma=0.9$. In each iteration, the Q function is estimated using a neural network with hidden layers $[32]$. The neural network is trained with $500$ epochs and a learning rate of $0.1$. The mean squared error is used as the loss function, and the Adam optimizer \citep{kingma2015adam} is also employed.

\paragraph{Implementation of CQL:}
Our implementation of CQL uses a discount factor of $\gamma=0.9$ and a batch size of $256$. The number of training iterations is set to be equal to the sample size. The Q function is estimated using a neural network with hidden layers $[64,64]$. The neural network is trained with a learning rate of $0.001$. The coefficient of the CQL regularization term is set to $\alpha=1$. The number of training iterations is set to be equal to the training sample size. The Adam optimizer \citep{kingma2015adam} is employed, and the target network is updated every $100$ training iterations.

\paragraph{Computing environment:} The numerical experiments were run on an internal cluster equipped with two 36-core Intel Xeon Gold 6154 CPUs clocked at 3.0 GHz and 187 GB of RAM. Each job was executed via SLURM on a single node using 1 CPU core and 7 GB of memory.

\section{Additional Implementation Details of the PowerED Data Analysis}
\label{sec:details_rda_powered}

\paragraph{Details on training and evaluation in the PowerED data analysis:} In our analysis, we randomly choose 80\% of the dataset as the training set and use the remaining 20\% as the testing set.

For CFSDP and CFSMDM, due to the limited size of the dataset, the data used to train the preprocessor is the same as the data that is to be preprocessed. Following \cite{wang2025counterfactuallyfairreinforcementlearning}, we employ cross-fitting to mitigate the bias introduced by the data reuse. Specifically, for the CFSMDM, we divide the training dataset into $10$ folds. For each $i \in \{1, \dots, 10\}$, we train a set of quantile regression models using all data folds except the $i$-th one, and then use this set of models to preprocess the $i$-th data fold. In the end, we will have processed the entire training dataset. 10-fold cross-fitting is implemented similarly for CFSDP.

At evaluation time, the policy value is estimated using fitted Q evaluation (FQE) \citep{le2019Batch}. For the CF metric, we first train an SMDM model on the entire PowerED dataset using quantile regression with a linear specification, and then employ this model to estimate the unobserved counterfactual states of the test dataset. The CF metric is estimated from the actions chosen by the policy based on the estimated counterfactual states.

\paragraph{Implementation of CFSDP:} Our implementation of CFSDP estimates the environment's transition kernel using a neural network with hidden layers $[128, 128]$. The neural network is trained using a maximum of $1000$ epochs, a learning rate of $0.0001$, and a batch size of $128$. The mean squared error is used as the loss function, and the Adam optimizer \citep{kingma2015adam} is also employed. The ternary actions are one-hot encoded. 

\paragraph{Implementation of CFSMDM:} Our implementation of CFSMDM estimates the conditional quantiles using quantile regression \citep{koenker1978regression, koenker2005qrtextbook}. The actions are one-hot encoded because they are ternary. Specifically, the following quantile regression specification is used:
\begin{align*}
S_0^{(1)} \sim& 1 + Z, \\
S_0^{(2)} \sim& 1 + Z, \\
S_t^{(1)} \sim& 1 + Z + S_{t-1}^{(1)} + (S_{t-1}^{(1)})^2 + (S_{t-1}^{(1)})^3 + S_{t-1}^{(2)} + (S_{t-1}^{(2)})^2 + (S_{t-1}^{(2)})^3 + A_{t-1}^{(1)} + A_{t-1}^{(2)} \\
&+ ZS_{t-1}^{(1)} + ZS_{t-1}^{(2)} + ZA_{t-1}^{(1)} + ZA_{t-1}^{(2)} + S_{t-1}^{(1)}A_{t-1}^{(1)} + S_{t-1}^{(1)}A_{t-1}^{(2)} \\
&+ S_{t-1}^{(2)}A_{t-1}^{(1)} + S_{t-1}^{(2)}A_{t-1}^{(2)} + S_{t-1}^{(1)}S_{t-1}^{(2)} \text{ (for $t \geq 1$)}, 
\end{align*}
\begin{align*}
S_t^{(2)} \sim& 1 + Z + S_{t-1}^{(1)} + (S_{t-1}^{(1)})^2 + (S_{t-1}^{(1)})^3 + S_{t-1}^{(2)} + (S_{t-1}^{(2)})^2 + (S_{t-1}^{(2)})^3 + A_{t-1}^{(1)} + A_{t-1}^{(2)} \\
&+ ZS_{t-1}^{(1)} + ZS_{t-1}^{(2)} + ZA_{t-1}^{(1)} + ZA_{t-1}^{(2)} + S_{t-1}^{(1)}A_{t-1}^{(1)} + S_{t-1}^{(1)}A_{t-1}^{(2)} \\
&+ S_{t-1}^{(2)}A_{t-1}^{(1)} + S_{t-1}^{(2)}A_{t-1}^{(2)} + S_{t-1}^{(1)}S_{t-1}^{(2)} \text{ (for $t \geq 1$)}, \\
R_t \sim& 1 + Z + S_{t}^{(1)} + (S_{t}^{(1)})^2 + (S_{t}^{(1)})^3 + S_{t}^{(2)} + (S_{t}^{(2)})^2 + (S_{t}^{(2)})^3 + A_{t}^{(1)} + A_{t}^{(2)} \\
&+ ZS_{t}^{(1)} + ZS_{t}^{(2)} + ZA_{t}^{(1)} + ZA_{t}^{(2)} + S_{t}^{(1)}A_{t}^{(1)} + S_{t}^{(1)}A_{t}^{(2)} \\
&+ S_{t}^{(2)}A_{t}^{(1)} + S_{t}^{(2)}A_{t}^{(2)} + S_{t}^{(1)}S_{t}^{(2)} \text{ (for $t \geq 0$)}.
\end{align*}
where $(A_t^{(1)}, A_t^{(2)})$ represents the one-hot encoded action at time $t$. The quantile regression is implemented via the \texttt{QuantReg} class of the \texttt{statsmodels} library \citep{seabold2010statsmodels}. The conditional quantile function is estimated by training quantile regression models for quantile levels $\tau \in \{0.01, 0.02, 0.03, \dots, 0.98, 0.99\}$ with a maximum of $2000$ training iterations. In addition, CFSMDM requires knowledge of the quantile level of the observed state component or reward on its conditional distribution, and this quantity is estimated as the quantile level in $\{0.01, \dots, 0.99\}$ whose corresponding quantile is the closest to the observed state component or reward. 

\paragraph{Implementation of FLAP\_M and ECOCF\_M:} The implementations of FLAP\_M and ECOCF\_M are the same as in the numerical experiments, as described in Section \ref{sec:details_experiments}.

\paragraph{Implementation of FQI:} FQI is employed in the numerical experiments for policy learning for all the methods except "Random". Our implementation of FQI runs for $100$ training iterations and uses a discount factor of $\gamma=0.9$. In each iteration, the Q function is estimated using a neural network with hidden layers $[128, 128]$. The neural network is trained with $1000$ epochs and a learning rate of $0.01$. The mean squared error is used as the loss function, and the Adam optimizer \citep{kingma2015adam} is also employed. If the loss is more than $10$ at the end of some iteration, then the training will stop early due to likely FQI divergence. If the loss is no less than $2$ at the end of FQI training, then the FQI will be refit unless refitting has happened for at least $5$ times.

\paragraph{Implementation of FQE:} Our implementation of FQE runs for $100$ training iterations and uses a discount factor of $\gamma=0.9$. In each iteration, the Q function is estimated using a neural network with hidden layers $[128,128]$. The neural network is trained with $500$ epochs and a learning rate of $0.01$. The mean squared error is used as the loss function, and the Adam optimizer \citep{kingma2015adam} is also employed. If the loss is no less than $3$ at the end of FQE training, then the FQE will be refit unless refitting has happened for at least $5$ times.

\paragraph{Implementation of the SMDM model for CF metric estimation:} Similar to CFSMDM, our implementation of the SMDM model for CF metric estimation estimates the conditional quantiles using quantile regression \citep{koenker1978regression, koenker2005qrtextbook}. The actions are one-hot encoded because they are ternary. Specifically, the following quantile regression specification is used:
\begin{align*}
S_0^{(1)} \sim& 1 + Z, \\
S_0^{(2)} \sim& 1 + Z, \\
S_t^{(1)} \sim& 1 + Z + S_{t-1}^{(1)} + (S_{t-1}^{(1)})^2 + (S_{t-1}^{(1)})^3 + S_{t-1}^{(2)} + (S_{t-1}^{(2)})^2 + (S_{t-1}^{(2)})^3 + A_{t-1}^{(1)} + A_{t-1}^{(2)} \\
&+ ZS_{t-1}^{(1)} + ZS_{t-1}^{(2)} + ZA_{t-1}^{(1)} + ZA_{t-1}^{(2)} + S_{t-1}^{(1)}A_{t-1}^{(1)} + S_{t-1}^{(1)}A_{t-1}^{(2)} \\
&+ S_{t-1}^{(2)}A_{t-1}^{(1)} + S_{t-1}^{(2)}A_{t-1}^{(2)} + S_{t-1}^{(1)}S_{t-1}^{(2)} \text{ (for $t \geq 1$)}, \\
S_t^{(2)} \sim& 1 + Z + S_{t-1}^{(1)} + (S_{t-1}^{(1)})^2 + (S_{t-1}^{(1)})^3 + S_{t-1}^{(2)} + (S_{t-1}^{(2)})^2 + (S_{t-1}^{(2)})^3 + A_{t-1}^{(1)} + A_{t-1}^{(2)} \\
&+ ZS_{t-1}^{(1)} + ZS_{t-1}^{(2)} + ZA_{t-1}^{(1)} + ZA_{t-1}^{(2)} + S_{t-1}^{(1)}A_{t-1}^{(1)} + S_{t-1}^{(1)}A_{t-1}^{(2)} \\
&+ S_{t-1}^{(2)}A_{t-1}^{(1)} + S_{t-1}^{(2)}A_{t-1}^{(2)} + S_{t-1}^{(1)}S_{t-1}^{(2)} \text{ (for $t \geq 1$)}, \\
R_t \sim& 1 + Z + S_{t}^{(1)} + (S_{t}^{(1)})^2 + (S_{t}^{(1)})^3 + S_{t}^{(2)} + (S_{t}^{(2)})^2 + (S_{t}^{(2)})^3 + A_{t}^{(1)} + A_{t}^{(2)} \\
&+ ZS_{t}^{(1)} + ZS_{t}^{(2)} + ZA_{t}^{(1)} + ZA_{t}^{(2)} + S_{t}^{(1)}A_{t}^{(1)} + S_{t}^{(1)}A_{t}^{(2)} \\
&+ S_{t}^{(2)}A_{t}^{(1)} + S_{t}^{(2)}A_{t}^{(2)} + S_{t}^{(1)}S_{t}^{(2)} \text{ (for $t \geq 0$)}.
\end{align*}
where $(A_t^{(1)}, A_t^{(2)})$ represents the one-hot encoded action at time $t$. The quantile regression is implemented via the \texttt{QuantReg} class of the \texttt{statsmodels} library \citep{seabold2010statsmodels}. The conditional quantile function is estimated by training quantile regression models for quantile levels $\tau \in \{0.01, 0.02, 0.03, \dots, 0.98, 0.99\}$ with a maximum of $2000$ training iterations. In addition, SMDM requires knowledge of the quantile level of the observed state component on its conditional distribution, and this quantity is estimated as the quantile level in $\{0.01, \dots, 0.99\}$ whose corresponding quantile is the closest to the observed state component. However, unlike for CFSDP and CFSMDM, we do not use cross-fitting for this SMDM model.

\paragraph{Computing environment:} The numerical experiments were run on an internal cluster equipped with two 36-core Intel Xeon Gold 6154 CPUs clocked at 3.0 GHz and 187 GB of RAM. Each job was executed via SLURM on a single node using 1 CPU core and 7 GB of memory.

\section{Additional Implementation Details of the Intern Health Study Data Analysis}
\label{sec:details_rda_ihs}

\paragraph{Details on training and evaluation in the Intern Health Study data analysis:} In our analysis, we randomly choose 40\% of the dataset for preprocessor training and policy learning, another 40\% of the dataset for training the ground truth model (in particular, the SMDM model used for CF metric estimation; discussed in detail later), and use the remaining 20\% as the testing set. 10-fold cross-fitting, as described in Section \ref{sec:details_rda_powered}, is employed for training CFSDP and CFSMDM.

At evaluation time, the policy value is estimated using fitted Q evaluation (FQE) \citep{le2019Batch}. For the CF metric, we first train an SMDM model on an independent split of the Intern Health Study dataset using quantile regression neural network, and then employ this model to estimate the unobserved counterfactual states of the test dataset. The CF metric is estimated from the actions chosen by the policy based on the estimated counterfactual states.

\paragraph{Implementation of CFSDP:} Our implementation of CFSDP estimates the environment's transition kernel using a neural network with hidden layers $[128, 128]$. The neural network is trained for $200$ epochs, a learning rate of $0.0001$, and a batch size of $128$. The mean squared error is used as the loss function, and the Adam optimizer \citep{kingma2015adam} is also employed. No early stopping is enforced.

\paragraph{Implementation of CFSMDM\_QRNN:} Our implementation of CFSMDM\_QRNN estimates the conditional quantiles using quantile regression neural network \citep{taylor2000qrnn, oscar2022qrrelu}. The conditional quantile function is estimated by estimating quantile levels $\tau \in \{0.01, 0.02, 0.03, \dots, 0.98, 0.99\}$. In particular, the quantiles are jointly estimated by a neural network using the loss function described in Section 2.1 of \citep{oscar2022qrrelu}. The neural network uses hidden layers $[128, 128]$, a learning rate of $0.0001$, and a batch size of $128$. It trained for $200$ epochs. No early stopping is enforced.

\paragraph{Implementation of FLAP\_M and ECOCF\_M:} The implementations of FLAP\_M and ECOCF\_M are the same as in the numerical experiments, as described in Section \ref{sec:details_experiments}.

\paragraph{Implementation of FQI:} FQI is employed in the numerical experiments for policy learning for all the methods except "Random". Our implementation of FQI runs for $100$ training iterations and uses a discount factor of $\gamma=0.9$. In each iteration, the Q function is estimated using a neural network with hidden layers $[128, 128]$. The neural network is trained with $200$ epochs and a learning rate of $0.001$. The mean squared error is used as the loss function, and the Adam optimizer \citep{kingma2015adam} is also employed. 

\paragraph{Implementation of FQE:} Our implementation of FQE runs for $200$ training iterations and uses a discount factor of $\gamma=0.9$. In each iteration, the Q function is estimated using a neural network with hidden layers $[128, 128]$. The neural network is trained with $500$ epochs and a learning rate of $0.001$. The mean squared error is used as the loss function, and the Adam optimizer \citep{kingma2015adam} is also employed. 

\paragraph{Implementation of the SMDM model for CF metric estimation:} Similar to CFSMDM\_QRNN, our implementation of the SMDM model for CF metric estimation estimates the conditional quantiles using quantile regression neural network \citep{taylor2000qrnn, oscar2022qrrelu}. The conditional quantile function is estimated by estimating quantile levels $\tau \in \{0.01, 0.02, 0.03, \dots, 0.98, 0.99\}$. In particular, the quantiles are jointly estimated by a neural network using the loss function described in Section 2.1 of \citep{oscar2022qrrelu}. The neural network uses hidden layers $[128, 128]$, a learning rate of $0.0001$, and a batch size of $128$. It trained for $100$ epochs. No early stopping is enforced.

\paragraph{Computing environment:} The numerical experiments were run on an internal cluster equipped with two 36-core Intel Xeon Gold 6154 CPUs clocked at 3.0 GHz and 187 GB of RAM. Each job was executed via SLURM on a single node using 1 CPU core and 7 GB of memory.

\section{Additional Implementation Details of the Semi-synthetic Experiments}
\label{sec:details_semisyn}

\paragraph{Implementation of the generative model:} A generative model is used to approximate the transition kernels of the environments underlying the real-world dataset. In particular, for both datasets, the sensitive attributes are generated using a \texttt{CTGANSynthesizer} from the Synthetic Data Vault library \cite{patki2016sdv}, and the state transition kernel is approximated by a neural network with hidden layers $[128, 128]$. For the PowerED dataset, the reward transition kernel is approximated by another neural network with hidden layers $[128, 128]$, while the generative model for the Intern Health Study dataset does not have a separate reward model because its reward (which is next week's mood score) is part of the state variables and therefore can be estimated by the state model. The neural networks are trained for $1000$ epochs using a learning rate of $0.001$ and a batch size of $128$. The action is one-hot encoded if it has more than two levels. The mean squared error is used as the loss function, and the Adam optimizer \citep{kingma2015adam} is also employed. During training, $80\%$ of the data is used as the training set, and the remaining $20\%$ is used as the testing set to test for convergence. Early stopping will be triggered if the test loss barely improves for $10$ consecutive epochs.

\paragraph{Implementation of CFSDP:} For the PowerED semi-synthetic experiment, our implementation of CFSDP estimates the environment's transition kernel using a neural network with hidden layers $[64, 64]$. The neural network is trained using a maximum of $1000$ epochs, a learning rate of $0.005$, and a batch size of $512$. For the Intern Health Study semi-synthetic experiment, our implementation of CFSDP estimates the environment's transition kernel using a neural network with hidden layers $[256, 128]$. The neural network is trained using a maximum of $200$ epochs, a learning rate of $0.01$, and a batch size of $128$. For both datasets, the mean squared error is used as the loss function, and the Adam optimizer \citep{kingma2015adam} is also employed. The action is one-hot encoded if it has more than two levels. No early stopping is enforced.

\paragraph{Implementation of CFSMDM:} In the analysis of the PowerED data, our implementation of CFSMDM estimates the conditional quantiles using quantile regression \citep{koenker1978regression, koenker2005qrtextbook}. The actions are one-hot encoded because they are ternary. Specifically, the following quantile regression specification is used:
\begin{align*}
S_0^{(1)} \sim& 1 + Z, \\
S_0^{(2)} \sim& 1 + Z, \\
S_t^{(1)} \sim& 1 + Z + S_{t-1}^{(1)} + (S_{t-1}^{(1)})^2 + (S_{t-1}^{(1)})^3 + S_{t-1}^{(2)} + (S_{t-1}^{(2)})^2 + (S_{t-1}^{(2)})^3 + A_{t-1}^{(1)} + A_{t-1}^{(2)} \\
&+ ZS_{t-1}^{(1)} + ZS_{t-1}^{(2)} + ZA_{t-1}^{(1)} + ZA_{t-1}^{(2)} + S_{t-1}^{(1)}A_{t-1}^{(1)} + S_{t-1}^{(1)}A_{t-1}^{(2)} \\
&+ S_{t-1}^{(2)}A_{t-1}^{(1)} + S_{t-1}^{(2)}A_{t-1}^{(2)} + S_{t-1}^{(1)}S_{t-1}^{(2)} \text{ (for $t \geq 1$)}, 
\end{align*}
\begin{align*}
S_t^{(2)} \sim& 1 + Z + S_{t-1}^{(1)} + (S_{t-1}^{(1)})^2 + (S_{t-1}^{(1)})^3 + S_{t-1}^{(2)} + (S_{t-1}^{(2)})^2 + (S_{t-1}^{(2)})^3 + A_{t-1}^{(1)} + A_{t-1}^{(2)} \\
&+ ZS_{t-1}^{(1)} + ZS_{t-1}^{(2)} + ZA_{t-1}^{(1)} + ZA_{t-1}^{(2)} + S_{t-1}^{(1)}A_{t-1}^{(1)} + S_{t-1}^{(1)}A_{t-1}^{(2)} \\
&+ S_{t-1}^{(2)}A_{t-1}^{(1)} + S_{t-1}^{(2)}A_{t-1}^{(2)} + S_{t-1}^{(1)}S_{t-1}^{(2)} \text{ (for $t \geq 1$)}, \\
R_t \sim& 1 + Z + S_{t}^{(1)} + (S_{t}^{(1)})^2 + (S_{t}^{(1)})^3 + S_{t}^{(2)} + (S_{t}^{(2)})^2 + (S_{t}^{(2)})^3 + A_{t}^{(1)} + A_{t}^{(2)} \\
&+ ZS_{t}^{(1)} + ZS_{t}^{(2)} + ZA_{t}^{(1)} + ZA_{t}^{(2)} + S_{t}^{(1)}A_{t}^{(1)} + S_{t}^{(1)}A_{t}^{(2)} \\
&+ S_{t}^{(2)}A_{t}^{(1)} + S_{t}^{(2)}A_{t}^{(2)} + S_{t}^{(1)}S_{t}^{(2)} \text{ (for $t \geq 0$)}.
\end{align*}
where $(A_t^{(1)}, A_t^{(2)})$ represents the one-hot encoded action at time $t$. The quantile regression is implemented via the \texttt{QuantReg} class of the \texttt{statsmodels} library \citep{seabold2010statsmodels}. The conditional quantile function is estimated by training quantile regression models for quantile levels $\tau \in \{0.01, 0.02, 0.03, \dots, 0.98, 0.99\}$ with a maximum of $2000$ training iterations. In addition, CFSMDM requires knowledge of the quantile level of the observed state component or reward on its conditional distribution, and this quantity is estimated as the quantile level in $\{0.01, \dots, 0.99\}$ whose corresponding quantile is the closest to the observed state component or reward. 

\paragraph{Implementation of CFSMDM\_QRNN:} In the analysis of the Intern Health Study data, our implementation of CFSMDM\_QRNN estimates the conditional quantiles using quantile regression neural network \citep{taylor2000qrnn, oscar2022qrrelu}. The conditional quantile function is estimated by estimating quantile levels $\tau \in \{0.01, 0.02, 0.03, \dots, 0.98, 0.99\}$. In particular, the quantiles are jointly estimated by a neural network using the loss function described in Section 2.1 of \citep{oscar2022qrrelu}. The neural network uses hidden layers $[256, 128]$, a learning rate of $0.01$, and a batch size of $128$. It trained for $200$ epochs. No early stopping is enforced. 

\paragraph{Implementation of FLAP\_M and ECOCF\_M:} The implementations of FLAP\_M and ECOCF\_M are the same as in the numerical experiments, as described in Section \ref{sec:details_experiments}.

\paragraph{Implementation of FQI:} FQI is employed in the numerical experiments for policy learning for all the methods except "Random". For both datasets, our implementation of FQI runs for $200$ training iterations and uses a discount factor of $\gamma=0.9$. In each iteration, the Q function is estimated using a neural network with hidden layers $[32]$. The neural network is trained with $500$ epochs and a learning rate of $0.1$. The mean squared error is used as the loss function, and the Adam optimizer \citep{kingma2015adam} is also employed.

\paragraph{Computing environment:} The numerical experiments were run on an internal cluster equipped with two 36-core Intel Xeon Gold 6154 CPUs clocked at 3.0 GHz and 187 GB of RAM. Each job was executed via SLURM on a single node using 1 CPU core and 7 GB of memory.

\section{Broader Impacts}
\label{sec:impacts}

This work introduces the CFSMDM algorithm to prevent reinforcement learning agents from learning policies that may disadvantage certain subpopulations, offering significant positive societal impacts such as promoting equity in high-stakes domains and correcting for implicitly encoded historical biases. However, deploying this framework involves a fairness-utility trade-off where enforcing strict equity may reduce overall population-level utility. Additionally, practitioners need to carefully navigate data privacy concerns due to the necessary collection and processing of sensitive demographic attributes for auditing and controlling fairness.




\newpage
\section*{NeurIPS Paper Checklist}

\begin{enumerate}

\item {\bf Claims}
    \item[] Question: Do the main claims made in the abstract and introduction accurately reflect the paper's contributions and scope?
    \item[] Answer: \answerYes{} 
    \item[] Justification: The abstract and introduction state the paper's contribution in developing the CFSMDM algorithm, providing theoretical guarantees for it, and evaluating its performance using experiments and real data analyses. This aligns with the main content of the paper.

\item {\bf Limitations}
    \item[] Question: Does the paper discuss the limitations of the work performed by the authors?
    \item[] Answer: \answerYes{} 
    \item[] Justification: The limitations are discussed in a dedicated ``limitations'' subsection in Section \ref{sec:conclusion}.

\item {\bf Theory assumptions and proofs}
    \item[] Question: For each theoretical result, does the paper provide the full set of assumptions and a complete (and correct) proof?
    \item[] Answer: \answerYes{} 
    \item[] Justification: Assumptions are stated and discussed in Sections \ref{sec:cfsmdm}, \ref{sec:theory}, and \ref{subsubsec:assumptions}. Theoretical proofs are provided in Section \ref{sec:proofs}.

    \item {\bf Experimental result reproducibility}
    \item[] Question: Does the paper fully disclose all the information needed to reproduce the main experimental results of the paper to the extent that it affects the main claims and/or conclusions of the paper (regardless of whether the code and data are provided or not)?
    \item[] Answer: \answerYes{} 
    \item[] Justification: The experimental setup is discussed along with the experiment results in Sections \ref{sec:experiments}, \ref{sec:rda}, \ref{sec:additional_experiments}, \ref{sec:rda_ihs}, and \ref{sec:semisyn_experiments}, and the implementation details are provided in Sections \ref{sec:details_experiments}, \ref{sec:details_rda_powered}, \ref{sec:details_rda_ihs}, and \ref{sec:details_semisyn}. The code is also released both in the supplementary materials and on GitHub.

\item {\bf Open access to data and code}
    \item[] Question: Does the paper provide open access to the data and code, with sufficient instructions to faithfully reproduce the main experimental results, as described in supplemental material?
    \item[] Answer: \answerYes{} 
    \item[] Justification: We release the code of the experiments and real data analyses both in the supplementary materials and on GitHub. The PowerED and Intern Health Study datasets used for the real data analysis is available upon reasonable request, following the guidelines detailed in the ``Data and Code Availability Statement'' after Section \ref{sec:conclusion}.

\item {\bf Experimental setting/details}
    \item[] Question: Does the paper specify all the training and test details (e.g., data splits, hyperparameters, how they were chosen, type of optimizer) necessary to understand the results?
    \item[] Answer: \answerYes{} 
    \item[] Justification: The experimental setting and details are discussed along with the experiment results in Sections \ref{sec:experiments}, \ref{sec:rda}, \ref{sec:additional_experiments}, \ref{sec:rda_ihs}, and \ref{sec:semisyn_experiments}, and the implementation details are provided in Sections \ref{sec:details_experiments}, \ref{sec:details_rda_powered}, \ref{sec:details_rda_ihs}, and \ref{sec:details_semisyn}. The code is also released both in the supplementary materials and on GitHub.

\item {\bf Experiment statistical significance}
    \item[] Question: Does the paper report error bars suitably and correctly defined or other appropriate information about the statistical significance of the experiments?
    \item[] Answer: \answerYes{} 
    \item[] Justification: When reporting experiment and real data analysis results in Sections \ref{sec:experiments}, \ref{sec:rda}, \ref{sec:additional_experiments}, \ref{sec:rda_ihs}, and \ref{sec:semisyn_experiments}, we also report either the standard deviation or the band between $0.025$ and $0.975$ empirical quantiles.

\item {\bf Experiments compute resources}
    \item[] Question: For each experiment, does the paper provide sufficient information on the computer resources (type of compute workers, memory, time of execution) needed to reproduce the experiments?
    \item[] Answer: \answerYes{} 
    \item[] Justification: We detail the compute resources used for the experiments and real data analyses in Sections \ref{sec:details_experiments}, \ref{sec:details_rda_powered}, \ref{sec:details_rda_ihs}, and \ref{sec:details_semisyn}.
    
\item {\bf Code of ethics}
    \item[] Question: Does the research conducted in the paper conform, in every respect, with the NeurIPS Code of Ethics \url{https://neurips.cc/public/EthicsGuidelines}?
    \item[] Answer: \answerYes{} 
    \item[] Justification: The PowerED and Intern Health Study datasets used in the real data analyses (Sections \ref{sec:rda} and \ref{sec:rda_ihs}, respectively) were accessed and analyzed ethically.

\item {\bf Broader impacts}
    \item[] Question: Does the paper discuss both potential positive societal impacts and negative societal impacts of the work performed?
    \item[] Answer: \answerYes{} 
    \item[] Justification: We provide a discussion of both the positive and negative societal impacts in Section \ref{sec:impacts}.
    
\item {\bf Safeguards}
    \item[] Question: Does the paper describe safeguards that have been put in place for responsible release of data or models that have a high risk for misuse (e.g., pre-trained language models, image generators, or scraped datasets)?
    \item[] Answer: \answerNA{} 
    \item[] Justification: This paper does not release data or models that are of high risk of misuse.

\item {\bf Licenses for existing assets}
    \item[] Question: Are the creators or original owners of assets (e.g., code, data, models), used in the paper, properly credited and are the license and terms of use explicitly mentioned and properly respected?
    \item[] Answer: \answerYes{} 
    \item[] Justification: The PowerED dataset used in real data analysis has been cited in Section \ref{sec:rda}. The Intern health Study dataset used in real data analysis has been cited in Section \ref{sec:rda_ihs}.

\item {\bf New assets}
    \item[] Question: Are new assets introduced in the paper well documented and is the documentation provided alongside the assets?
    \item[] Answer: \answerNA{} 
    \item[] Justification: This paper does not release new assets.

\item {\bf Crowdsourcing and research with human subjects}
    \item[] Question: For crowdsourcing experiments and research with human subjects, does the paper include the full text of instructions given to participants and screenshots, if applicable, as well as details about compensation (if any)? 
    \item[] Answer: \answerNA{} 
    \item[] Justification: This paper does not involve crowdsourcing or human subject research.

\item {\bf Institutional review board (IRB) approvals or equivalent for research with human subjects}
    \item[] Question: Does the paper describe potential risks incurred by study participants, whether such risks were disclosed to the subjects, and whether Institutional Review Board (IRB) approvals (or an equivalent approval/review based on the requirements of your country or institution) were obtained?
    \item[] Answer: \answerNA{} 
    \item[] Justification: This paper does not involve crowdsourcing or human subject research.

\item {\bf Declaration of LLM usage}
    \item[] Question: Does the paper describe the usage of LLMs if it is an important, original, or non-standard component of the core methods in this research? Note that if the LLM is used only for writing, editing, or formatting purposes and does \emph{not} impact the core methodology, scientific rigor, or originality of the research, declaration is not required.
    \item[] Answer: \answerNA{} 
    \item[] Justification: The core method development for this paper does not involve LLMs.

\end{enumerate}

\end{document}